\documentclass[conference]{IEEEtran}
\usepackage{times}
\usepackage[numbers]{natbib}
\usepackage[bookmarks=true]{hyperref}
\usepackage{amsmath,amssymb, amsthm}
\usepackage{algorithm}
\usepackage{algorithmicx}
\usepackage{caption}
\usepackage{algpseudocode}
\usepackage{graphicx}
\usepackage{bibunits}
\defaultbibliography{references}
\defaultbibliographystyle{plainnat}
\usepackage{booktabs}
\usepackage{multirow}
\usepackage{multicol}
\usepackage{xcolor}
\usepackage{subcaption}
\usepackage{siunitx}
\usepackage{tabularew}
\newtheorem{theorem}{Theorem}[section]
\newtheorem{lemma}[theorem]{Lemma}

\newcommand{\MethodName}{\emph{Q2RL}}

\newcommand{\bcpi}{\pi_{\rm BC}}
\newcommand{\rlpi}{\pi_{\rm RL}}
\newcommand{\bcQ}{Q_{\rm BC}}

\definecolor{myblue}{RGB}{0, 114, 178}
\definecolor{myyellow}{RGB}{230, 159, 0}

\begin{document} 

\title{When Life Gives You BC, Make Q-functions: \\ Extracting Q-values from Behavior Cloning for On-Robot Reinforcement Learning}
\makeatletter
\renewcommand{\authorrefmark}[1]{\textsuperscript{#1}}
\makeatother
\author{\authorblockN{Lakshita Dodeja\authorrefmark{1}\textsuperscript{,}\authorrefmark{2},
Ondrej Biza\authorrefmark{1},
Shivam Vats\authorrefmark{2}, 
Stephen Hart\authorrefmark{1}, 
Stefanie Tellex\authorrefmark{1}\textsuperscript{,}\authorrefmark{2}, \\ 
Robin Walters\authorrefmark{3},
Karl Schmeckpeper\authorrefmark{1}, and
Thomas Weng\authorrefmark{1}}

\vspace{0.8em}
\authorblockA{
\begin{tabular}{ccc}
\authorrefmark{1}Robotics and AI Institute & 
\authorrefmark{2}Brown University & 
\authorrefmark{3}Northeastern University \\
Cambridge, MA, USA & 
Providence, RI, USA & 
Boston, MA, USA
\end{tabular}
}
}

\makeatletter
\patchcmd{\@maketitle}
  {\end{center}}  
  {%
    \end{center}  
    \vskip 0.4em
    {\centering
      \includegraphics[width=1.0\linewidth]{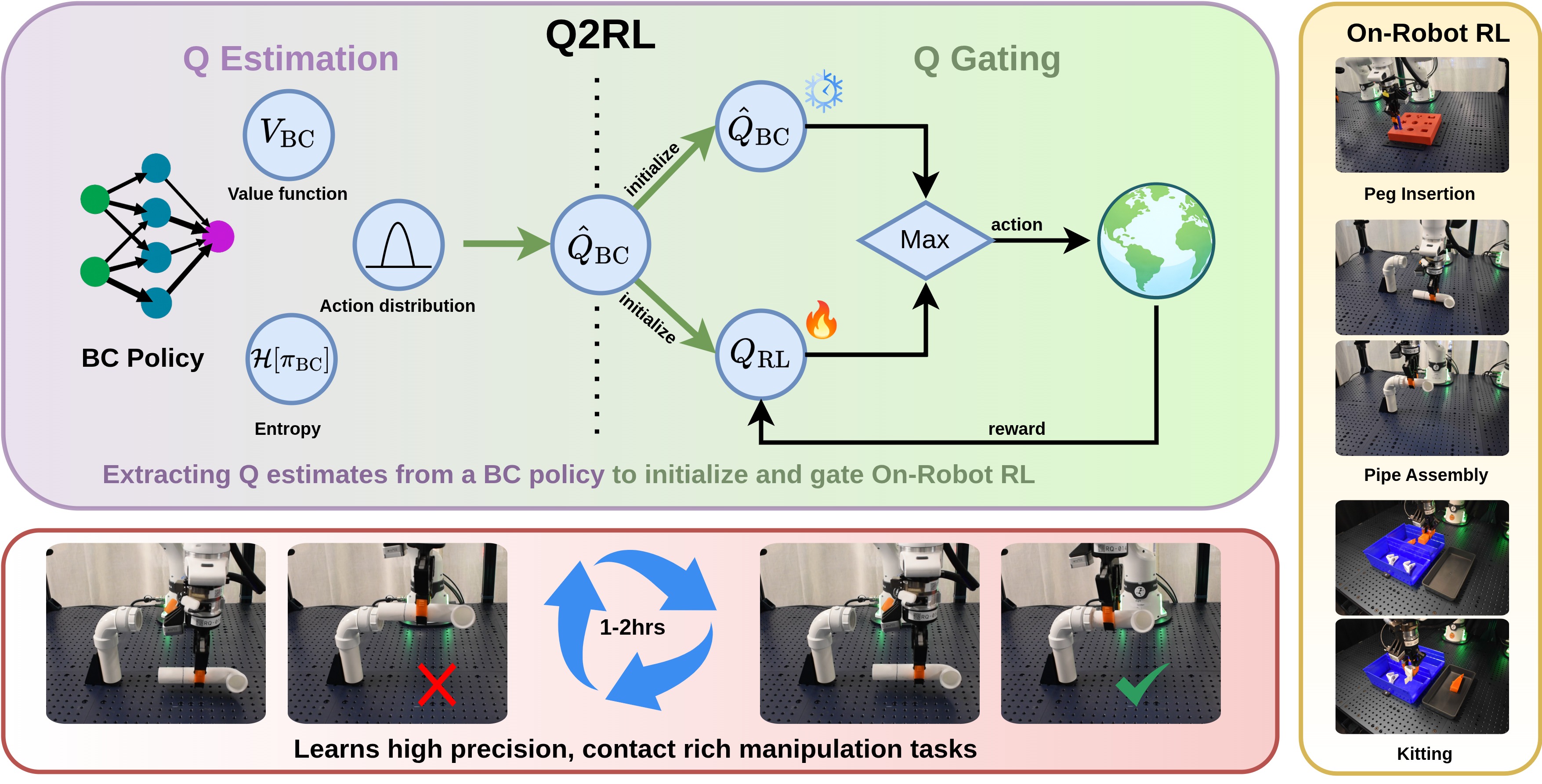}\par
      \captionof{figure}{\MethodName{} consists of Q-Estimation and Q-Gating. Q-Estimation extracts a Q-function from a BC Policy. 
      Then, during RL policy training, Q-Gating selects and executes the BC or RL action with highest respective Q-value, updating the RL policy on the collected interactions. \MethodName{} 
      learns contact-rich manipulation tasks in 1-2 hours of online interaction.
      \label{fig:teaser}}\par
    }
    \vskip -0.6em
  }
  {}{}
\makeatother

\maketitle
\setcounter{figure}{1}  

\begin{abstract}
Behavior Cloning (BC) has emerged as a highly effective paradigm for robot learning.
However, 
BC lacks a self-guided mechanism for online improvement after demonstrations have been collected.
Existing offline-to-online learning methods often cause policies to replace previously learned good actions due to a distribution mismatch between offline data and online learning.
In this work, we propose \textcolor{blue}{\MethodName{}}, \textcolor{blue}{Q}-Estimation and \textcolor{blue}{Q}-Gating from BC for \textcolor{blue}{R}einforcement \textcolor{blue}{L}earning, an algorithm for efficient offline-to-online learning.
Our method consists of two parts: (1) \emph{Q-Estimation} extracts a Q-function from a BC policy using a few interaction steps with the environment,
followed by online RL with (2) \emph{Q-Gating}, which switches between BC and RL policy actions based on their respective Q-values to collect samples for RL policy training. 
Across manipulation tasks from D4RL and robomimic benchmarks, \MethodName{} outperforms SOTA offline-to-online learning baselines on success rate and time to convergence.
\MethodName{} is efficient enough to be applied in an on-robot RL setting, learning robust policies for contact-rich and high precision manipulation tasks such as pipe assembly and kitting, in 1-2 hours of online interaction, achieving success rates of up to 100\% and up to 3.75x improvement against the original BC policy. Code and video are available at \url{https://q2rl.rai-inst.com/}.
\end{abstract}

\IEEEpeerreviewmaketitle

\begin{bibunit}

\vspace{-0.8em}

\section{Introduction}
\label{sec:intro}

Behavior Cloning (BC) has been highly successful in robot learning due to its simple supervised training objective, which directly models an action distribution conditioned on observed states. 
However, BC does not have a mechanism for self-guided online improvement.
When the training data for a BC policy is insufficient for the scope of the task, and/or changes in the environment induce even subtle distribution shifts, BC policy performance can collapse due to compounding errors from covariate shift~\cite{ross2010efficient}.

Several research directions have been proposed to address the shortcomings of behavior cloning.
Interactive imitation learning approaches finetune BC policies online, but require expert annotators~\cite{kelly19hgdagger}.
Offline reinforcement learning (RL) learns a value function offline~\cite{agarwal2020optimistic, kumar2020conservative} to guide online learning, but require pretraining on large datasets of positive and negative offline samples~\cite{levine2020offline} for good performance.
Offline-to-online learning approaches have also been proposed, but they often result in unlearning good BC actions~\cite{nakamoto2023cal, zhou2025efficient}.
Another challenge faced by offline-to-online methods is the tendency to produce unsafe behaviors due to training an RL policy from scratch. 
In the on-robot RL setting in which online RL is trained directly on the robot, executing random exploratory actions is unsafe, causing significant wear and tear on the robot, workspace, and manipulated objects. 

In this paper, our objective is to combine the benefits of pre-training a BC policy with the iterative improvement capability of online RL.
We aim to keep good BC policy behavior, and when BC performs poorly, post-train with RL to find better actions. 
We propose Q-Estimation and Q-Gating from BC for Reinforcement Learning (\MethodName{}), an algorithm for efficient offline-to-online learning. 
In the \textit{Q-Estimation} phase, we estimate the Q-function of the pretrained BC policy from initial online rollouts, then use it to initialize off-policy reinforcement learning.
We are able to derive these estimated Q-values with just the action selection probability and entropy of the BC policy, \textit{without any additional information of the underlying policy class or training distribution}. 
To ensure a stable transition from offline pretraining to online RL, we propose \textit{Q-Gating}, which takes the Q-function from the Q-Estimation phase and uses it to guide online reinforcement learning.
Two copies of the Q-function are used: the initially extracted Q function is kept frozen to preserve good BC actions, and a learnable RL Q-function is optimized to search for potentially better actions. 
At each state, \MethodName{} compares these Q-estimates to selectively execute BC or RL actions, ensuring that the RL policy is able to initialize learning from good BC actions and explore online to improve beyond the BC policy.

We provide extensive experiments in simulation and real-world on-robot RL to validate the method.
\MethodName{} outperforms SOTA offline RL and offline-to-online methods across D4RL and robomimic benchmark tasks, improving the success rates from 50-60\% for the original BC policy up to 80-100\% for most environments.
Ablation experiments show that both Q-Estimation and Q-Gating are valuable components of the algorithm. 
Our real-world experiments show that \MethodName{} learns contact-rich, high-precision manipulation tasks in 1-2 hours of online interaction, outperforming baselines especially as tasks increase in difficulty, and 
improving the success rate compared to the original BC policies by up to 100\% and up to 3.75x improvement. 
Furthermore, our real-world experiments demonstrate that \MethodName{} can recover from distribution shifts in the environment.

Our contributions are as follows:
\begin{itemize}
    \item ``Q-Estimation'', an algorithm for estimating a Q-function from a behavior cloning policy and a few online rollouts. 
    Q-Estimation can be applied to any policy class that provides action likelihoods and entropy. 
    \item ``Q-Gating'', an algorithm that guides online reinforcement learning via a gating mechanism to select between BC and RL actions.
    \item Experiments demonstrating that \MethodName{}, which combines Q-Estimation and Q-Gating, outperforms baselines on D4RL and robomimic benchmarks.
    \item Real-world, on-robot reinforcement learning experiments with \MethodName{} that surpass the original BC policy's performance on contact-rich manipulation tasks and adapt to task distribution shifts unseen by the BC policy, in 1-2 hours of environment interaction. 
\end{itemize}

\section{Related Work}
\label{sec:related}

We review existing approaches to offline learning such as Behavior Cloning and offline RL, as well as offline-to-online approaches that combine offline learning with online RL. 
Then, we review on-robot reinforcement learning, which require sample-efficient techniques for real-world learning. 

\subsection{Offline Learning}
Offline learning in robotics is broadly divided into Imitation Learning based approaches~\cite{zhang2018deep,mandlekar2020learning} and Offline Reinforcement Learning approaches~\cite{kumar2020conservative, levine2020offline}.
Imitation learning approaches that train policies via supervised learning (i.e., Behavior Cloning) learn a direct mapping from states to an action distribution using demonstration data.
Offline RL methods learn value functions that estimate expected returns from data generated by a behavior policy.  
Offline RL techniques like CQL~\cite{kumar2020conservative} penalize out-of-distribution actions to address dataset shift, but this can introduce pessimistic bias in value estimates.
IQ-Learn~\cite{garg2021iq} learns soft Q-functions directly from demonstrations, enabling policy extraction through RL objectives. 
In contrast, our approach does not learn a new Q-function from demonstrations, but instead estimates the Q-values of an existing black-box BC policy to guide online improvement.
Furthermore, offline RL methods are more suited to learn from large and potentially suboptimal datasets~\cite{fujimoto2019benchmarking}, and not necessarily from robotics datasets where successful demonstrations are easier to curate~\cite{nasiriany2022learning, BerkeleyUR5Website}.
BC approaches can further be divided into policy classes that explicitly parameterize the action distribution, such as Gaussian or Gaussian mixture heads paired with any network backbone~\cite{mandlekar2021matters, sheebaelhamd2025quantization}, and policy classes that implicitly represent the action distribution using generative denoising processes such as diffusion \cite{chi2025diffusion,team2024octo} or flow-based models \cite{lipman2022flow, jiang2025streaming}. 
\MethodName{} only requires access to the action selection probabilities and the entropy, which is easier to obtain for policy classes that explicitly parameterize the action distribution.

\subsection{Offline-to-Online Learning}
Offline learning can provide strong initial performance from readily available datasets, but policies trained purely offline cannot improve further without additional supervision or online interaction~\cite{kelly19hgdagger,ross2011reduction}.
Calibrated Q-Learning (CalQL)~\cite{nakamoto2023cal} aims to enable a smooth transition from offline to online RL by calibrating offline Q-values against a reference policy, but can struggle with limited offline data.
WSRL~\cite{zhou2025efficient} mitigates this by adding a warm-up phase that seeds the online replay buffer with rollouts from the offline policy.
In contrast, other methods bypass offline RL by first pretraining an IL policy, then updating the behavior of the agent by adding corrective residual terms to the output action \cite{silver2018residual,johannink2019residual,yuan2024policy,dodeja2025accelerating}. 
These approaches require task-specific tuning of the residual action scale to ensure that the residual policy does not deviate from the BC policy, and can fail to optimize large action parameterizations. 
\MethodName{} does not use a residual formulation and allows the RL policy to learn completely different actions from the original BC policy, while still relying on the BC policy for guidance.
Imitation Bootstrapped RL (IBRL)~\cite{hu2023imitation}, most closely related to our approach, also combines a pretrained IL/BC policy with an RL policy and selects between their actions using a randomly initialized critic.
A key benefit of \MethodName{} is that it \textit{estimates the Q-function of the BC policy} from early online interaction, and uses it in conjunction with an RL Q-function to evaluate actions during online, off-policy reinforcement learning.
In our evaluations, we show that \MethodName{} outperforms IBRL in simulated and real-world experiments, and that \MethodName{} does not require seeding the online replay buffer with high quality demonstrations, unlike IBRL.

\subsection{On-Robot Reinforcement Learning} 
Online reinforcement learning on a real robot is a challenging but promising post-training strategy for achieving robust robotic policies~\cite{lei2026rl100performantroboticmanipulation}. 
Since random exploration is costly on hardware, practical methods must begin with reasonable initial performance and adapt sample-efficiently.
SERL~\cite{luo2024serl} introduces a system designed specifically for on-robot RL.
It provides a software suite to run online RL algorithms~\cite{haarnoja2018soft, yarats2021image, ball2023efficient}
using async actor and learner processes~\cite{agentlace2024}.
Extensions like HIL-SERL~\cite{luo2025precise} incorporate human interventions to further improve efficiency, while other approaches like GCR~\cite{11128466} use reward shaping to guide learning. 
\MethodName{} uses the async framework for on-robot RL from SERL, but leverages readily available BC policies to provide strong initial behavior, without additional reward engineering or human intervention.  


\section{Problem Formulation}

Assume we are given a behavior cloning policy $\bcpi$, pre-trained on a dataset $\mathcal{D}$ of successful human demonstrations for a task $\mathcal{M}_{\rm train} \in \mathfrak{M}$.
Here $\mathfrak{M}$ denotes a class of tasks that share the same success condition; in the case of task distribution shift, $\mathcal{M}_{\rm test}$ and $\mathcal{M}_{\rm train}$ are drawn from task class $\mathfrak{M}$.
The dataset consists of $M$ trajectories of observation-action tuples:
\begin{equation}
  \mathcal{D} = \left\{ \tau_j \right\}^M_{j=1}, \qquad \tau_j = \left\{ (o_{j,t}, a_{j,t}) \right\}^{T_j}_{t=1},
\end{equation}
where $T_j$ is the length of trajectory $\tau_j$, and $o_{j,t}, a_{j,t}$ are the observation and action, respectively.

Assume the BC policy $\bcpi$ performs sub-optimally at test time (i.e. failing to achieve the success condition), because the amount of training data or number of training epochs were insufficient, or due to distribution shifts induced by changes in environment conditions such as lighting variations, the presence of visual distractors, modified robot hardware or controllers, or other changes to environment dynamics.

\begin{algorithm}[b]
\small
\captionsetup{font=small}{
\caption{Q2RL}
\label{alg:one}
\begin{algorithmic}[1]
\State \textbf{Given:} behavior cloning policy $\bcpi$
\Statex
\State Roll out $\bcpi$ for $N$ environment steps
\State Estimate $\hat{Q}_{\rm BC}$ using Eq.~\ref{eq:Q_estimate}
\Comment{Q-Estimation}
\State Initialize $Q_{\rm RL}$ with $\hat{Q}_{\rm BC}$, Initialize $\rlpi$
\Statex
\For{$t=0 \dots T$ environment steps}
    \State $a_{\rm BC} \sim \bcpi(s_t)$
    \State $a_{\rm RL} \sim \rlpi(s_t)$

    \State Select $a_t$ using Eq.~\ref{eq:Q_gate}
    \Comment{Q Gating}
        
    \Statex
    \State Observe next state $s_{t+1}$ and reward $r_t$
    \State Store $(s_t, a_t, r_t, s_{t+1})$ in replay buffer $\mathcal{B}$
    \State Sample a minibatch $\{(s, a, r, s')\} \in \mathcal{B}$
    \State Update $Q_{RL}$, $\rlpi$ using off-policy RL
\EndFor
\end{algorithmic}
}
\end{algorithm}

Our objective is to improve performance on task $\mathcal{M}_{\rm test}$ beyond the original BC policy's performance.
To achieve this, we assume access to the test environment for online interactions and a sparse reward signal.
For this online phase, we formulate the task as a Markov Decision Process (MDP) $\mathcal{M} = \left \{ \mathcal{S}, \mathcal{A}, \mathcal{R}, \mathcal{T}, \rho_0, \gamma \right \}$, where $\mathcal{S} \in \mathbb{R}^n$ and $\mathcal{A} \in \mathbb{R}^m$ represent states and actions, $\mathcal{R} : \mathcal{S} \times \mathcal{A} \rightarrow \mathbb{R}$ is the reward function, $\rho_0$ is the probability distribution over initial states, and $\gamma \in [0, 1)$ is a discount factor.

\subsection{Background} Our proposed approach will involve estimating a Q-function, or the expected return from taking an action $a\in \mathcal{A}$ from state $s \in \mathcal{S}$, then following a policy $\pi$ thereafter: 
\begin{equation}
\label{eq:q_fn}
Q_\pi(s, a) = \mathbb{E}_{\tau \sim \pi} \left [ \sum^{\infty}_{t=0} \gamma^t r(s_t, a_t) \mid s_0 = s, a_0 = a \right ].
\end{equation}

The value function for policies is defined as the expectation over Q-values:
\begin{equation} 
\label{eq:value-Qfn}
V(s) = \mathbb{E}_{a \sim \pi}
\left[ Q(s,a) \right].
\end{equation}

Although $Q_\pi(s, a)$ is a standard action-value function, it can also be interpreted as defining an energy-based model over actions. In particular, treating $E(s, a) = -Q(s, a)$ as an energy yields a Boltzmann policy

\begin{equation}
\label{eq:qboltzmann}
\pi(a \mid s) \propto \exp\left(\frac{1}{\alpha}Q(s, a)\right ),
\end{equation}
where $\alpha$ is a temperature parameter.

\section{Q2RL: Q-Estimation and Q-Gating from Behavior Cloning for Reinforcement Learning}
\label{sec:method}

In this section, we present \MethodName{}, an algorithm for estimating Q-values from a BC policy to guide online reinforcement learning.
\MethodName{} consists of two phases: Q-Estimation (Sec.~\ref{sec:Q_estimation}) and Q-Gating (Sec.~\ref{sec:Q_gating}).  
See Alg.~\ref{alg:one} for a high-level description and Fig.~\ref{fig:teaser} for a visual overview.

\subsection{Q-Estimation from a Behavior Cloning Policy}
\label{sec:Q_estimation}

In the first phase of our method, we estimate the Q-function of a pretrained BC policy $\bcpi$.  We first derive an analytic formula for the Q-function in terms of the value function, policy likelihoods, and entropy. 
We assume the action distribution of the human demonstrations $\mathcal{D}$ used to train $\bcpi$ is approximately a Boltzmann distribution, which has been widely used to model human actions~\cite{laidlaw2022boltzmann, luce1959individual, 7139282, bobu2018learning}. 
Assuming the BC policy's action distribution $\bcpi(a \mid s)$ can be expressed as a Boltzmann distribution with respect to $Q_{BC}$ gives:
\begin{equation}
\label{eq:boltzman}
    \bcpi(a \mid s) \approx \frac{\exp\left( \bcQ{}(s,a) / \alpha \right)}{\int\exp\left( \bcQ{}(s,a') / \alpha \right) da'}.
\end{equation}

Using Eq.~\ref{eq:value-Qfn} and Eq.~\ref{eq:boltzman}, we derive the equation for the Q-function $\bcQ$ using the value function, log probability of the action, and entropy:
\begin{equation}
\label{eq:Q_estimate}
    \hat{Q}_{\rm BC} = V_{\rm BC}(s) + \alpha\log\bcpi(a \mid s) + \alpha\mathcal{H}[\bcpi(\cdot \mid s)].
\end{equation}
Appendix \ref{app:derivations} provides a complete derivation for Eq.~\ref{eq:Q_estimate}. 

Next, we detail how each term in Eq.~\ref{eq:Q_estimate} is computed or approximated. 
First, we must estimate the value function $V_{\rm BC}$. 
We collect rollouts of the BC policy for a few initial online interaction episodes $\{\tau_i\}_{i=1}^{N}$ and calculate the Monte Carlo returns under the online reward function:
\begin{equation}
\label{eq:vhat_bc}
\hat{V}_{BC}(s_t^{(i)}) = \frac{1}{N} \sum_{i=1}^{N}
\sum_{k=t}^{\tau_i}
\gamma^{k} r_{t+k}^{(i)},\ \text{for } t \in \{1, \dots, \tau_i\}.
\end{equation}
These Monte Carlo returns are used to train a value estimator.

The remaining terms of Eq.~\ref{eq:Q_estimate} are derived from the BC policy $\bcpi$.
Our method is agnostic to the BC policy class, as long as it provides the log probabilities of actions and entropy.
We provide the remaining components of Eq.~\ref{eq:Q_estimate} using Gaussian policies as an illustrative case.
The log probability of the actions $\log\pi(a \mid s)$ for Gaussian policies has a closed-form expression:
\begin{equation}
\label{eq:gaussian_logprob}
\log \pi (a \mid s) = \frac{1}{2} \sum^d_{i=1} \left[ \frac{(a_i - \mu_i(s))^2}{\sigma^2_i (s)} + \log(2\pi\sigma^2_i(s)) \right],
\end{equation}
where $(\mu, \sigma)$ is Gaussian distribution output by the policy.
Finally, the entropy $\mathcal{H}[\pi(\cdot \mid s)]$ for Gaussian policies also has a closed-form expression:
\begin{equation}
\label{eq:gaussian_entropy}
\mathcal{H}\!\left[ \pi(\cdot \mid s) \right]
= \frac{1}{2}\,\log\!\left( 2\pi e\, \sigma^2(s) \right).
\end{equation}
See Appendix~\ref{app:derivations} for a similar treatment for Gaussian Mixture Models (GMMs). 

\subsection{Q-Gating for Online Reinforcement Learning}
\label{sec:Q_gating}


        

We now describe Q-Gating, our approach to using $\hat{Q}_{\rm BC}$ estimated from Eq.~\ref{eq:Q_estimate} to guide online reinforcement learning. 
We first initialize two separate Q-functions, both from $\hat{Q}_{\rm BC}$: $\hat{Q}_{\rm BC}$ and $Q_{\rm RL}$.  
$\hat{Q}_{\rm BC}$ is kept frozen to serve as a stable reference for the BC policy.
$Q_{\rm RL}$ initialization involves an initial supervised training period to match $\hat{Q}_{\rm BC}$ on the same BC rollouts used for Q-Estimation.
$Q_{\rm RL}$ then serves as the critic for the RL policy, and is updated through online interaction.
During online training, we sample both the BC policy and RL policy for actions $a_{\rm BC}$ and $a_{\rm RL}$, respectively. 
We evaluate $a_{\rm BC}$ using $\hat{Q}_{\rm BC}$ and $a_{\rm RL}$ using $Q_{\rm RL}$ to get their estimated Q-values. 
The action corresponding to the greater of these two values is executed: 
\begin{equation}
\label{eq:Q_gate}
a =
\begin{cases}
a_{\rm BC} \sim \bcpi(s),&  \hat{Q}_{BC}(s, a_{\rm BC}) > Q_{RL}(s, a_{\rm RL}) \\
a_{\rm RL} \sim \rlpi(s_t),& \rm otherwise.
\end{cases}
\end{equation}

\begin{figure*}[htb]
    \centering
    \includegraphics[width=0.9\linewidth]{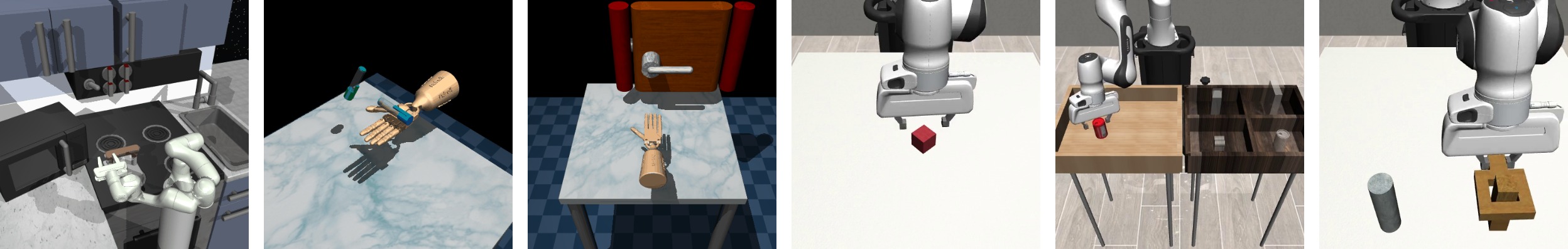}
    \caption{Tasks from D4RL ({\small \texttt{Kitchen-Complete}, \texttt{Adroit-Pen}, \texttt{Adroit-Door}}) and robomimic ({\small \texttt{Lift}, \texttt{Can}}, \texttt{Square}).}
    \label{fig:task-visualization}
\end{figure*}
\begin{figure*}[ht]
    \centering
    \includegraphics[width=0.8\linewidth]{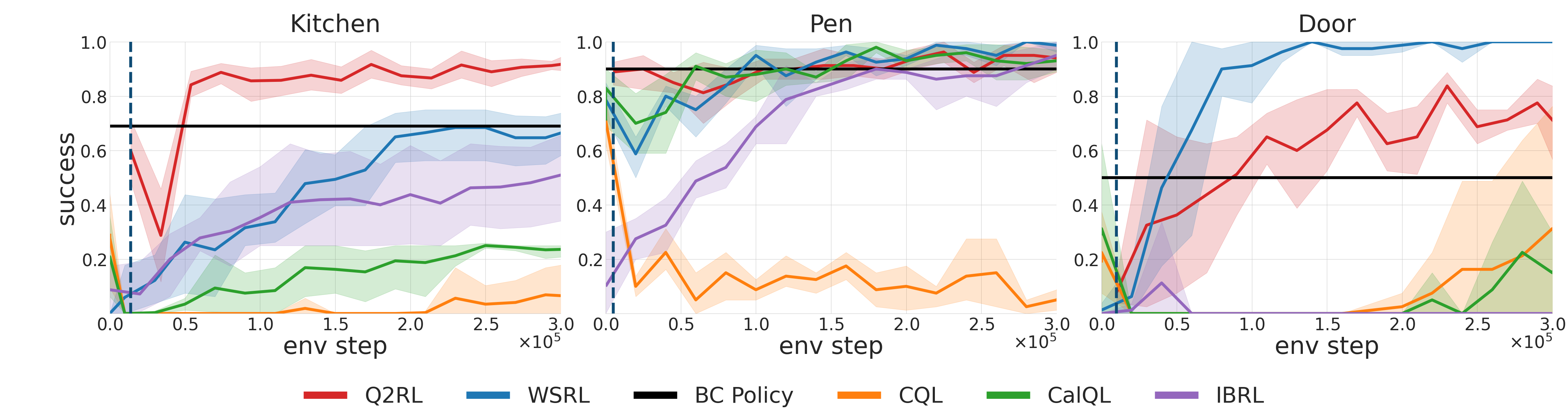}
    \caption{Results on D4RL. 
    Plots begin when a method starts online interaction.
The dotted blue line signifies the end of the Q-Estimation phase and the start of online RL with Q-Gating for our method.
Results are reported over 20 trajectories and 5 seeds, with 95\% confidence intervals.}
    \label{fig:d4rl}
\end{figure*}

We propose this Q-Gating mechanism to preserve good BC action proposals while still enabling policy improvement through RL.
$\hat{Q}_{\rm BC}$ represents the (estimated) value of taking an action in the current state, then following the BC policy thereafter.
We evaluate BC actions through a frozen $\hat{Q}_{\rm BC}$, so that BC actions with high value under $\hat{Q}_{\rm BC}$ are likely to be executed. 
On the other hand, $Q_{\rm RL}$ is only initialized as $\hat{Q}_{\rm BC}$, and trains on samples with either BC or RL actions. 
$Q_{\rm RL}$ therefore improves its estimate of both BC and RL actions, enabling the reinforcement learning policy to surpass BC performance through online exploration and interaction.
Note that our approach does not require access to the training data for $\bcpi$; we rely on our Q-estimate $\hat{Q}_{\rm BC}$ to bootstrap online reinforcement learning.

\subsection{Implementation Details}

We assume a temperature of $\alpha=1$ for the BC policy's Boltzmann parameterization in Eq.~\ref{eq:Q_estimate}.
\MethodName{} can be applied to any behavior cloning policy that provides log probabilities of actions and 
entropy; in this work, we demonstrate the approach using Gaussian policies and GMMs.
Similarly, though \MethodName{} is agnostic to the choice of off-policy RL algorithm, we use SAC optimized for real-world learning~\cite{haarnoja2018soft, zhou2025efficient}.


To prevent the RL policy from deviating too far from the BC state-action distribution, we also weight the RL actor loss with an auxiliary BC loss~\cite{Rajeswaran-RSS-18}.
While this regularization term can occasionally prevent the policy from reaching perfect success rates, our experiments in the following section show that training with BC regularization still outperforms baselines and plays a crucial role in producing safe, smooth behaviors on real robotic systems. 

\section{Experiments}

Our experiments are designed to answer two key questions: (i) does \MethodName{} improve performance beyond the BC policy faster and better than baseline approaches? (ii) Does \MethodName{} use high-value actions from the BC policy during online RL?
To provide a comprehensive evaluation, we conduct experiments across multiple simulation environments using BC policies parameterized by Gaussian and GMM-RNN architectures, and compare against offline-to-online RL and BC-to-RL baselines. 
Our evaluation spans both state-based and image-based observation settings. 
Finally, we validate the practicality of our approach through real-world RL experiments, assessing whether \MethodName{} can support online learning on physical hardware.
\label{sec:experiments}

\subsection{Simulation Environments}

We test \MethodName{} on tasks from the D4RL~\cite{fu2020d4rl} and robomimic~\cite{mandlekar2021matters} benchmarks. All tasks are visualized in Fig~\ref{fig:task-visualization}.
\subsubsection{\textbf{D4RL}} We run experiments on two Adroit manipulation tasks: \textit{Pen} and \textit{Door}. 
The Pen task involves dexterous in-hand manipulation, requiring the robotic hand to reorient a pen to a target configuration, while Door consists of a standard door-opening task. 
In addition, we evaluate performance on the \textit{Kitchen} environment, which features a Franka robot interacting with multiple objects to achieve a multi-stage goal configuration. 
For these experiments, we use the offline Kitchen-Complete dataset, which contains 19 successful demonstrations which are representative of standard robot learning setups where we only have access to a few successful trajectories.
We conduct all D4RL experiments without access to offline training data during the online RL phase.
\subsubsection{\textbf{robomimic}} We run experiments on three widely used robomimic tasks, \textit{Lift}, \textit{Can}, and \textit{Square}, all featuring a Franka robotic arm. In the Lift task, the robot must lift a cube from a randomly sampled position on the table. The Can task requires the robot to grasp a can from one side of the table and place it into the correct bin among four target bins located on the opposite side. The Square task involves picking up a square nut from the table and inserting it onto a square peg. 
For robomimic tasks, we test both with and without access to offline training data during the online RL phase.

\begin{figure*}[ht!]
    \centering
    \subfloat[Results with data seeded in the online replay buffer.\label{fig:rm_with_data_all}]{
        \includegraphics[width=\linewidth]{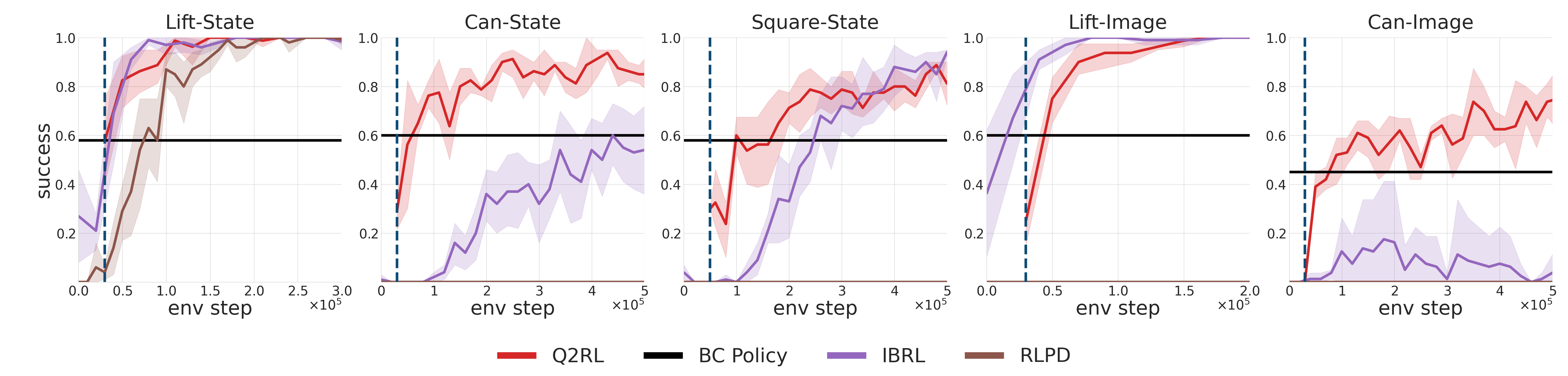}
    }
    \hfill
    \subfloat[Results without data seeded in the online replay buffer. \label{fig:rm_without_data_all}]{
        \includegraphics[width=\linewidth]{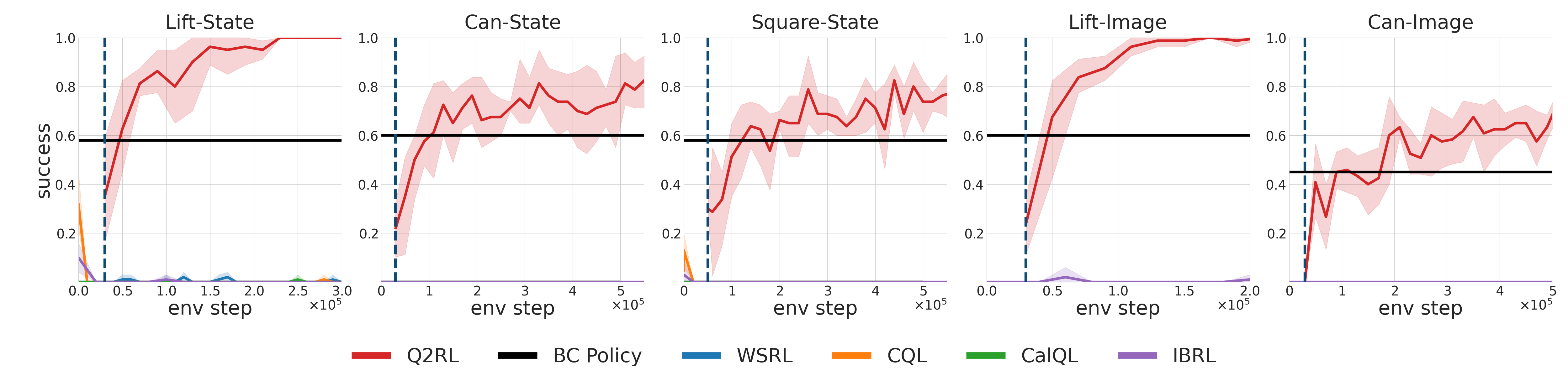}
    }
    \caption{Robomimic Results. 
    The plot for each method begins at the start of online interaction.
    The dotted blue line signifies the end of the Q-Estimation phase and the start of online RL with Q-Gating for our method.
    Results are reported over 20 trajectories and 5 seeds, with 95\% confidence intervals.
    }
    \label{fig:rm}
\end{figure*}

\subsection{Baselines} \label{sec:baselines}
We compare \MethodName{} with the following baselines: 
\subsubsection{\textbf{CQL}} Conservative Q-Learning~\cite{kumar2020conservative} is an offline RL method that trains exclusively on an offline dataset by prioritizing in-distribution actions while penalizing out-of-distribution actions. This induces a pessimistic bias in the learned Q-function, which often requires a recalibration phase during online interaction.
\subsubsection{\textbf{CalQL}} Calibrated Q-Learning~\cite{nakamoto2023cal} addresses the pessimism of Q-value estimates in CQL by calibrating the critic loss with respect to the value of a reference policy. However, despite this calibration, transitioning to online training can still lead to unlearning of its offline pretraining.
\subsubsection{\textbf{WSRL}} Warm-Start RL~\cite{zhou2025efficient} performs for offline-to-online RL without explicit data retention. 
To mitigate unlearning during the transition from offline to online training, 
WSRL introduces a warm-up phase in which the online replay buffer is seeded with rollouts generated by the offline pretrained policy. 
However, a key limitation of WSRL is its dependence on the quality of the offline pretraining data.
\subsubsection{\textbf{RLPD}} RL with Prior Demonstrations~\cite{ball2023efficient} incorporates offline data directly into online learning by uniformly sampling from both offline and online replay buffers during actor and critic updates. 
\subsubsection{\textbf{IBRL}} Imitation Bootstrapped RL~\cite{hu2023imitation} uses a pre-trained BC policy to bootstrap Q-value estimates and action proposals for RL. 
IBRL randomly initializes the Q-function and uses demonstration data to seed the online replay buffer.

Additional details on simulation experiments can be found in Appendix~\ref{sec:app_sim_details}.

\subsection{Results on D4RL}
We present the results for Kitchen, Door and Pen tasks from D4RL in Fig~\ref{fig:d4rl}. 
Offline RL methods perform poorly on the \textbf{Kitchen} task because the offline Kitchen-Complete dataset only contains successful demonstrations. 
Even with online RL after offline pre-training, these methods fail to improve over BC for the Kitchen task. 
\MethodName{}, on the other hand, starts with better initial performance than the baselines and improves beyond the BC policy, which we attribute to Q-Estimation and online RL with Q-Gating.

In the \textbf{Pen} task, the offline data is able to provide sufficient pretraining for both BC and offline RL methods. 
We observe that \MethodName{} starts with near-optimal performance and maintains it throughout online learning.
WSRL and CalQL are also able to reach high performance with online interaction on Pen, whereas CQL experiences a drop in performance, likely due to its conservative Q-value estimates.
IBRL is also eventually able to reach high performance. 

For the \textbf{Door} task, WSRL converges to a higher success rate than the baselines and \MethodName{}.
However, from the rollout videos we observe that the behavior of the WSRL policy diverged from the original human demonstrations, and that it has learned to exploit the simulator through online RL, resulting in actions that would not be feasible in the real world.
(Fig.~\ref{fig:qualitative_door}, Appendix~\ref{sec:app_ablations}). 
\MethodName{} uses Q-Estimation and Q-Gating to remain close to the BC policy, resulting in actions that can be realistically executed in the real world.

\subsection{Robomimic Results}
Next, we evaluate \MethodName{} on both state-based and image-based robomimic tasks. We use two cameras for image-based tasks, the agentview of the workspace and wrist camera. We implement \MethodName{} and all baselines with a learnable CNN encoder for the images to keep the implementation standard and comparisons fair. As these tasks are more challenging than the D4RL benchmarks, we consider two experimental settings: one in which training data is available in the online replay buffer, and another in which the policy does not have access to any offline training data during online learning. 

\subsubsection{Access to BC Training Data}
Results for the robomimic tasks with access to training data during online learning are shown in Fig.~\ref{fig:rm_with_data_all}. When training data is available in the online replay buffer, IBRL performs competitively. 
IBRL randomly initializes the critic for off-policy RL and has no mechanism to preserve the initial performance of the BC policy, causing online interaction to begin with limited or no meaningful performance. 
In contrast, \MethodName{} begins online interaction with non-trivial initial performance rather than starting from near-zero success, which is particularly important for real-world deployment where unsafe or unproductive exploration is costly. Moreover, by restricting learning to targeted policy improvements, \MethodName{} converges to high-performing policies more rapidly.
We also observe that for Can-Image, a harder high-dimensional task, IBRL is unable to reach base policy performance. 
We note that in the original IBRL paper, IBRL is able to converge for the Can-Image task, but we hypothesize that this is due to other optimizations used, such as specialized ViT encoders.  
\MethodName{} on the other hand is able to outperform the base policy with only the standard set of encoders. 
RLPD succeeds only on simpler tasks such as Lift and struggles to converge on the more challenging Can and Square tasks.

\subsubsection{No Access to BC Training Data}
BC policies are often provided by third parties without the data or infrastructure used to train them.
Due to such occurrences, we evaluated our approach in cases where pre-training data was not accessible. 
Results for the robomimic tasks without access to training data during online learning are shown in Fig.~\ref{fig:rm_without_data_all}.
Offline RL and offline-to-online RL methods are unable to learn for harder robotic tasks with limited offline datasets.
Without access to data, IBRL is also unable to learn any meaningful performance.
Early in training, the randomly initialized RL policy can propose arbitrary actions that may receive spuriously high Q-values from an untrained critic. 
To mitigate this issue, IBRL relies on seeding the replay buffer with training data, so that the critic learns to assign higher value to BC action proposals than to random RL actions. 
In the absence of a seeded replay buffer, IBRL loses this mechanism to stabilize initial Q-values, leading to unreliable action selection. 
\MethodName{} is able to continually improve performance even in the absence of training data.
We attribute \MethodName{}'s strong performance without seeding to the effectiveness of the proposed Q-Estimation step.
Q-Estimation computes $Q_{BC}$ using BC policy rollouts \emph{only}; it does \emph{not} require expert demonstrations, nor a diverse offline dataset.
\MethodName{} therefore has broad applicability in both limited access, memory-constrained settings where expert data is unavailable, as well as in unconstrained settings where expert data can seed online RL. 
We also conduct a thorough evaluation for seeding the replay buffer with different amounts of data in 
Appendix~\ref{sec:app_ablations}.

\subsection{Analysis of BC Actions Used During Online Training}
Next, we investigate the relationship between success rate and the ratio of BC vs. RL actions used during online training.
Fig.~\ref{fig:base_action_captured} shows the success rate and the fraction of BC actions selected by \MethodName{} vs. IBRL on Can-State and Square-State tasks. 
The vertical dotted line serves as a visual reference indicating when \MethodName{}'s online performance approximately matches the original BC policy performance. 

\begin{figure}[th]
  \centering
  \begin{subfigure}[t]{\linewidth}
    \centering
    \includegraphics[width=\linewidth]{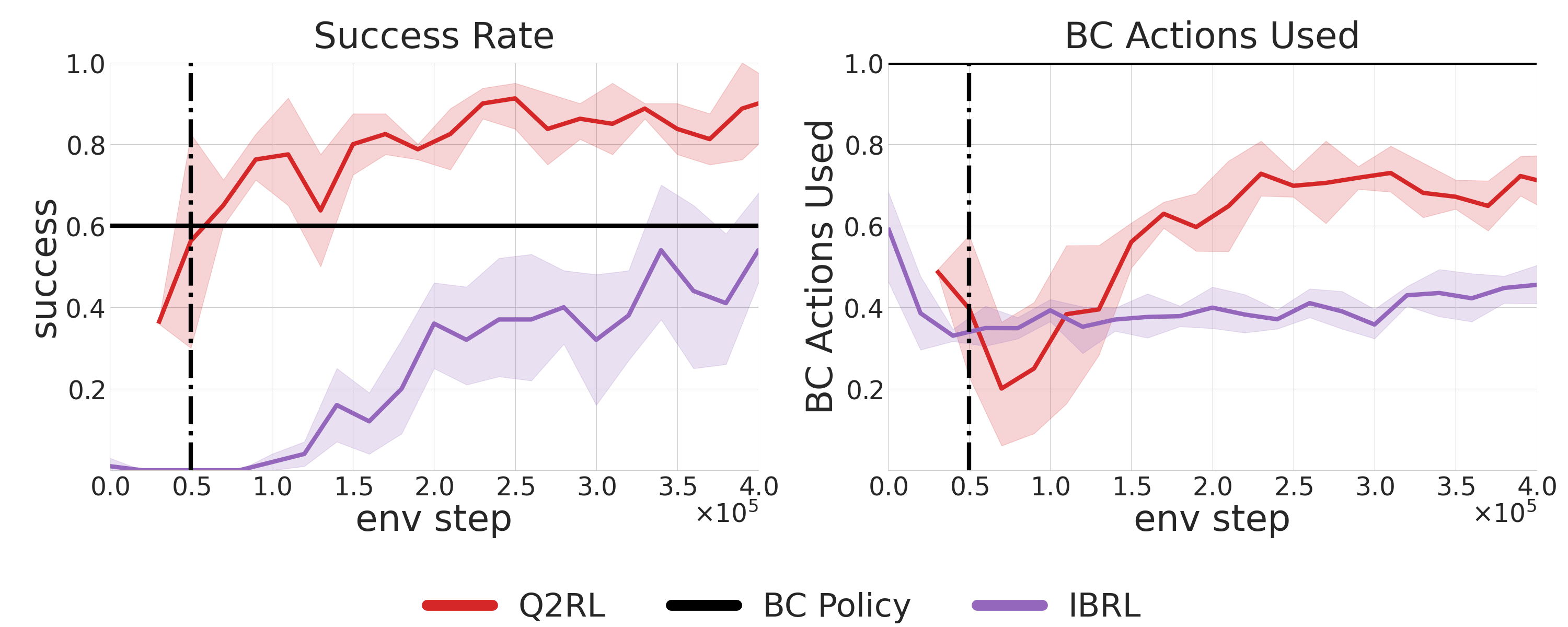}
    \caption{Can-State}
  \end{subfigure}

  \begin{subfigure}[t]{\linewidth}
    \centering
    \includegraphics[width=\linewidth]{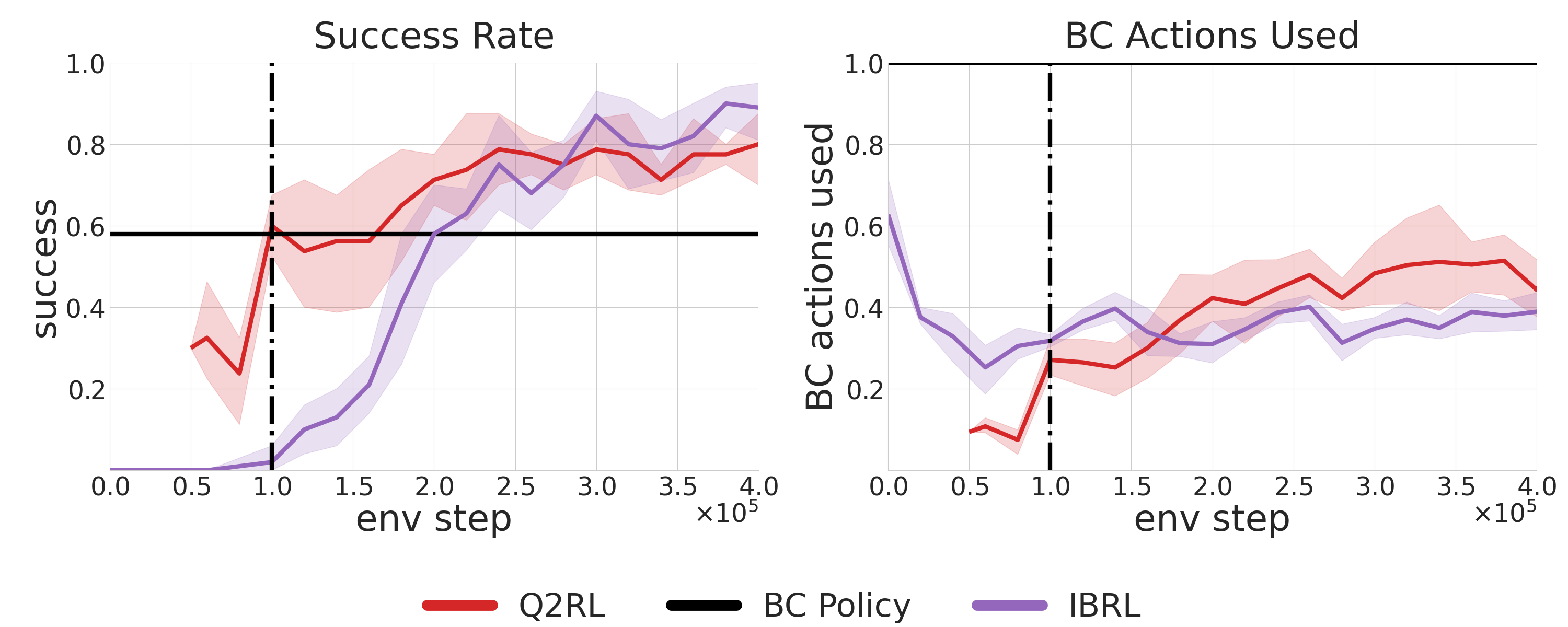}
    \caption{Square-State}
  \end{subfigure}

  \caption{BC vs. RL Actions Used During Online Training. \MethodName{} optimizes BC vs. RL action selection via Q-Gating, and achieves higher success rate earlier than IBRL.}
  \label{fig:base_action_captured}
  \vspace{-1em}
\end{figure}

We observe that \MethodName{} rapidly recovers the original BC performance within only a few online interaction steps, demonstrating that it identified and exploited high-value actions of the BC policy.
This performance is achieved while using BC actions less than half of the time, indicating that the RL policy also learned to produce effective actions.
With continued training, 
as $Q_{\rm RL}$ estimates improve, Q-Gating more reliably identifies states where each policy has strong performance, so that Q2RL relies on BC for simple motions and reserves RL for difficult, contact-rich segments.

In contrast, IBRL selects a similar fraction of BC actions, but requires substantially more interaction to reach the same performance.
Further, the ratio of BC actions used by IBRL during training stays relatively flat, indicating that it trades off between BC and RL actions less efficiently than \MethodName{}.

\subsection{Ablations}

\label{sec:ablation_gating_init}
We ablate Q-Initialization and Q-Gating in Fig.~\ref{fig:ablation_our_method}. 
Note that Q-Initialization, or initializing $Q_{RL}$ with $\hat{Q}_{BC}$, is distinct from Q-Estimation, which is the algorithm for estimating $\hat{Q}_{BC}$ (Sec.~\ref{sec:Q_estimation}).
In the No Q-Gating condition, BC and RL action proposals are evaluated using the same Q-function during online RL.  
We see that Q-Gating is critical for achieving high performance. 
When Q-Gating is enabled, Q-Initialization
yields performance comparable to a randomly initialized critic; however, we retain this initialization because it provides a safer starting distribution that remains closer to the BC policy, which is preferable for on-robot RL.
See Appendix~\ref{sec:app_ablations} for additional ablation studies.

\begin{figure}[t]
    \centering
    \includegraphics[width=\linewidth]{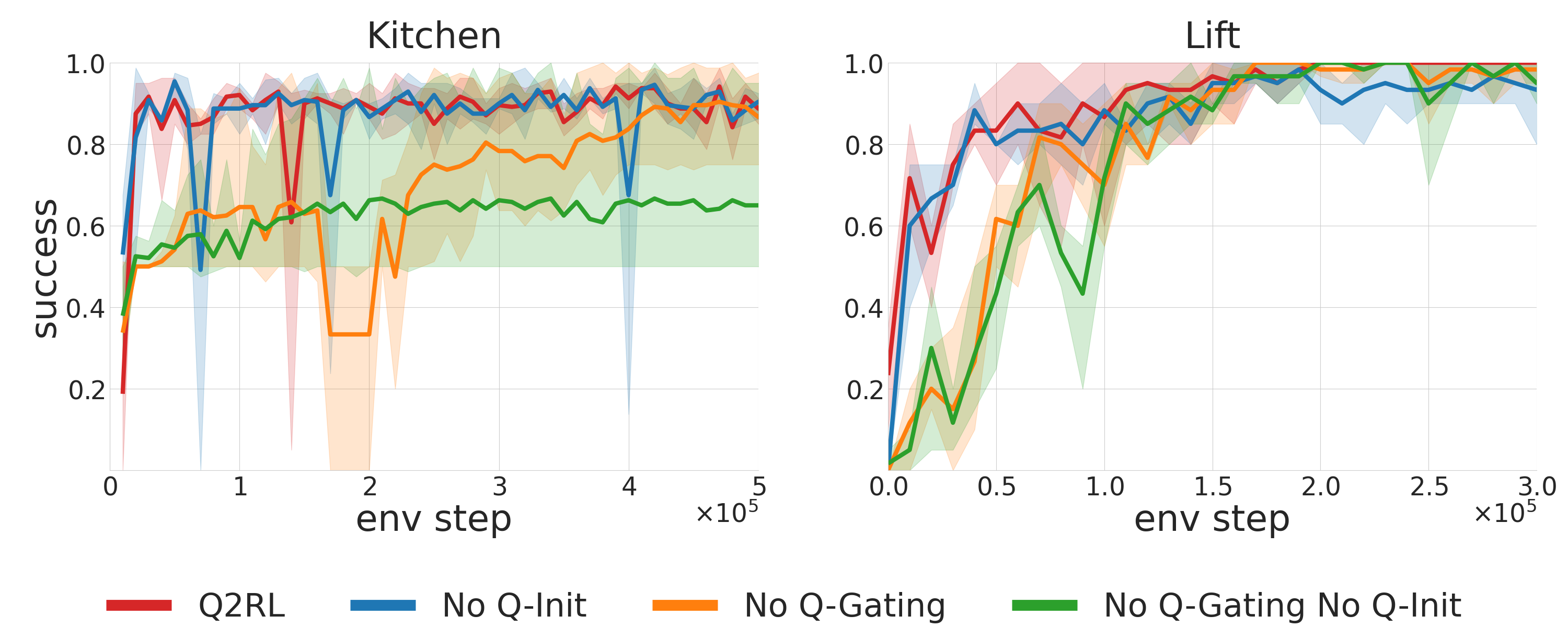}
    \caption{Q-Initialization and Q-Gating Ablations.}
    \label{fig:ablation_our_method}
    \vspace{-1em}
\end{figure}

\begin{figure*}[ht]
    \centering
    \subfloat[On-Robot RL Setup.\label{fig:real_setup}]{
        \includegraphics[width=0.25\linewidth, 
        trim=300pt 50pt 0 0, clip
        ]{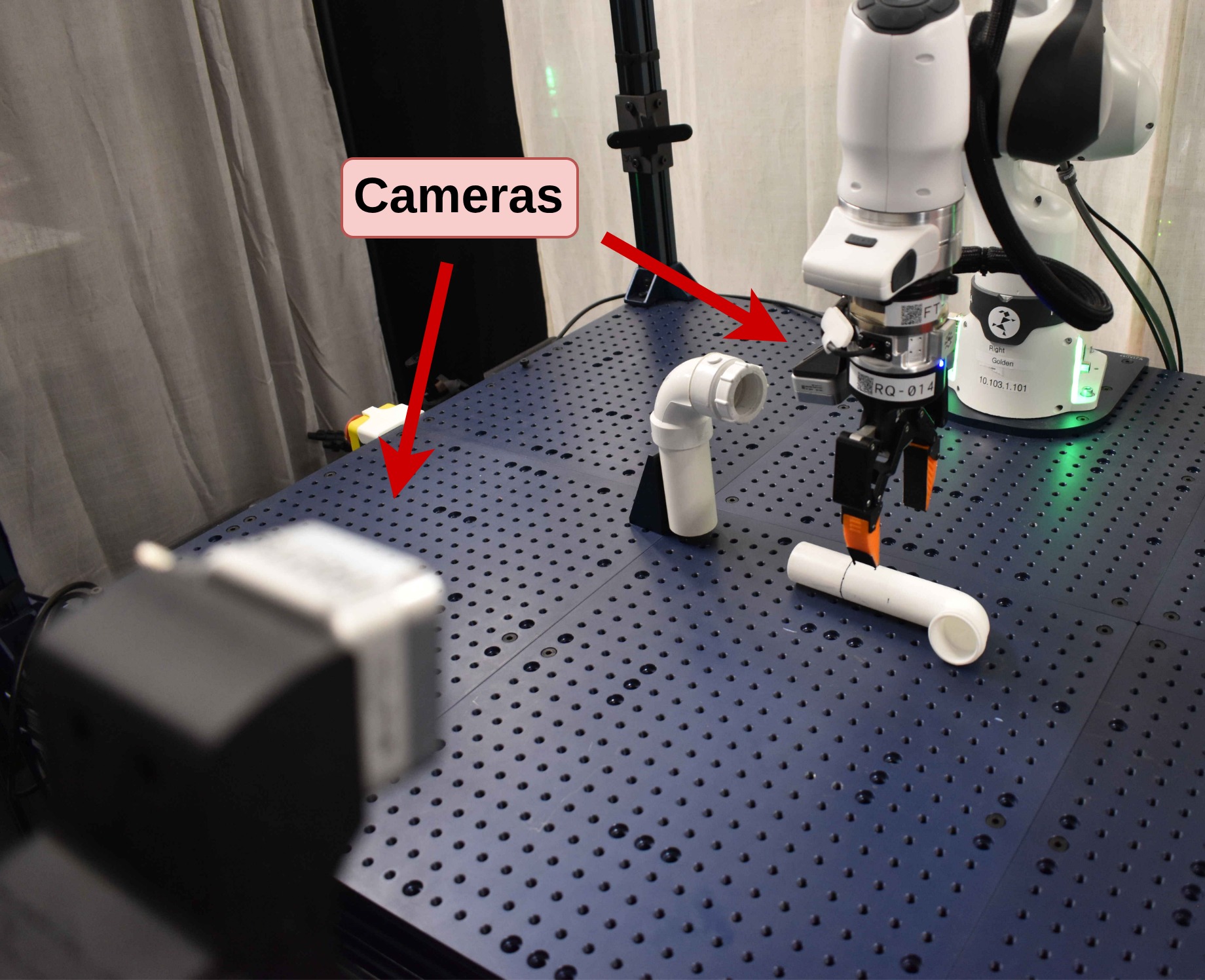}
    }
    \hfill
    \subfloat[Trained \MethodName{} policies for pipe assembly and kitting tasks. 
    \label{fig:timelapse}]{
        \includegraphics[width=0.6\linewidth]{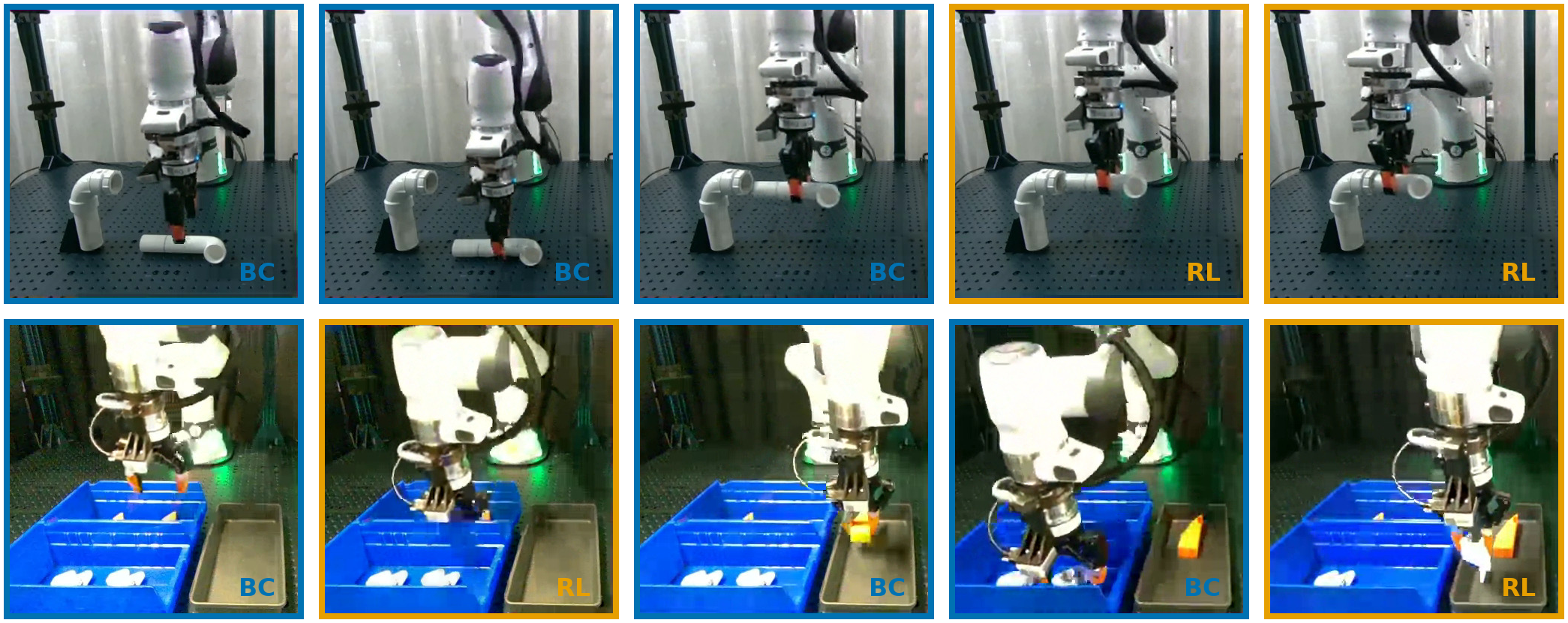}
    }
    \caption{Real World Experiments. (a) A Franka Panda arm with a Robotiq 2F-85 gripper, with RGB images from wrist and workspace Intel RealSense D405s. (b) We analyze usage of BC (\textcolor{myblue}{\textbf{blue}}) vs. RL (\textcolor{myyellow}{\textbf{yellow}}) actions in representative trajectories executed by trained \MethodName{} policies. For pipe assembly, \MethodName{} relies on BC for grasping and initial alignment, then switches to RL for high-precision, contact-rich insertion. For kitting, the BC policy was trained in a task setting with only one object present in each bin; \MethodName{} learns to complete the task with multiple objects in each bin. In the modified setting, \MethodName{} uses BC actions for behaviors common to both task settings (e.g. moving between bins), and uses RL for part grasping and placement. 
    }
    \label{fig:real_images}
    \vspace{-0.5em}
\end{figure*}

\begin{figure}[htb]
    \centering
    \includegraphics[width=\linewidth]{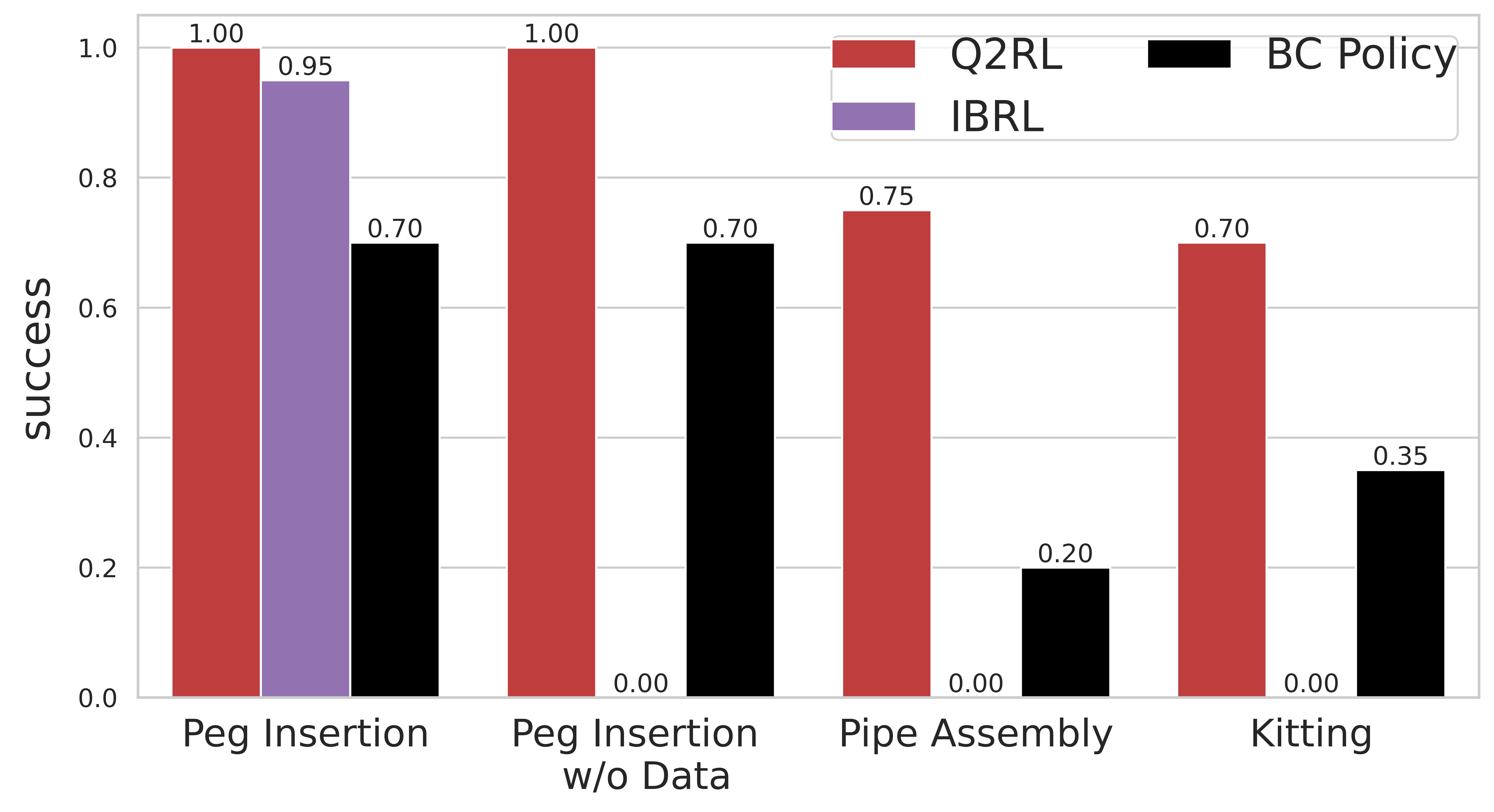}
    \caption{On-Robot RL Results. Task success is reported over 20 trials per method. \MethodName{} significantly improves BC policy performance on every task. IBRL fails when (1) demonstrations are not seeded in the replay buffer (w/o Data), and (2) when tasks involve large action spaces and long horizon manipulation (even with seeding).}
    \label{fig:real_exps}
    \vspace{-0.5em}
\end{figure}

\subsection{Real World RL Experiments}
\subsubsection{Experiment Setup}
We run all experiments on a 7-DOF Franka Panda Robot with a Robotiq 2F-85 gripper. 
For all tasks, the robot receives the end effector pose, gripper width, and $84 \times 84$ RGB images from a workspace camera and wrist camera (see Fig.~\ref{fig:real_setup}).
The robot outputs delta pose actions executed by a Cartesian impedance controller.
We evaluate on the following tasks (see Fig.~\ref{fig:teaser}):

\textbf{Peg Insertion.} This insertion task from FMB~\cite{luo2025fmb} has been used in recent on-robot RL methods~\cite{luo2024serl, zhou2025efficient}.
Success occurs when the peg is fully inserted into the board.

\textbf{Plumbing Pipe Assembly.} This task involves both grasping and low tolerance insertion.
Insertion requires rotating the gripper to align the pipe with the fixture.
Success occurs when the pipe is fully inserted into the fixture.

\textbf{Kitting with Task Distribution Shift.}
In this task, two parts must be retrieved from their respective bins and placed in a kitting tray.
As the longest horizon task, this task requires multiple picks and places.
Success occurs when the robot has placed both parts in the kitting tray.
The BC policy for this task was only trained on demonstrations with a single object per bin. 
To evaluate robustness to task distribution shift, the test task setting has two objects per bin, placed in new locations.

GMM-RNN BC policies were trained for each task on 50 to 100 demos collected by a human operator using a 3D Connexion Spacemouse.
We compare \MethodName{} with IBRL for offline-to-online improvement with the pre-trained BC policies.
During online learning, a termination signal and sparse reward signal were provided by a human operator when the task's success conditions were met. 
Environment resets were also handled by a human operator.
Additional details on the real experiment setup can be found in Appendix~\ref{app:real_exps}.

\subsubsection{Real World Results}
For the Kitting task, all results are for the modified setting (two parts in each bin instead of one). 
The best checkpoint after up to 2.5 hours of online training is used for each method, representing no more than 170k gradient steps and 80k environment steps (much fewer than in simulation).
From Fig.~\ref{fig:real_exps}, \textbf{\MethodName{} significantly improves performance compared to the original BC policy, achieving a 1.4x improvement and 100\% success on the Peg Insertion task, 3.75x improvement on Pipe Assembly, and 2x improvement on Modified Kitting.}
While IBRL performed on par with \MethodName{} for the Peg Insertion task, it achieved zero success when demonstration samples were not seeded in the online replay buffer. 
Even with data seeding, when tasks involved large action spaces or were long horizon, IBRL had zero success, failing to recover BC performance.
We attribute IBRL's poor performance to its use of a single, randomly initialized Q-function to select BC vs. RL actions, which hinders it from accurately rating BC actions, especially for long horizon, sparse reward tasks.
In contrast, we initialize $Q_{\rm RL}$ from $\hat{Q}_{\rm BC}$, and our proposed Q-Gating procedure uses both a frozen $\hat{Q}_{\rm BC}$ and trainable $Q_{\rm RL}$ to score respective actions more accurately.

\textbf{Safety.} During the course of training IBRL for Peg Insertion, we noted two safety violations, in which the Franka exerted excessive force against the board and caused the robot to fault. 
These safety violations occurred despite careful tuning of controller gains during preliminary testing that was subsequently kept fixed for all reported runs.
\MethodName{} experienced no such safety violations, and we qualitatively observed that it produced relatively smoother and safer actions early in training: we attribute this behavior to a combination of the auxiliary BC loss, Q-Gating, and $Q_{\rm RL}$ initialization with $\hat{Q}_{\rm BC}$.

\textbf{Qualitative Results.}
Fig.~\ref{fig:timelapse} shows frames from representative rollouts of trained \MethodName{} policies, annotated to indicate BC vs. RL action selection.  
We observe that \MethodName{} selects BC actions to perform behaviors supported by the BC policy's training data, e.g. initial pipe grasping and alignment for the Pipe Assembly task, before switching to RL actions for high-precision, contact-rich insertion.
With Modified Kitting, BC actions were selected for moving between bins, with RL actions handling parts of the task that had shifted from their original setting, e.g. grasping parts from new locations. 
Other cases of switching between BC and RL actions corresponded to recovery behaviors, such as insertion re-alignment and re-grasping. 
Not all rollouts or trained \MethodName{} policies exhibit these behaviors exactly as described here, but we generally find meaningful switching between BC and RL actions.
Please refer to the website for videos of annotated rollouts, and Appendix~\ref{app:real_exps} for additional details on results.

\section{Conclusion} 
\label{sec:conclusion}

We presented \MethodName{}, an algorithm that improves the performance of a behavior cloning policy through online reinforcement learning. 
\MethodName{} estimates a Q-function from the BC policy, then uses it to initialize and guide online RL. 
Our approach does not require access to the BC policy's original training data, and is compatible with any BC policy class that provides action likelihoods and entropy.
In on-robot reinforcement learning experiments, \MethodName{} successfully improves performance on high precision, contact-rich manipulation tasks within a few hours of online interaction.

One limitation of our work is that \MethodName{} requires the BC policy to provide action likelihoods and entropy; as future work, we aim to extend \MethodName{} to other policy classes that do not provide these by default, such as diffusion and flow matching policies. 

\section*{Acknowledgments}
This work was supported in part by ONR grant N00014-22-1-2592 and ONR grant ``Long-Term Autonomy for Ground and Aquatic Robotics" N00014-24-1-2784.
The authors thank Justin Rojas, Will Heitman, and Steve Proulx for assistance with real robot experiments and hardware, Victoria Coleman and Colin Kohler for real robot infrastructure support, and Eric Rosen for helpful paper feedback.


\clearpage
\appendices
\section{Derivations}
\label{app:derivations}

\subsection{Derivation for Estimating Q-values for BC: $Q_{\rm BC}$}
\label{app:q_derivation}

Assuming $\pi_{BC}$ is a Boltzmann policy that is defined by the value function $Q$ as follows,
\begin{align}
\label{eq:boltz_policy}
    \pi_{BC}(a \mid s) = &\exp \frac{1}{\alpha} Q(s, a) / Z(s), \\
    {\rm where}\ Z(s) = \int_{\mathcal{A}} &\exp \frac{1}{\alpha} Q(s, a') da', \ a \in \mathcal{A}
\end{align}
where $\alpha$ is a temperature parameter, then we can re-arrange the above equations to estimate $Q(s, a)$ using only the state-value and the policy distribution.

\begin{lemma}
    Given a Boltzmann policy $\pi_{BC}$ and assuming $Z(s)$ is finite for all $s \in \mathcal{S}$, we can express the Boltzmann value function $Q(s, a)$ of $\pi_{BC}$ as the following:
    \begin{align}
    Q(s, a) = V(s) + \alpha\log \pi_{BC}(a \mid s) + \alpha\mathcal{H}[\pi_{BC}(\cdot \mid s)],
    \end{align}
    for all $s \in \mathcal{S}$ and $a \in \mathcal{A}$.
\end{lemma}
\begin{proof}
    Firstly, we derive an equation for the partition function $\log Z$ using the policy entropy $\mathcal{H}$ and Eq.~\ref{eq:boltz_policy}:
    \begin{align*}
        \mathcal{H}[\pi_{BC}(\cdot \mid s)] &= - \mathbb{E}_{a \sim \pi_{BC}(\cdot \mid s)}[\log \pi_{BC}(a \mid s)] \\
        &= - \mathbb{E}_{a \sim \pi_{BC}(\cdot \mid s)}\left[\frac{1}{\alpha}Q(s, a) - \log Z(s)\right] \\
        &= - \frac{1}{\alpha}V(s) + \log Z(s).
    \end{align*}
    Rearranging,
    \begin{align*}
        \log Z(s) &= \frac{1}{\alpha}V(s) + \mathcal{H}[\pi_{BC}(\cdot \mid s)]
    \end{align*}
    Secondly, we substitute this equation into the definition of our Boltzmann policy,
    \begin{align*}
        \log \pi_{BC}(a \mid s) &= \frac{1}{\alpha} Q(s, a) - \log Z(s) \\
        &= \frac{1}{\alpha}Q(s, a) - \frac{1}{\alpha}V(s) - \mathcal{H}[\pi_{BC}(\cdot \mid s)].
    \end{align*}
    So
    \begin{align*}
        Q(s, a) &= V(s) + \alpha\log \pi_{BC}(a \mid s) + \alpha\mathcal{H}[\pi_{BC}(\cdot \mid s)]
    \end{align*}
    as desired.
\end{proof}

\subsection{Derivations for Gaussian Mixture Models}
In this section, we provide additional details on using \MethodName{} with Gaussian Mixture Model (GMM) based BC policies. 
\subsubsection{Action Likelihoods}
The probability density function of a GMM with $C$ components is defined as the weighted sum of the probability density of each of the mixture component $\pi_i$~\cite{kolchinsky2017estimating}: 
\begin{equation}
\label{eq:gmm_logprob}
 \pi (a \mid s) = \sum_{i=1}^{N}c_{i} \pi_{i} (a \mid s),
\end{equation}
where $c_i$ is the weight of each component $i$ ($c_{i} \geq 0, \sum_{i}c_{i}=1$).

\subsubsection{Entropy}
For a GMM, $\mathcal{H}[\pi(\cdot \mid s)]$ is upper-bounded by the sum of expected component entropy and discrete entropy of mixture weights~\cite{kolchinsky2017estimating}: 
\begin{equation}
    \mathcal{H}[ \pi (\cdot \mid s)] \leq  \mathcal{H}[ \pi (\cdot \mid s) \mid C] + \mathcal{H}(C)
\end{equation}
where the expected component entropy can be defined as:
\begin{equation}
    \mathcal{H}[\pi (\cdot \mid s) \mid C] = \sum_{i}c_i\mathcal{H}[\pi_{i} (\cdot \mid s)]
\end{equation}

We use this upper bound of entropy for our GMM policy experiments.

\section{Details on Simulation Experiments}
\label{sec:app_sim_details}

We define the BC policy hyperparameters in Table~\ref{tab:bc_hyperparameters} along with the BC loss weights.
Table~\ref{tab:rl_hyperparameters} contains hyperparameters of the RL agent for \MethodName{}. 
We use the same RL hyperparameters for IBRL and WSRL as used in their respective papers. 

\vspace{0.5em}
\begin{table}[htb]
    \centering
    \caption{BC Epochs and Auxiliary Loss Weights for RL}
    \begin{tabular}{ccccc}
        \toprule
         & BC Loss Weight & BC Epoch Used \\
        \midrule
        Kitchen & 0.3 & 200k \\
        Pen & 0.3 &  200k \\
        Door & 0.3 & 100k \\
        \midrule
        Lift-State & 0.3 & 350 \\
        Can-State & 0.3 & 250 \\
        Square-State & 0.2 & 1000\\
        Lift-Image & 0.2 & 20 \\
        Can-Image & 0.2 & 15 \\
        \midrule
        Peg Insertion & 0.1 & 600 \\
        Pipe Assembly  & 0.1 & 600 \\
        Kitting (Modified)  & 0.1 & 600 \\
        \bottomrule
    \end{tabular}
    \label{tab:bc_hyperparameters}
    \vspace{0.5em}
\end{table}

\subsection{Simulation Hyperparameters}

We use the same reward shaping from WSRL, namely a reward bias $b$ at each timestep and scale factor $a$ for nonzero rewards.
The transformed reward is:
\begin{equation}
    \tilde{r} = ar + b
\end{equation}
In our experiments, the only nonzero environment reward $r$ is the sparse reward given upon success ($r=1$).
For D4RL tasks, we use the reward scale and bias values from WSRL. 
For all other tasks (including Kitchen), bias $b=-1$, a negative per-step penalty that encourages the agent to complete the task faster.
Scale factor $a$ helps with TD error propagation and its value in our experiments differs between simulation and the real world (see Table~\ref{tab:rl_hyperparameters}).

We use the same Soft Actor-Critic implementation as WSRL~\cite{zhou2025efficient} as our reinforcement learning algorithm.
We use DrQ image augmentation for image-based environments as implemented in SERL~\cite{yarats2021image, luo2024serl}.

\begin{table*}[tb]
    \centering
    \caption{Reinforcement Learning Agent Hyperparameters}
    \begin{tabular}{c|c|c|c}
        \toprule
         & Simulation (State) & Simulation (Image) & Real-world \\
        \midrule
        Batch Size & 256 (Kitchen: 1024) & 256  & 256 \\
        Hidden Dimensions (Actor and Critic) & [512, 512, 512] & [1024, 1024] & [1024, 1024] \\
        Learning Rate (Actor and Critic) & 3e-4 & 1e-4 & 1e-4 \\
        Reward Scale & 5 (Adroit: 10, Kitchen: 4) & 5 &10 \\
        Num. Rollouts for Q-Estimation Phase & 50 (D4RL), 100 (Robomimic) & 100 & See Table~\ref{tab:task_config} \\
        \midrule
        Reward Bias & \multicolumn{3}{c}{-1 (Adroit: 5)} \\
        Num. Q-Estimation Training Steps & \multicolumn{3}{c}{20k (50k for Square)} \\
        UTD Ratio & \multicolumn{3}{c}{4} \\
        Critic Ensemble Size & \multicolumn{3}{c}{10} \\
        Critic Subsample Size & \multicolumn{3}{c}{2} \\
        Discount Factor & \multicolumn{3}{c}{0.99} \\
        Soft Target Update Rate & \multicolumn{3}{c}{0.005} \\
        
        Layer Normalization & \multicolumn{3}{c}{Yes} \\
        Replay Buffer Size & \multicolumn{3}{c}{2e6} \\
        
        \bottomrule
    \end{tabular}
    \label{tab:rl_hyperparameters}
\end{table*}

\begin{table*}[htb]
    \centering
    \footnotesize
    \caption{Simulation Results - State}
    \label{tab:sim_results_state}
    \setlength{\tabcolsep}{2.5pt} 

    \begin{tabular}{l ccc  ccc  ccc}
    \toprule
    Without Data & \multicolumn{3}{c}{Kitchen}
    & \multicolumn{3}{c}{Pen}
    & \multicolumn{3}{c}{Door} \\
    \cmidrule(lr){2-4}
    \cmidrule(lr){5-7}
    \cmidrule(lr){8-10}
    
    Method
    & 100k & 200k & 300k
    & 0 & 150k & 300k
    & 100k & 200k & 300k \\
    \midrule
    
    BC Policy
    & 0.69 & 0.69  & 0.69
    & \textbf{0.9} & \textbf{0.9} & \textbf{}{0.9}
    & 0.5 &  0.5 &  0.5 \\
    
    WSRL 
    & $0.31\pm0.09$ & $0.65\pm0.09$ & $0.64\pm0.08$
    & $0.78\pm0.06$ & $\mathbf{0.92\pm0.05}$ & $\mathbf{0.98\pm0.01}$
    & $\mathbf{0.91\pm0.11}$ & $\mathbf{0.98\pm0.01}$ & $\mathbf{1.0}$ \\
    
    CQL 
    & 0.0 & 0.0 & $.06 \pm 0.08$
    & $0.7\pm0.08$ & $0.12\pm0.03$ & $0.05\pm0.03$
    & 0.0 & $0.02\pm0.03$ & $0.31\pm0.33$ \\
    
    CalQL
    & $0.07\pm0.07$ & $0.19\pm0.07$ & $0.23\pm0.02$
    & $0.71\pm0.11$ & $\mathbf{0.87\pm0.02}$ & $\mathbf{0.93\pm0.03}$
    & $0.0$ & $0.0$ & $0.15\pm0.15$ \\
    
    IBRL
    & $0.35\pm0.14$ & $0.43\pm0.18$ & $0.50\pm0.15$
    & $0.1\pm0.15$ & $0.82\pm0.02$ & $\mathbf{0.95\pm0.05}$
    & $0.0$ & $0.0$ & $0.0$ \\
    
    \MethodName{} (Ours)
    & $\mathbf{0.85\pm0.06}$ & $\mathbf{0.87\pm0.03}$ & $\mathbf{0.91\pm0.01}$
    & $\mathbf{0.88\pm0.04}$ & $\mathbf{0.91\pm0.05}$ & $\mathbf{0.93\pm0.03}$
    & $0.55\pm0.15$ & $0.73\pm0.12$ & $0.87\pm0.08$ \\
    
    \midrule
    
    With Data & \multicolumn{3}{c}{Lift-State}
    & \multicolumn{3}{c}{Can-State}
    & \multicolumn{3}{c}{Square-State} \\
    \cmidrule(lr){2-4}
    \cmidrule(lr){5-7}
    \cmidrule(lr){8-10}
    
    Method
    & 100k & 200k & 300k
    & 100k & 300k & 500k
    & 100k & 300k & 500k \\
    \midrule
    
    BC Policy
    & $0.58$ & $0.58$ & $0.58$ 
    & $0.6$ &  $0.6$& $0.6$
    & $\mathbf{0.58}$ & $0.58$ & $0.58$  \\
    
    RLPD 
    &$0.87\pm0.07$ & $\mathbf{0.98\pm0.02}$ & $\mathbf{0.99\pm0.01}$
    & $0.0$ & $0.0$ & $0.0$
    & $0.0$ & $0.0$ &  $0.0$\\
    
    IBRL
    & $\mathbf{0.97\pm0.03}$ & $\mathbf{1.0}$  & $\mathbf{0.98\pm0.02}$
    & $0.02\pm0.02$ & $0.32\pm0.16$ & $0.54\pm0.18$ 
    & $0.0$ & $\mathbf{0.72\pm0.11}$ & $\mathbf{0.94\pm0.01}$ \\
    
    \MethodName{} (Ours)
    & $0.88\pm0.06$ & $\mathbf{1.0}$ & $\mathbf{1.0}$
    & $\mathbf{0.76\pm0.05}$ & $\mathbf{0.86\pm0.03}$ & $\mathbf{0.85\pm0.03}$ 
    & $\mathbf{0.60\pm0.07}$ & $\mathbf{0.78\pm0.06}$ & $0.81\pm0.08$\\
    
    \midrule
    
    Without Data & \multicolumn{3}{c}{Lift-State}
    & \multicolumn{3}{c}{Can-State}
    & \multicolumn{3}{c}{Square-State} \\
    \cmidrule(lr){2-4}
    \cmidrule(lr){5-7}
    \cmidrule(lr){8-10}
    
    Method
    & 100k & 200k & 300k
    & 100k & 300k & 550k
    & 100k & 300k & 550k \\
    \midrule
    
    BC Policy
    & $0.58$ & $0.58$ & $0.58$ 
    & $\mathbf{0.6}$ &  $0.6$& $0.6$
    & $\mathbf{0.58}$ & $0.58$ & $0.58$  \\
    
    WSRL 
    & $0.01\pm0.01$ & $0.0$ & $0.0$
    & $0.0$ & $0.0$ & $0.0$
    & $0.0$ & $0.0$ & $0.0$ \\
    
    CQL 
    & $0.0$ & $0.0$ & $0.0$
    & $0.0$ & $0.0$ & $0.0$
    & $0.0$ & $0.0$ & $0.0$ \\
    
    CalQL
    & $0.0$ & $0.0$ & $0.0$
    & $0.0$ & $0.0$ & $0.0$
    & $0.0$ & $0.0$ & $0.0$ \\
    
    IBRL
    & $0.01\pm0.01$ & $0.0$  & $0.0$
    & $0.0$ & $0.0$ & $0.0$
    & $0.0$ & $0.0$ & $0.0$ \\
    
    \MethodName{} (Ours)
    & $\mathbf{0.86\pm0.08}$ & $\mathbf{0.96\pm0.05}$  & $\mathbf{1.0}$ 
    & $\mathbf{0.57\pm0.12}$ & $\mathbf{0.75\pm0.13}$ & $\mathbf{0.82\pm0.10}$
    & $\mathbf{0.51\pm0.15}$ & $\mathbf{0.68\pm0.06}$ & $\mathbf{0.76\pm0.06}$ \\
    
    \bottomrule
    \end{tabular}
\end{table*}

\begin{table*}[htb]
    \centering
    \footnotesize
    \caption{Simulation Results - Image}
    \label{tab:sim_results_image}
    \setlength{\tabcolsep}{2.5pt} 

    \begin{tabular}{l cc  ccc}
    \toprule
    
    With Data & \multicolumn{2}{c}{Lift-Image}
    & \multicolumn{3}{c}{Can-Image} \\
    \cmidrule(lr){2-3}
    \cmidrule(lr){4-6}
    
    Method
    & 100k & 200k
    & 100k & 300k & 500k \\
    \midrule
    
    BC Policy
    & $0.6$ &  $0.6$
    & $0.45$ &  $0.45$& $0.45$  \\
    
    RLPD 
    & $0.0$ & $0.0$
    & $0.0$ & $0.0$ & $0.0$\\
    
    IBRL
    & $\mathbf{1.0}$ & $\mathbf{1.0}$
    & $0.12\pm0.13$ & $0.01\pm0.01$ & $0.03\pm0.05$  \\
    
    \MethodName{} (Ours)
    & $0.93\pm0.05$ & $\mathbf{1.0}$
    & $\mathbf{0.53\pm0.07}$ & $\mathbf{0.63\pm0.02}$ & $\mathbf{0.73\pm0.06}$ \\
    \midrule
    
    Without Data & \multicolumn{2}{c}{Lift-Image}
    & \multicolumn{3}{c}{Can-Image} \\
    \cmidrule(lr){2-3}
    \cmidrule(lr){4-6}
    
    Method
    & 100k & 200k
    & 100k & 300k & 550k \\
    \midrule
    
    BC Policy
    & $0.6$ &  $0.6$
    & $\mathbf{0.45}$ &  $0.45$& $0.45$  \\
    
    IBRL
    & $0$ & $0.01\pm0.01$
    & $0.0$ & $0.0$ & $0.0$ \\
    
    \MethodName{} (Ours)
    & $\mathbf{0.87\pm0.05}$ & $\mathbf{0.98\pm0.01}$ 
    & $\mathbf{0.45\pm0.07}$ & $\mathbf{0.57\pm0.10}$ & $\mathbf{0.63\pm0.05}$ \\
    
    \bottomrule
    \end{tabular}
    \vspace{-1em}
\end{table*}

\newpage

\subsection{Additional Details}
We provide some additional details on our simulation experiments: 
 
\subsubsection{IBRL implementation}
IBRL uses observation history ($t-2, t-1, t$) for their RL policy.
To keep comparisons fair, we implement our RL policy conditioned only on the current observation.
Our results demonstrate that \MethodName{} performs well using a simple RL policy architecture that takes a single observation as input and outputs a single action.
\subsubsection{Sampling Actions}
The Q-gating step in \MethodName{} selects the action with the highest Q-value.
Accordingly, in simulation experiments we evaluate RL actions using the mode of the policy distribution.
In real-world settings, which are inherently more stochastic, we instead sample actions using the policy’s distributional standard deviation.
\MethodName{} demonstrates good performance under both action sampling strategies.

\subsection{Details on Simulation Results}
To support the plots in the main paper, we provide success rate numbers and confidence intervals for state-based simulation tasks in Table~\ref{tab:sim_results_state} and image-based tasks in Table~\ref{tab:sim_results_image}.
All the evaluations are done for 20 rollouts and 5 seeds.
We do not evaluate offline RL-based methods in image-based environments, as they failed to successfully pretrain and already exhibited poor performance on the more challenging state-based tasks.

\section{Additional Experiments and Ablations}
\label{sec:app_ablations}

\subsection{Testing the Soft-Optimality Assumption}
To test \MethodName{}'s sensitivity to the soft-optimality assumption, we provide experiments simulating other forms of sub-optimality.
We use a \emph{deterministic} $\pi_{BC}$ plus $\epsilon$ noise sampled from Gaussian and uniform distributions.
This non soft-optimal $\pi_{BC}$ is used for both Q-Estimation and online RL.
Fig.~\ref{fig:nonsoft_BC} shows that such policies suffer an initial drop in success (performance at dotted line), but they recover during online RL, demonstrating the robustness of \MethodName{}.
\begin{figure}[htb]
  \centering
  \includegraphics[width=0.95\linewidth]{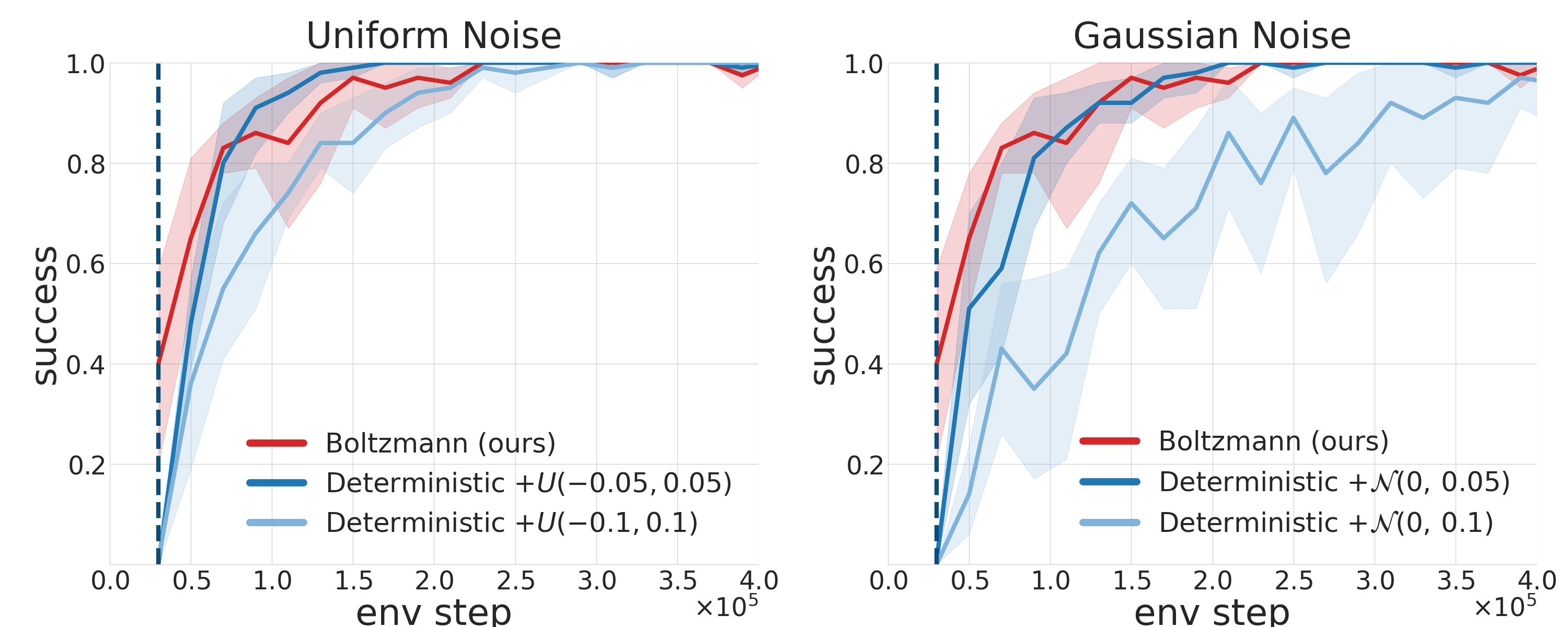}
  \caption{Results with non soft-optimal BC policies for Lift-State on 20 trials. \label{fig:nonsoft_BC}}
\end{figure}

\subsection{Baseline: Residual Reinforcement Learning}

We compare \MethodName{} with Policy Decorator~\cite{yuan2024policy}, a state-of-the-art residual learning method.
Residual policy learning is another approach to improving a pre-trained BC policy by providing corrective adjustments to the BC actions. 
Policy Decorator introduces two key enhancements to the Residual RL framework: (1) scaling the residual action added to the base policy, and (2) employing a uniform progressive exploration schedule for the residual policy.
The results are shown in Fig.~\ref{fig:residual_rl}.
We find that while Policy Decorator is able to recover and modestly improve upon the BC performance for the Can task, it requires significantly more interaction to recover the BC performance on more challenging Square task.

\begin{figure}[htb]
    \centering
    \includegraphics[width=\linewidth]{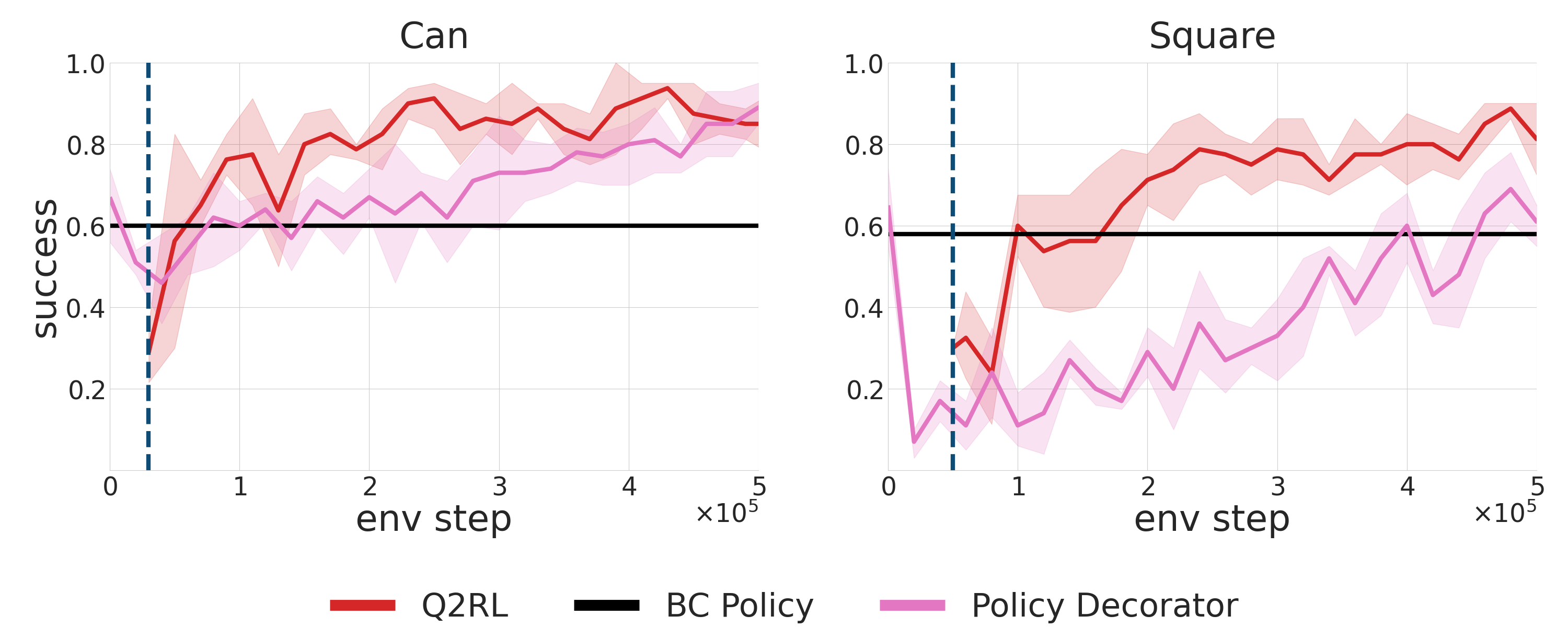}
    \caption{Policy Decorator takes more online interaction to recover BC performance for challenging tasks like Square.}
    \label{fig:residual_rl}
\end{figure}

\subsection{Ablation: Seeding the Online Replay Buffer}
\label{sec:ablation_seeding}

We analyze the effect of seeding the online replay buffer with different fractions of data for both \MethodName{} and IBRL (Fig.~\ref{fig:ablation_dataset}). 
\MethodName{} maintains strong performance regardless of whether offline training data is available in the replay buffer, corroborating our real-world experiments. 
In contrast, IBRL shows a strong dependence on replay buffer seeding.

\begin{figure}[htb]
    \centering
    \includegraphics[width=\linewidth]{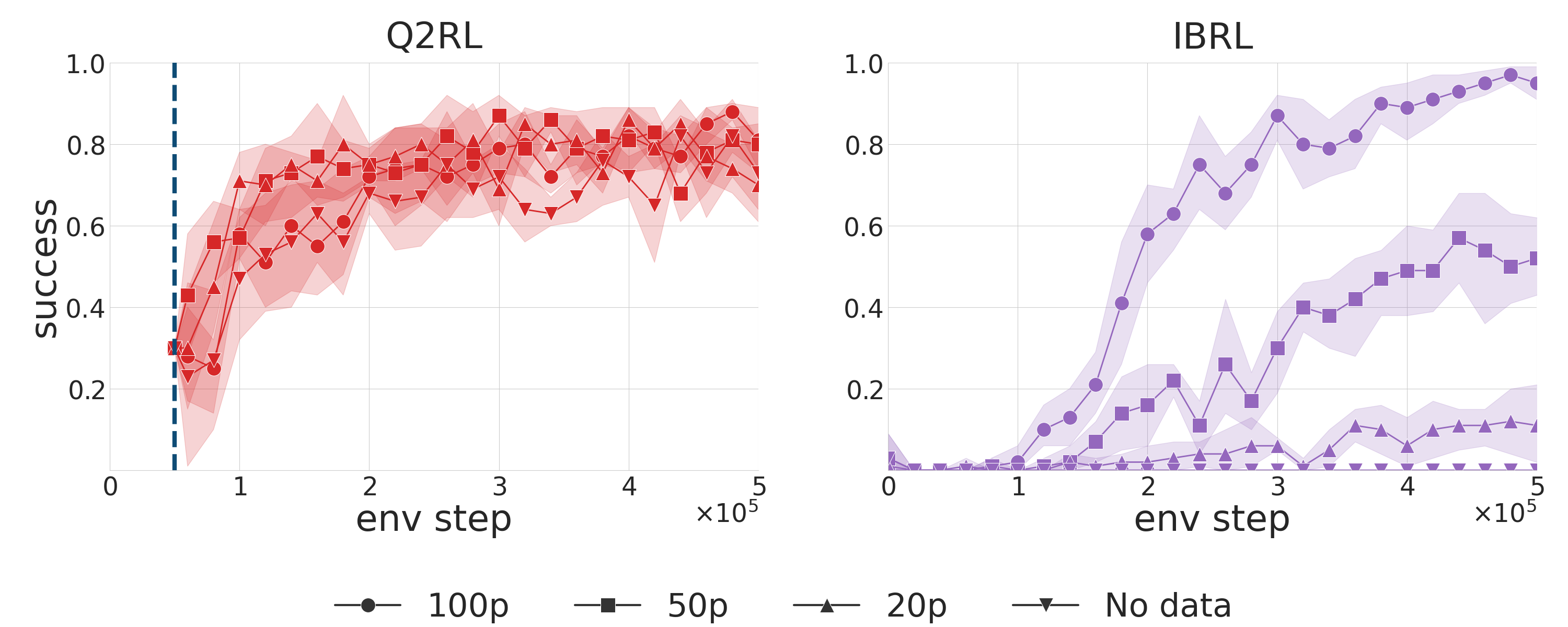}
    \caption{Seeded Online Replay Buffer Ablation. \MethodName{} does not require seeding the online replay buffer with data.}
    \label{fig:ablation_dataset}
\end{figure}

\subsection{Ablation: BC Auxiliary Loss}

\MethodName{} incorporates an auxiliary behavior cloning (BC) loss to stabilize training, whereas IBRL relies on access to demonstration data. To ensure a fair comparison in settings without data access, we also evaluate IBRL with the same BC loss in Fig.~\ref{fig:ablation_ibrl_bc_loss}. While the BC loss improves IBRL relative to the no-BC loss variant, it does not lead to meaningful performance gains compared to \MethodName{}.

\begin{figure}[htb]
    \centering
    \includegraphics[width=0.6\linewidth]{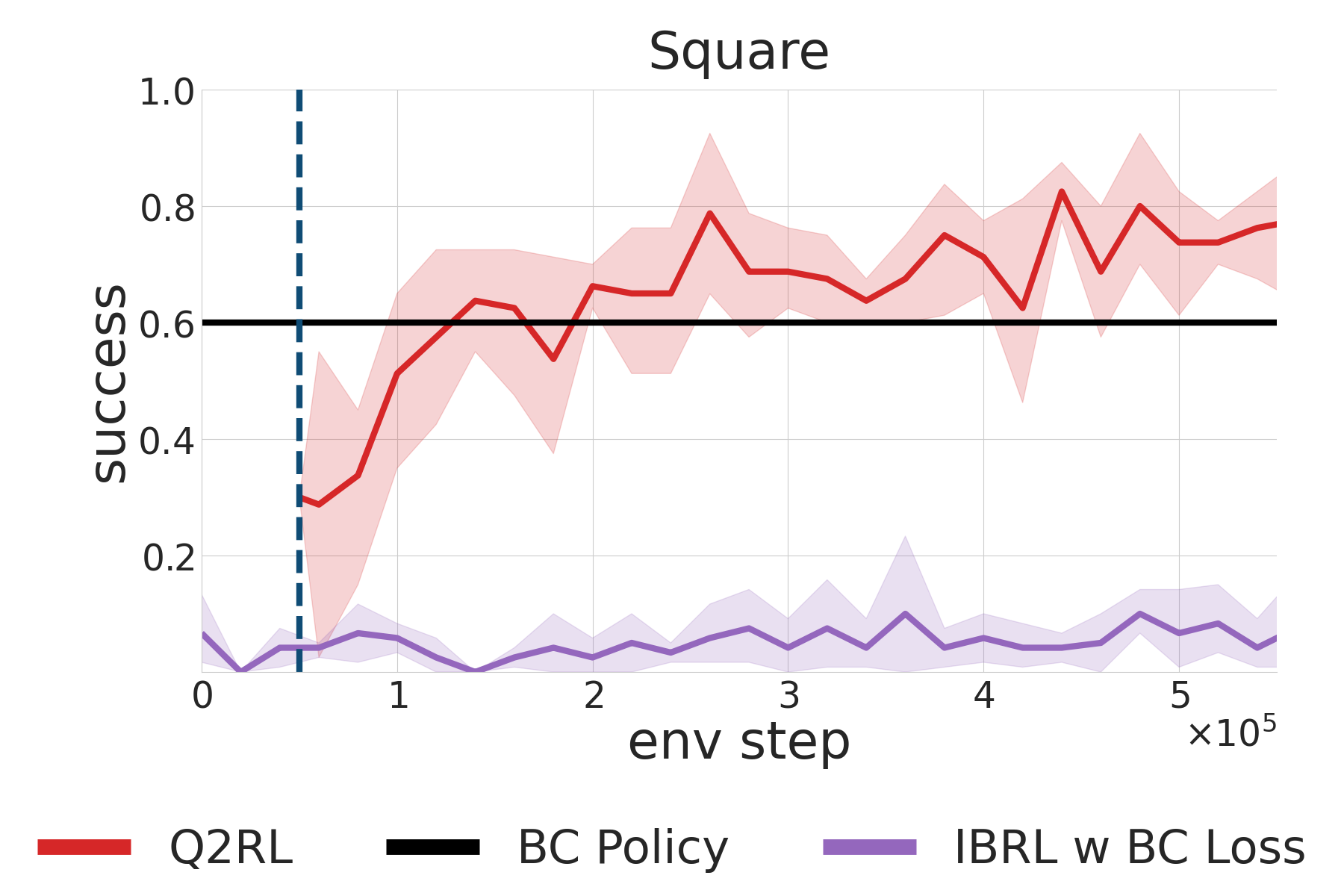}
    \caption{IBRL is not able to match the performance of \MethodName{} even with the BC auxiliary loss in the absence of training data.}
    \label{fig:ablation_ibrl_bc_loss}
\end{figure}

\subsection{Ablation: Number of Rollouts used for Q-Estimation} 
We ablate the number of rollouts used during the Q-estimation phase in Fig.~\ref{fig:ablation_rollout}, using the Can-State task in the with data setting.
We observe that \MethodName{} achieves competitive performance with as few as 25 rollouts.
In our main simulation experiments, we use 100 rollouts to obtain more reliable Q-estimates and ensure stable performance.

\label{sec:ablation_rollout}
\begin{figure}[htb]
    \centering
    \includegraphics[width=0.7\linewidth]{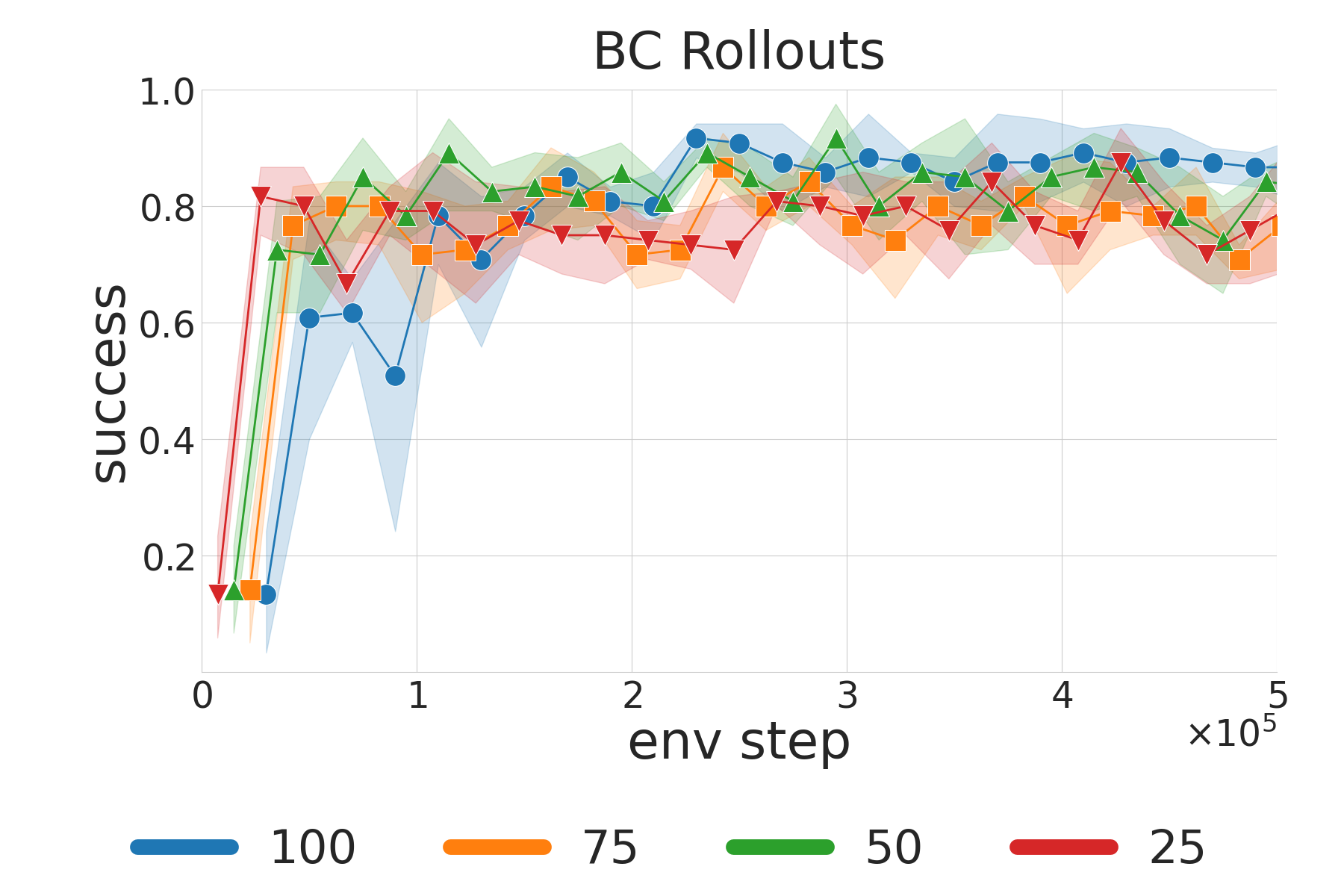}
    \caption{\MethodName{} achieves competitive performance with as few as 25 rollouts.}
    \label{fig:ablation_rollout}
\end{figure}

\subsection{Ablation: Initial BC Performance}

Next, we evaluate whether \MethodName{} can improve BC policies with varying levels of initial performance.
We use the same experimental setup as in Sec.~\ref{sec:ablation_rollout} and test \MethodName{} with different BC checkpoints, with initial success rates ranging from 10\% to 75\%, as shown in Fig.~\ref{fig:ablation_checkpoint}.
We observe that \MethodName{} consistently improves performance across all initializations.

\begin{figure}[htb]
    \centering
    \includegraphics[width=0.7\linewidth]{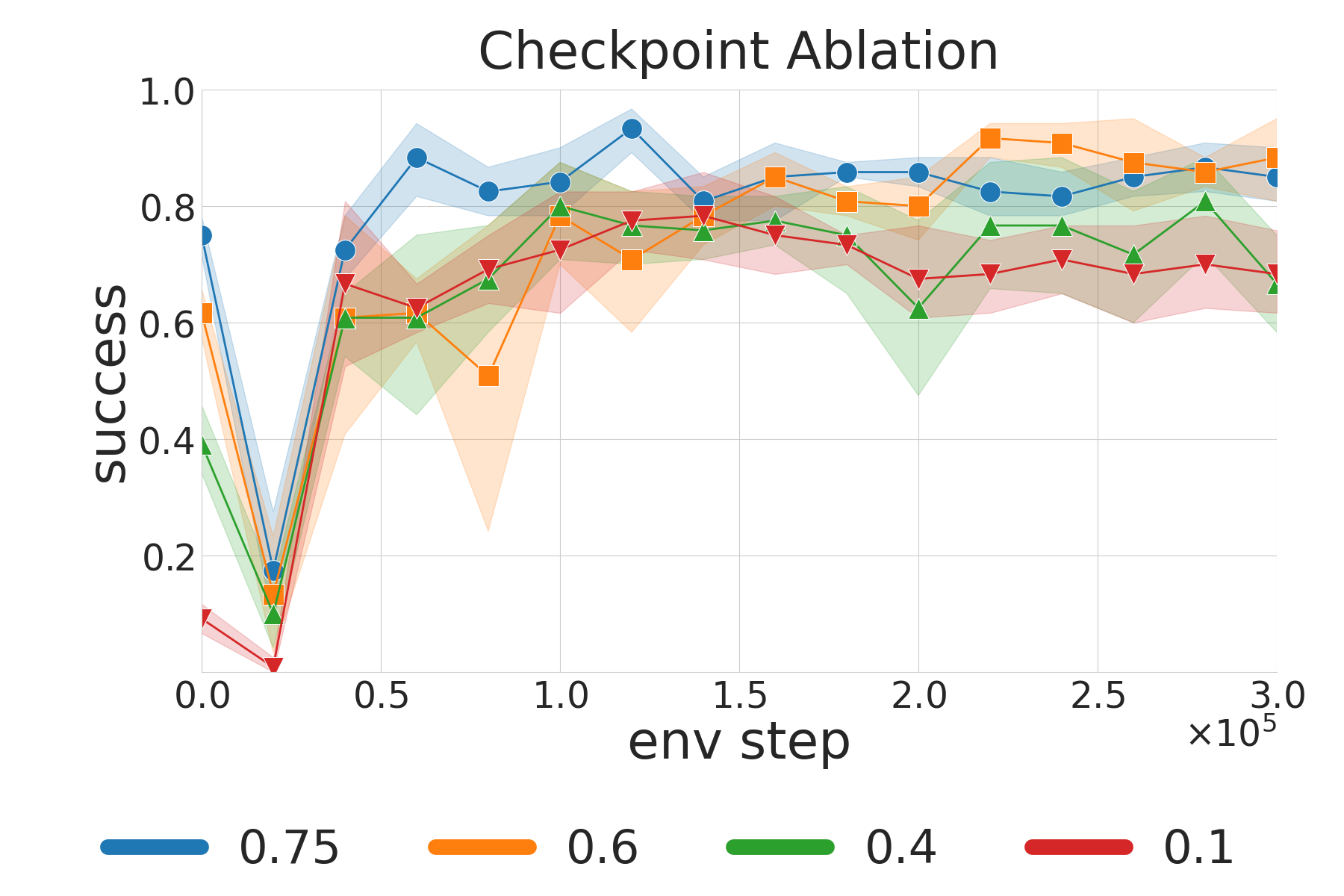}
    \caption{\MethodName{} is able to improve a wide range of initial BC performance.}
    \label{fig:ablation_checkpoint}
\end{figure}

\subsection{Qualitative Result: Door}
\label{app:qualitative_door}

\begin{figure*}[t]
\centering

\begin{subfigure}{0.24\textwidth}
    \centering
    \includegraphics[width=\linewidth]{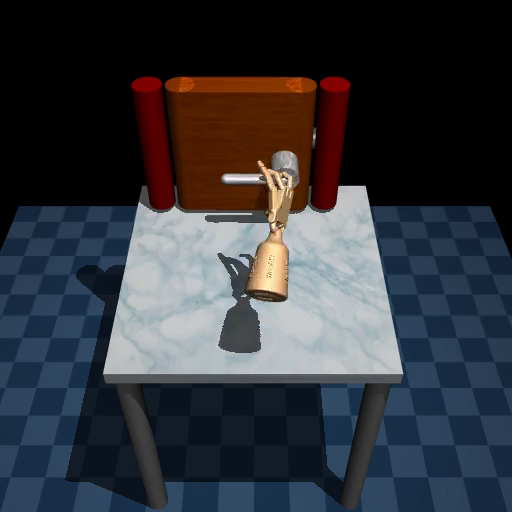}
\end{subfigure}
\hfill
\begin{subfigure}{0.24\textwidth}
    \centering
    \includegraphics[width=\linewidth]{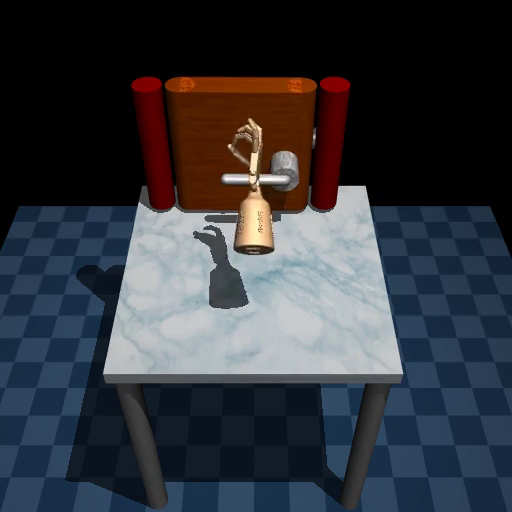}
\end{subfigure}
\hfill
\begin{subfigure}{0.24\textwidth}
    \centering
    \includegraphics[width=\linewidth]{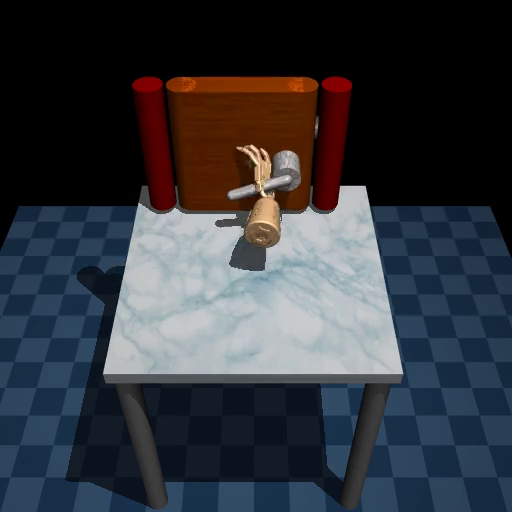}
\end{subfigure}
\hfill
\begin{subfigure}{0.24\textwidth}
    \centering
    \includegraphics[width=\linewidth]{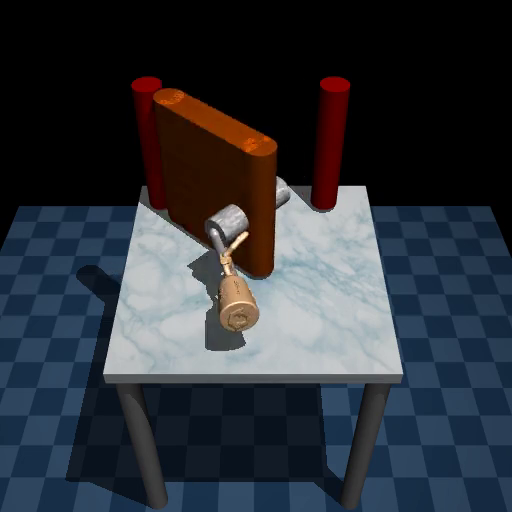}
\end{subfigure}
\vspace{0.6em}
\begin{subfigure}{\textwidth}
\centering
\caption*{WSRL rollout on Adroit-Door.}
\end{subfigure}
\vspace{0.6em}
\begin{subfigure}{0.24\textwidth}
    \centering
    \includegraphics[width=\linewidth]{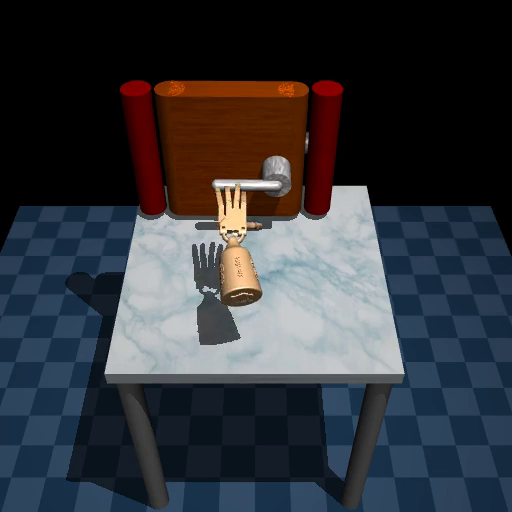}
\end{subfigure}
\hfill
\begin{subfigure}{0.24\textwidth}
    \centering
    \includegraphics[width=\linewidth]{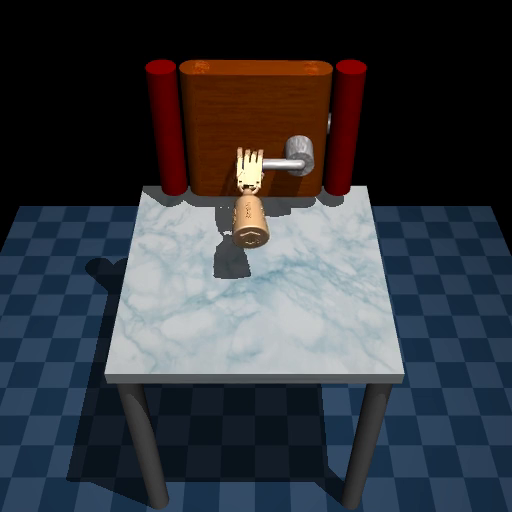}
\end{subfigure}
\hfill
\begin{subfigure}{0.24\textwidth}
    \centering
    \includegraphics[width=\linewidth]{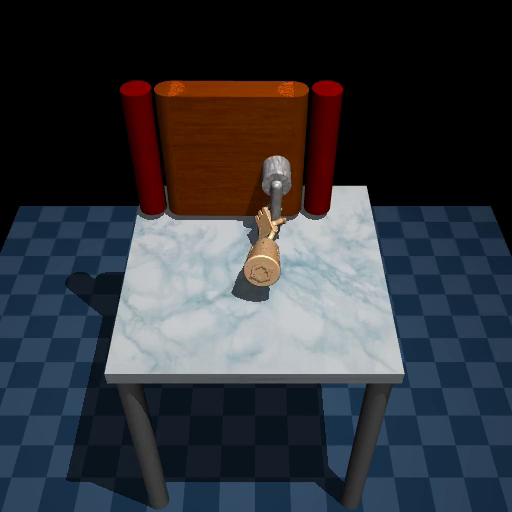}
\end{subfigure}
\hfill
\begin{subfigure}{0.24\textwidth}
    \centering
    \includegraphics[width=\linewidth]{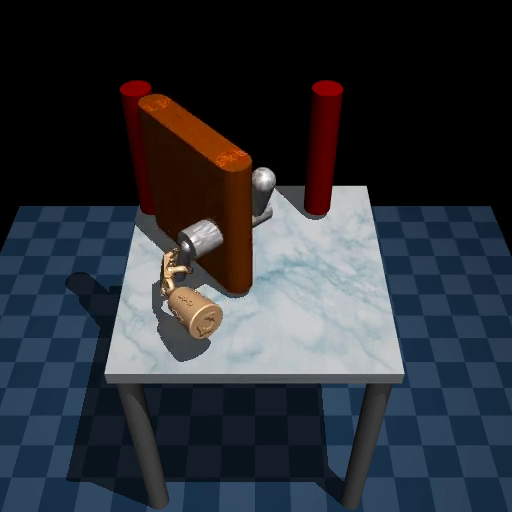}
\end{subfigure}

\begin{subfigure}{\textwidth}
\centering
\caption*{Q2RL rollout on Adroit-Door.}
\end{subfigure}

\caption{
\MethodName{} learns more realistic policies than WSRL, even though WSRL achieves higher success rate for Adroit-Door.
}
\label{fig:qualitative_door}
\end{figure*}
While Fig.~\ref{fig:d4rl} shows that WSRL achieves higher success rates than \MethodName{} for the Adroit-Door task, qualitative rollouts (Fig.~\ref{fig:qualitative_door}) reveal that \MethodName{} learns more realistic policies where the hand first grabs the door handle and then opens the door.
On the other hand, the policies learned by WSRL are less feasible in the real world.

\section{Details on Real World Experiments}
\label{app:real_exps}

\begin{table*}[htb]
    \footnotesize
    \setlength{\tabcolsep}{2pt}
    \centering
    \caption{Real World Experiment Configurations}
    \begin{tabular}{cccc}
        \toprule
         & Peg Insertion & Pipe Assembly & Kitting (Modified) \\
        \midrule
        Initial Pose Distribution (\unit{\meter}, \unit{\radian}) & $\pm$0.045, $\pm$0.07875 & fixed & fixed \\
        Delta Action Range (\unit{\meter}, \unit{\radian}) & $\pm$0.015, $\pm$0.02625 & $\pm$0.015, $\pm$0.02625 & $\pm$0.015, $\pm$0.02625 \\
        Safety Limit XYZ (\unit{\meter}) 
            & [-0.02, 0.05], [-0.1, 0.02], [-0.13, 0.02] 
            & [-0.2, 0.1], [-0.2, 0.2], [-0.4, 0.05] 
            & [-0.2, 0.3], [-0.2, 0.2], [-0.18, 0.05] \\
        Safety Limit RPY (\unit{\radian}) 
            & [-5, 5], [-5, 5], [-20, 20] 
            & [-5, 5], [-5, 5], [-20, 20] 
            & [-5, 5], [-5, 5], [-20, 20] \\
        Num. Demos for BC Training & 50 & 100 & 50 \\
        Num. Rollouts for Q-Estimation Phase & 100 & 100 & 30 \\
        Max Episode Length & 200 & 200 & 500 \\
        \midrule
        Cartesian Stiffness (XYZRPY) & & [1024, 1024, 1024, 100, 100, 100]  &  \\
        Cartesian Damping (XYZRPY) &  &  [64, 64, 64, 10, 10, 10] &  \\
        Policy Hz & & 10 & \\
        \bottomrule
    \end{tabular}
    \label{tab:task_config}
\end{table*}

\subsection{Hardware} 
We conducted real world, on-robot experiments on a tabletop Franka FR3 arm.
The arms is equipped with Robotiq 2F-85 gripper with 3D-printed, compliant ``finray'' fingertips.  
Two Realsense D405s provide RGB input, one providing the workspace view, and the other the wrist view. 
See Fig.~\ref{fig:real_setup} for an image of the robot and workspace.
A force-torque sensor is mounted on the wrist but is not used.
A 3D Connexion Spacemouse is used to collect teleoperated human demonstrations and to assist with environment resets.

\subsection{Software} 
The inputs for all policies are the end effector pose in the robot base frame, and 84 x 84 RGB images (cropped and resized from 480 x 848) from each of the two cameras.
For the FMB~\cite{luo2025fmb} peg insertion task, the gripper is disabled.
The policy outputs delta end effector pose actions and gripper commands (if relevant).
The delta end effector actions are scaled to metric range, clipped to be within safe limits, converted to the robot base frame, and executed by a Cartesian impedance controller.
The stiffness and damping gains for the controller were set prior to starting experiments and the values were kept the same for all tasks.
The policy maintains a 10 Hz action frequency.

Our on-robot RL system consists of an actor process and a learner process.
The actor process collects transition samples by executing actions on the robot.
The learner process performs gradient descent based on collected data. 
The actor and learner run as independent processes and communicate asynchronously via ZeroMQ~\cite{agentlace2024}.
This asynchronous approach allows both processes to run simultaneously, speeding up time to convergence~\cite{luo2024serl, wu2022daydreamer}. 
For Q2RL, the actor process loads the frozen BC policy and maintains a copy of the RL policy's weights.
The learner process also maintains a copy of the RL weights for updating via learning. 
The actor asynchronously updates the learner's replay buffer every 30 actor samples, and the learner asynchronously updates the actor's network every 30 learner training steps. 
On average, our system executes around 13k RL actions and 44k RL learner steps per hour.
\subsection{Task Descriptions}

This section provides detailed descriptions for real world tasks. 
For all tasks, we reset the environment via motion planning to move to a fixed reset pose or a randomized reset pose. 
During reset, spacemouse teleoperation is sometimes used to move the arm into free-space before executing motion planning (e.g., when a pipe is partially inserted at the end of an episode). 
A human assists with resets and also provides the success signal based on the success conditions below.
The success signal terminates the episode; otherwise, the episode is truncated upon reaching the max episode length. 
See Fig.~\ref{fig:real_kitting} for key-frame images for each task and Tab.~\ref{tab:task_config} for task configuration parameters.

\begin{figure*}[ht]
    \centering
    \subfloat[Peg Insert (start)
    ]{\includegraphics[width=0.2\linewidth]{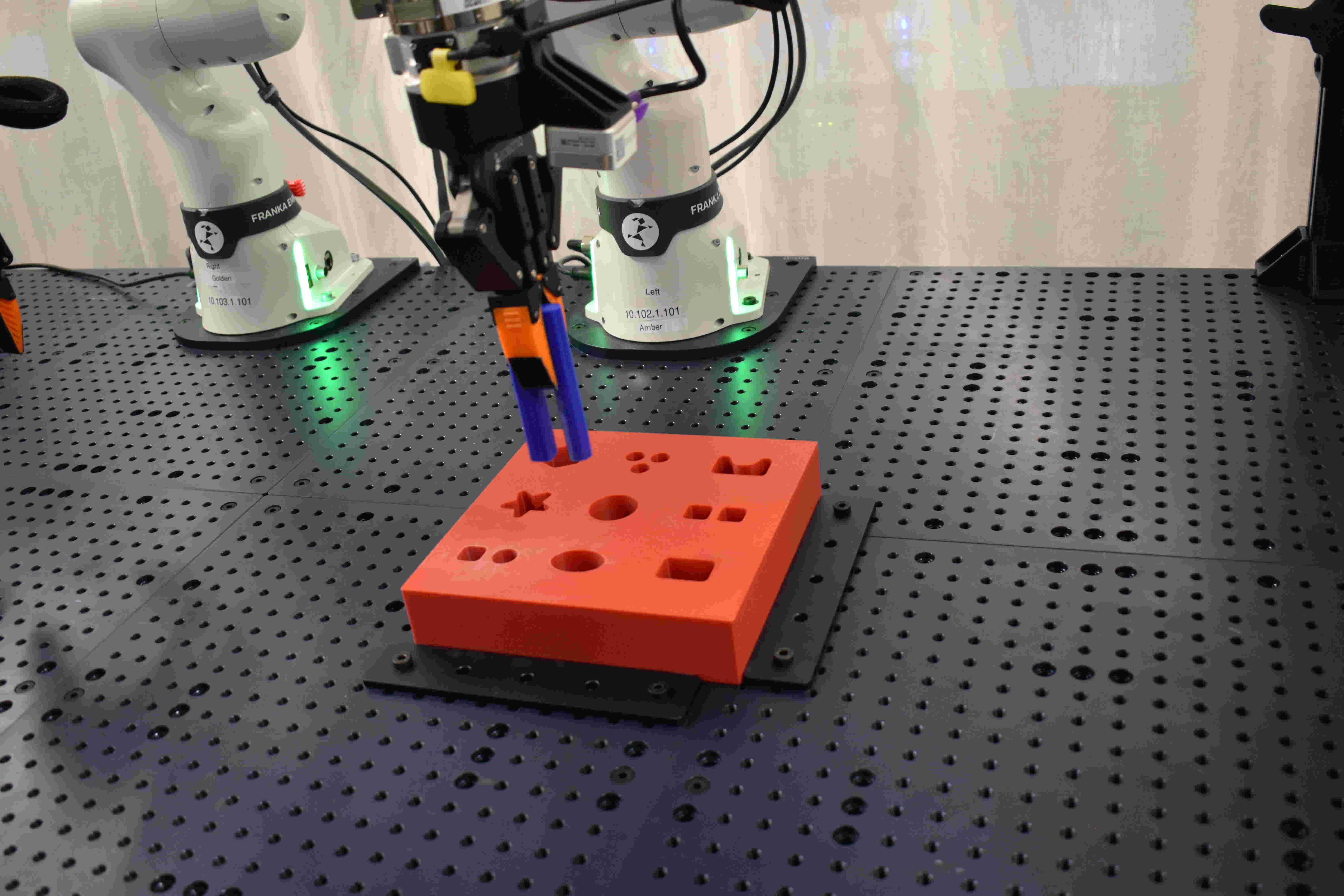}}
    \subfloat[Align
    ]{\includegraphics[width=0.2\linewidth]{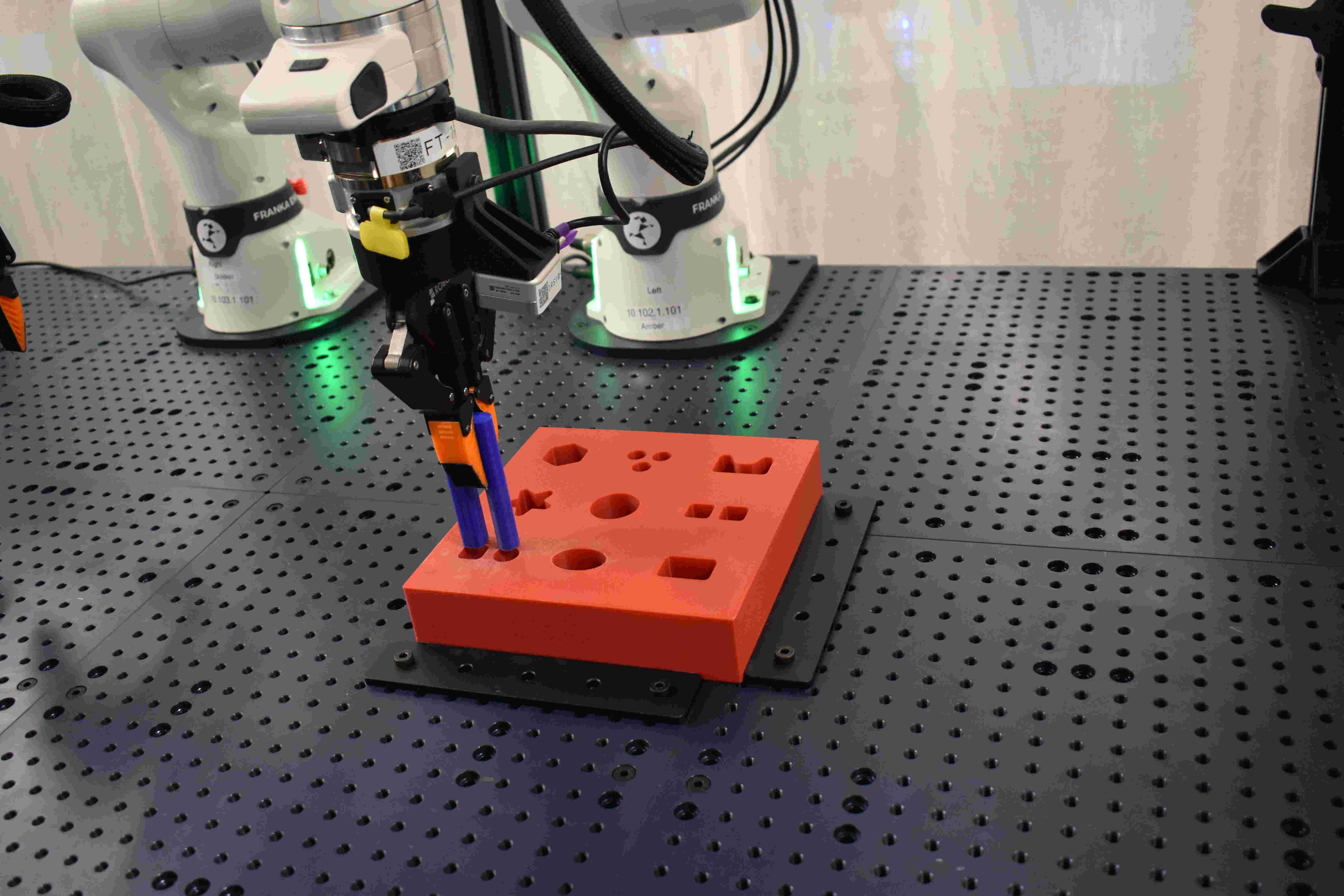}}
    \subfloat[Insert (done)
    ]{\includegraphics[width=0.2\linewidth]{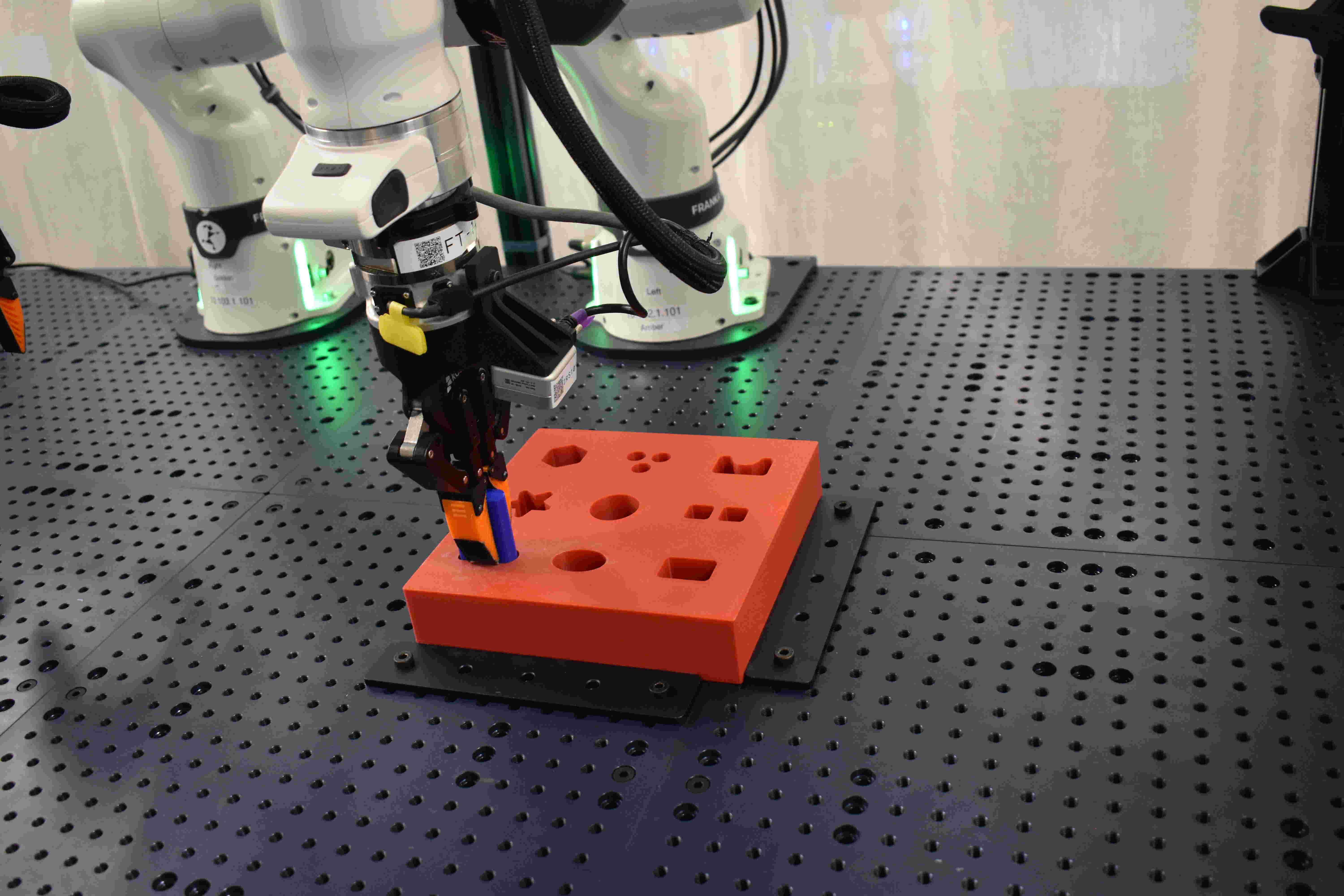}}
    \hfill
    \vspace{1em}
    \subfloat[Pipe Assembly (start)
    ]{\includegraphics[width=0.2\linewidth]{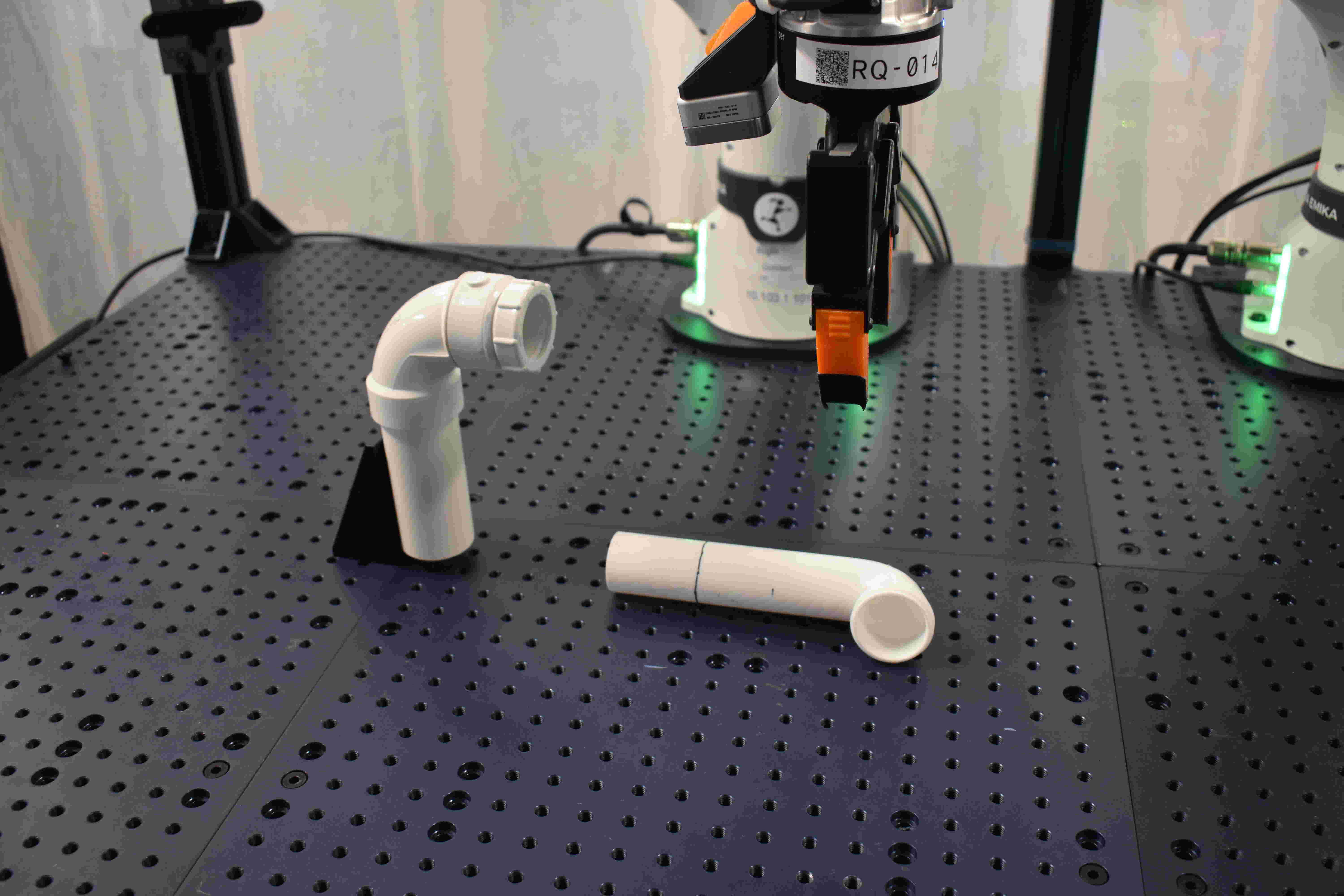}}
    \subfloat[Grasp
    ]{\includegraphics[width=0.2\linewidth]{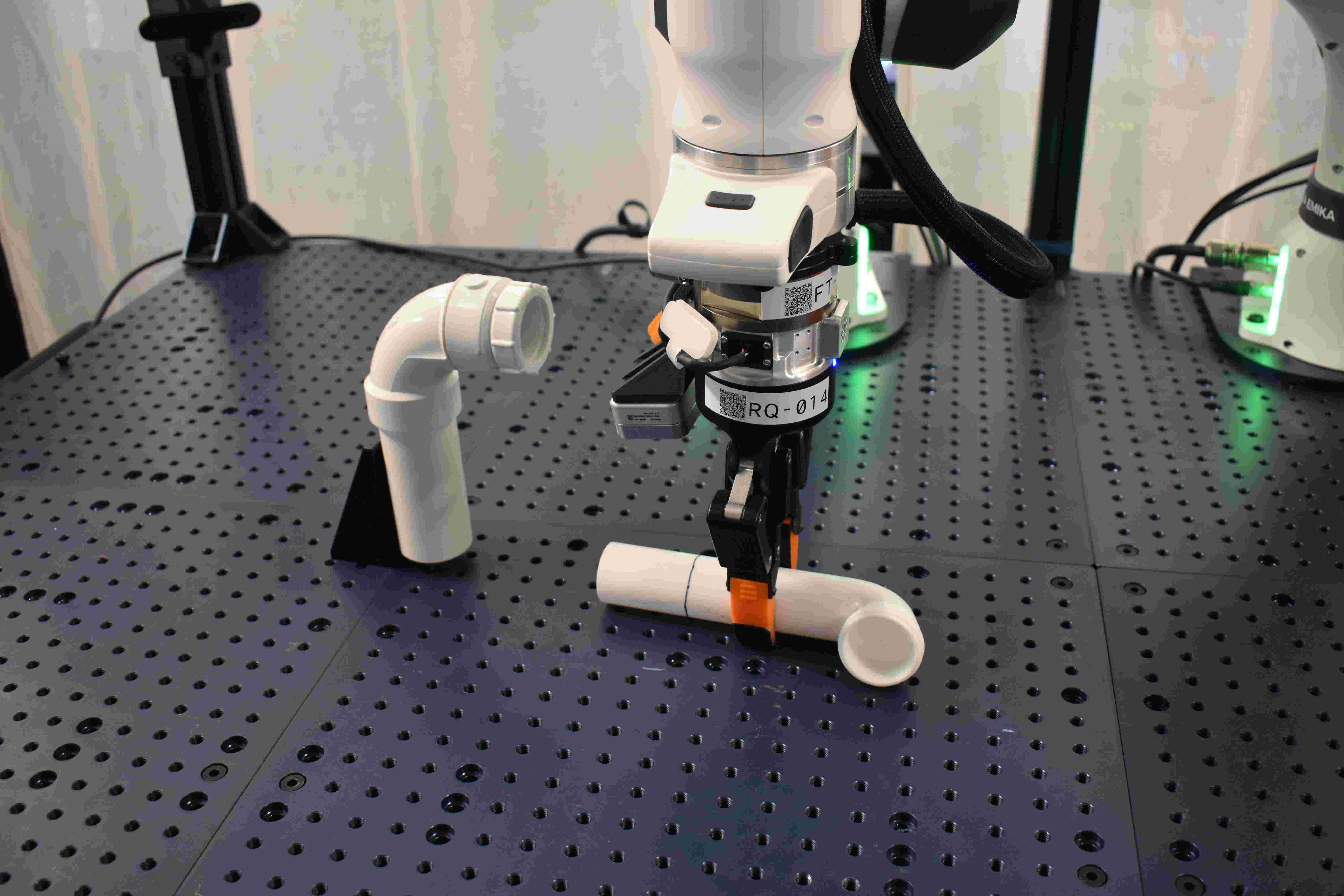}}
    \subfloat[Align
    ]{\includegraphics[width=0.2\linewidth]{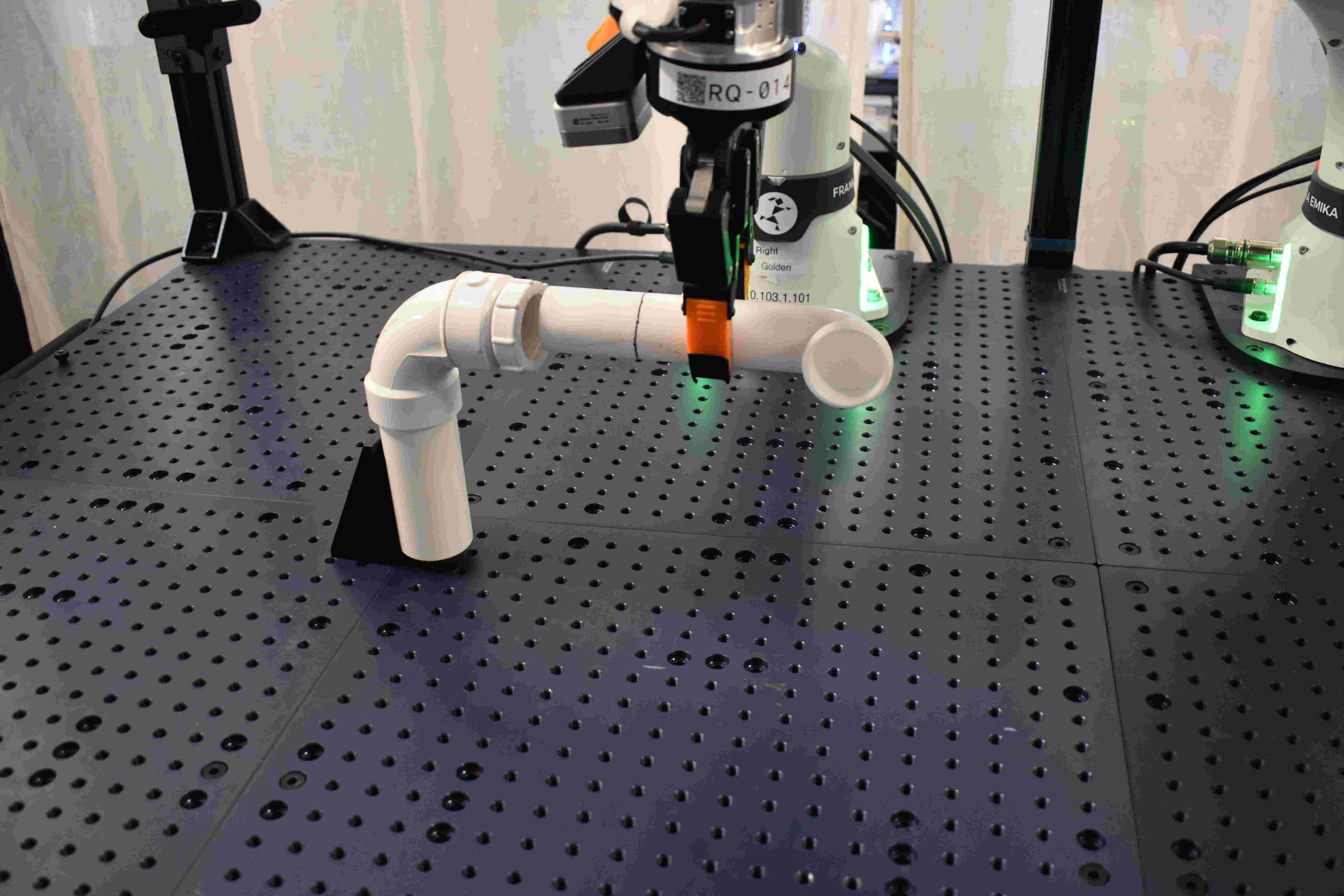}}
    \subfloat[Insert (done)
    ]{\includegraphics[width=0.2\linewidth]{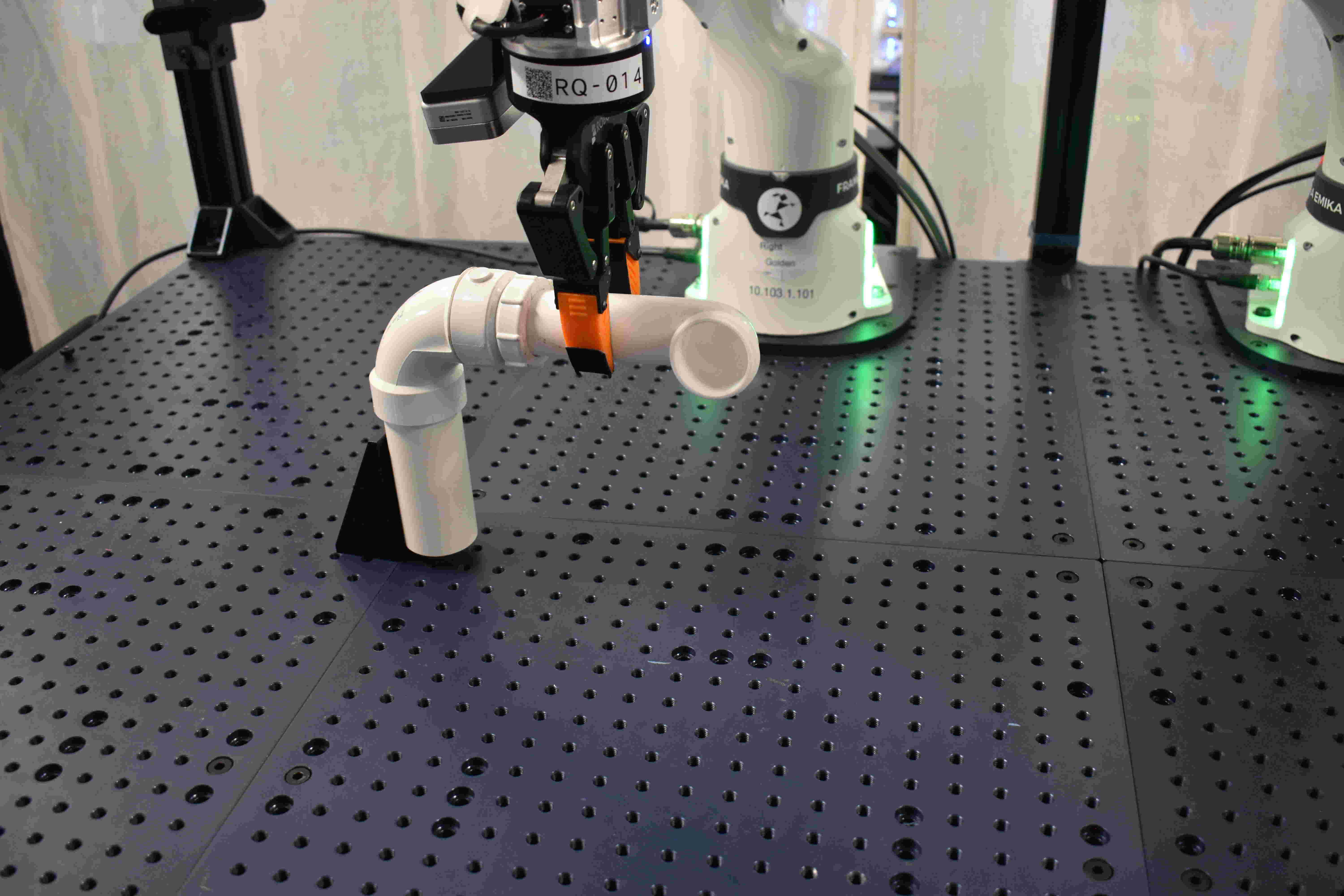}}
    \hfill
    \vspace{1em}
    \subfloat[Kitting-Original (start)
    ]{\includegraphics[width=0.2\linewidth]{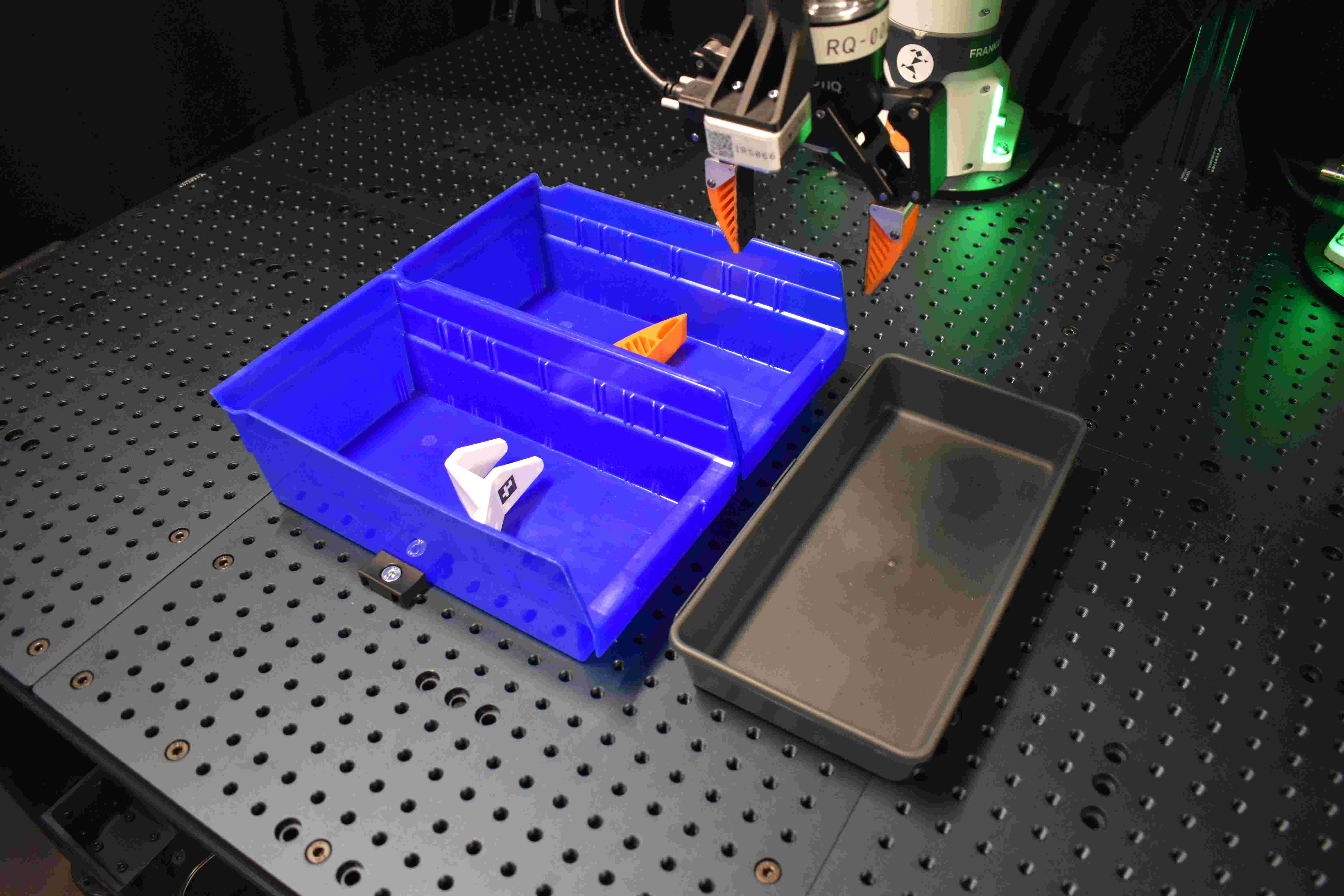}}
    \subfloat[Pick \#1
    ]{\includegraphics[width=0.2\linewidth]{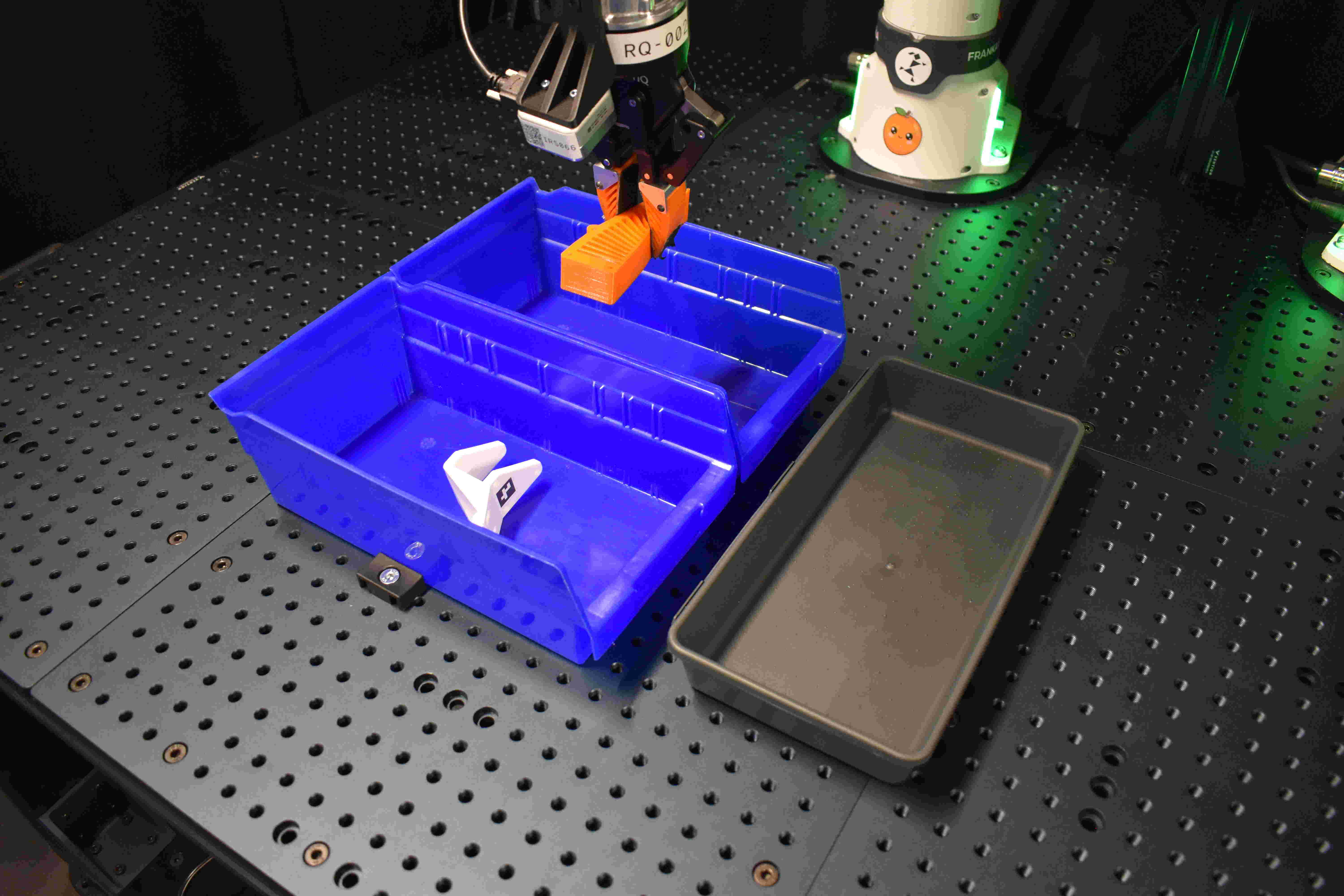}}
    \subfloat[Place \#1
    ]{\includegraphics[width=0.2\linewidth]{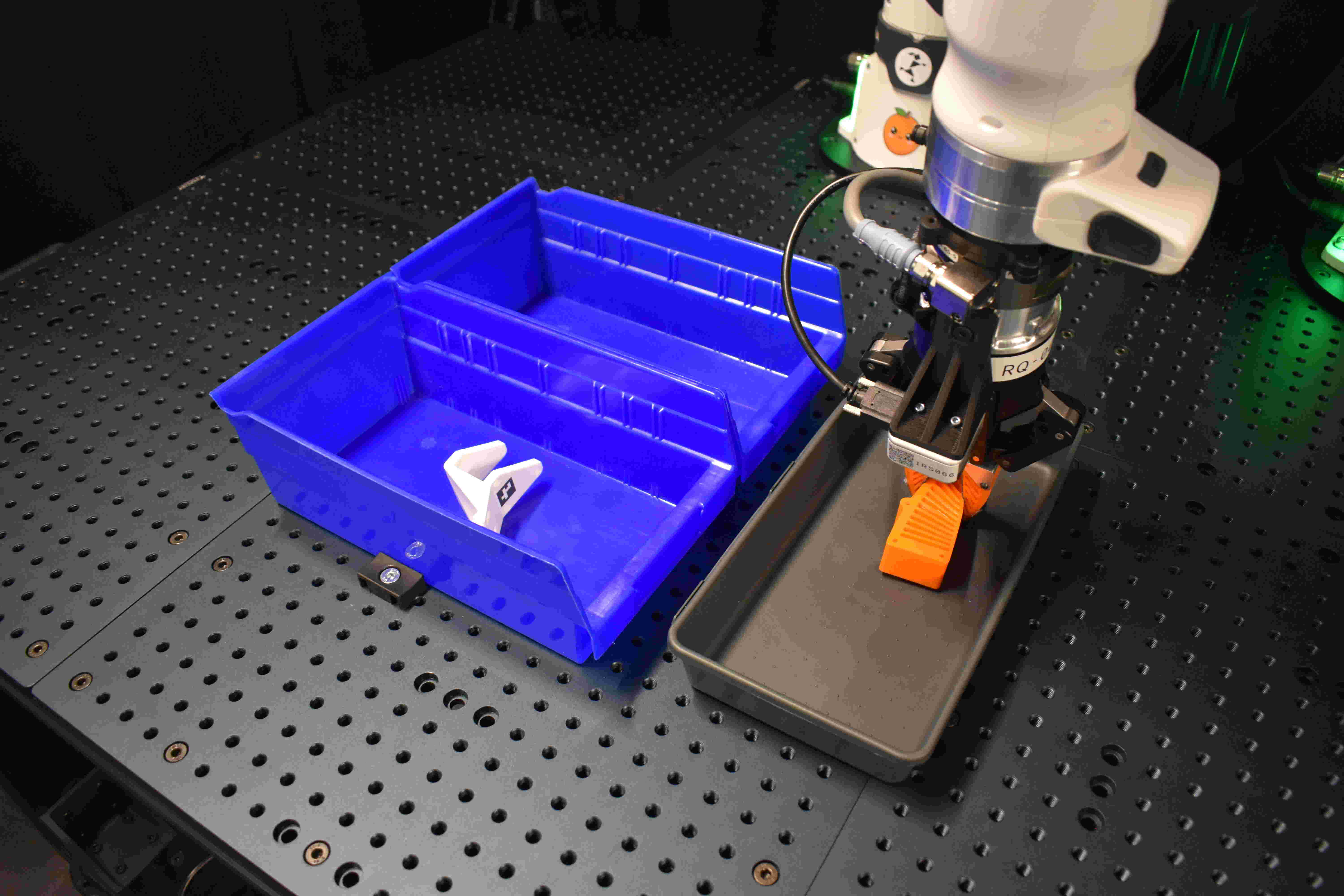}}
    \subfloat[Pick \#2
    ]{\includegraphics[width=0.2\linewidth]{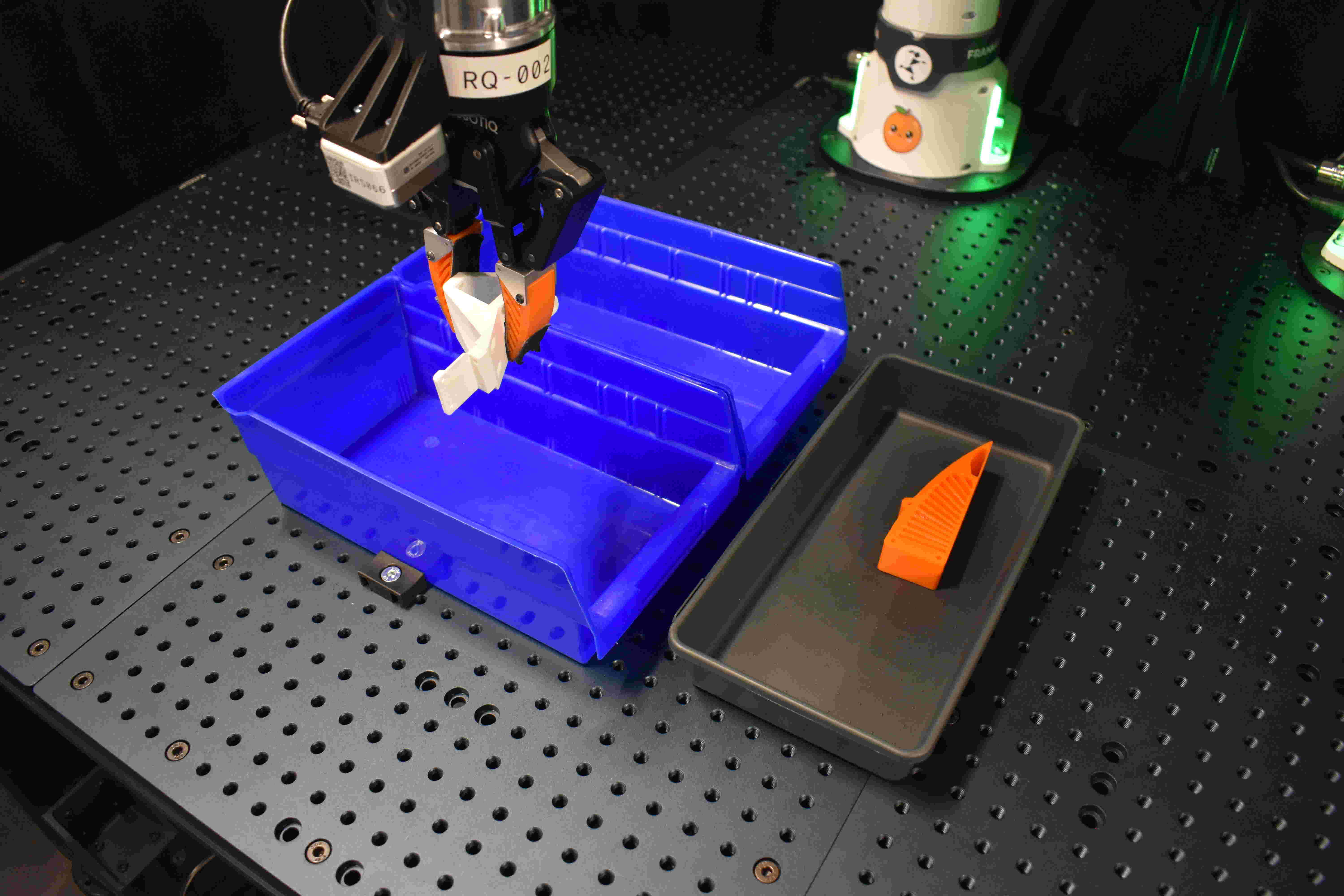}}
    \subfloat[Place \#2 (done)
    ]{\includegraphics[width=0.2\linewidth]{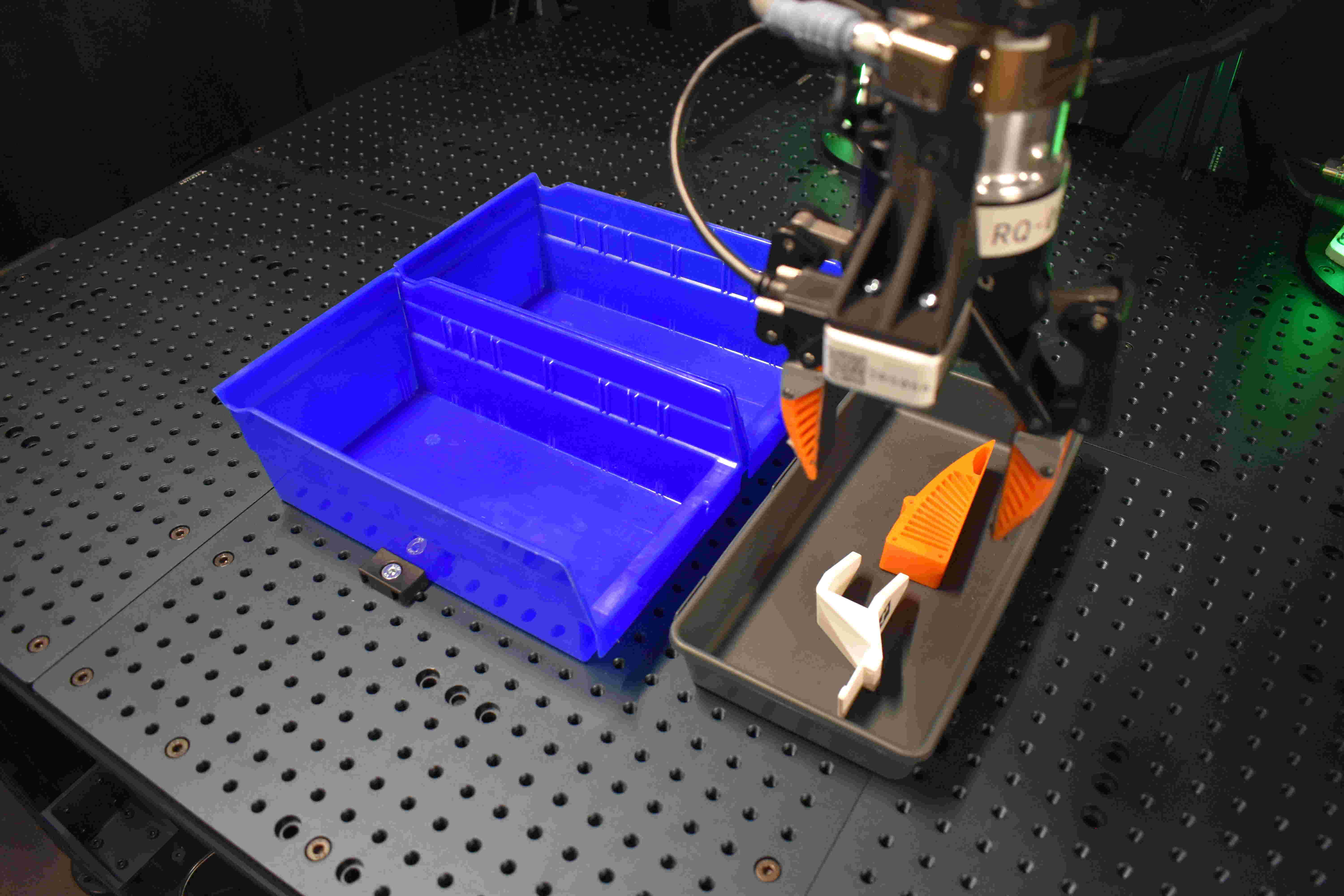}}
    \hfill
    \vspace{1em}
    \subfloat[Kitting-Modified (start)
    ]{\includegraphics[width=0.2\linewidth]{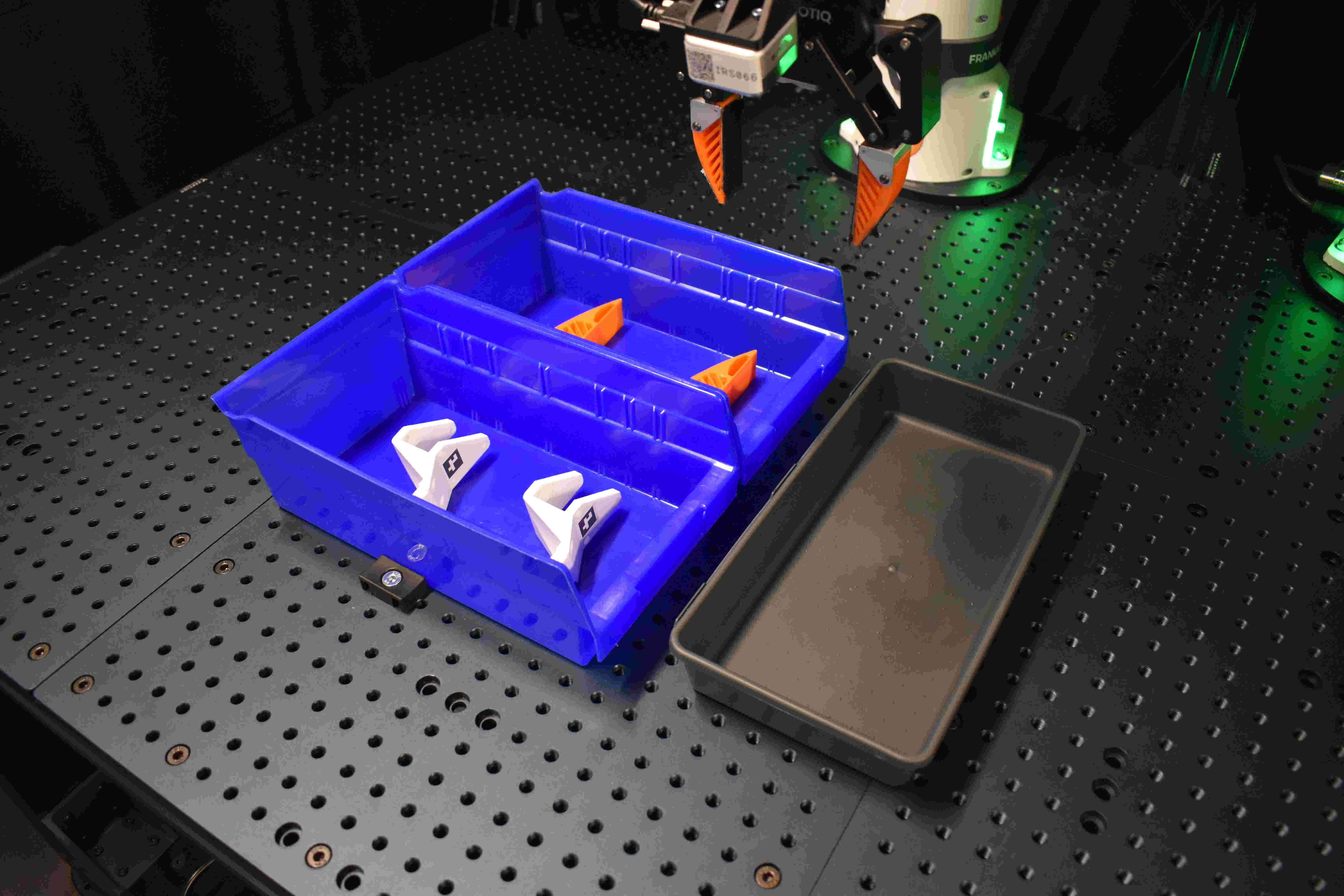}}
    \subfloat[Pick \#1
    ]{\includegraphics[width=0.2\linewidth]{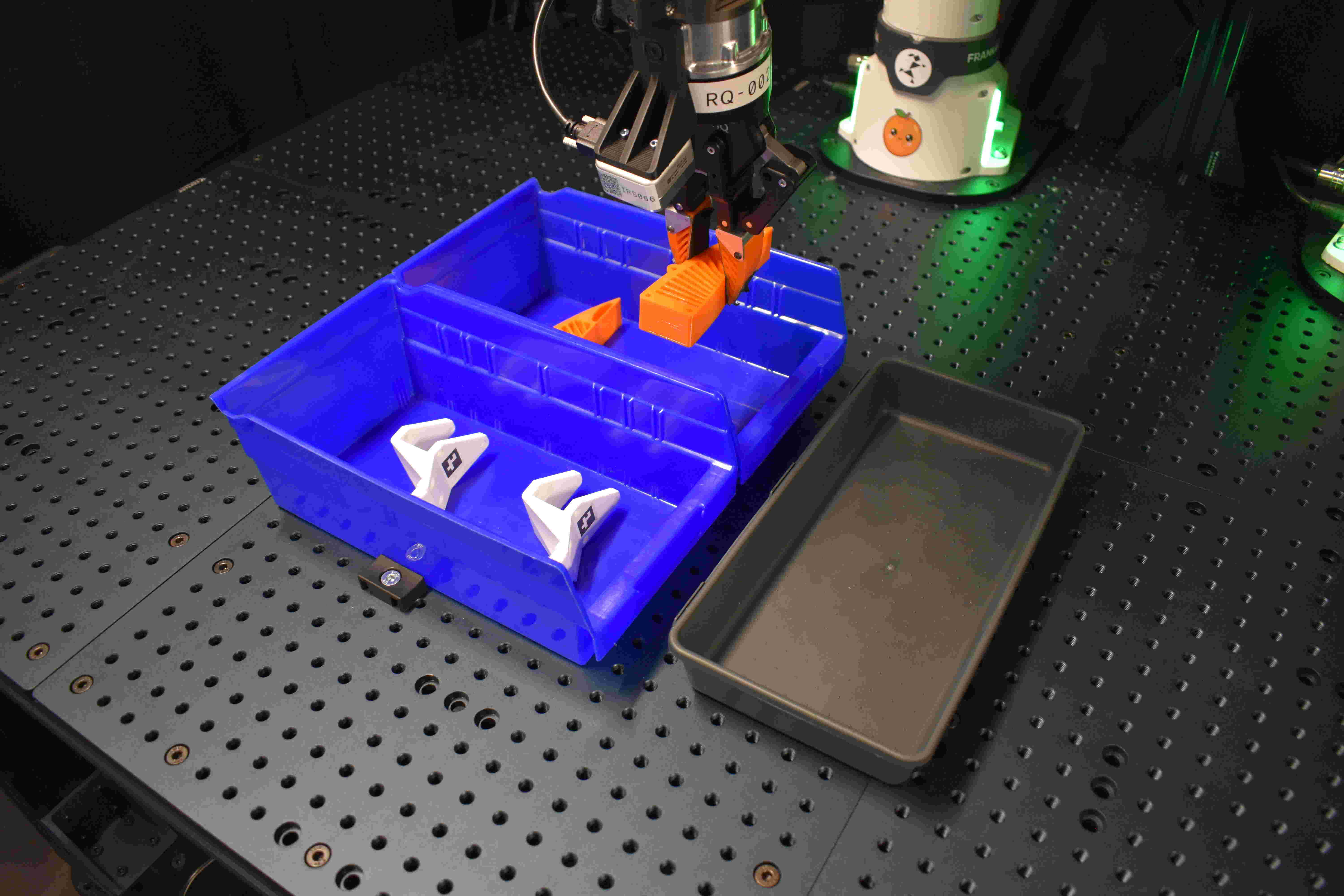}}
    \subfloat[Place \#1
    ]{\includegraphics[width=0.2\linewidth]{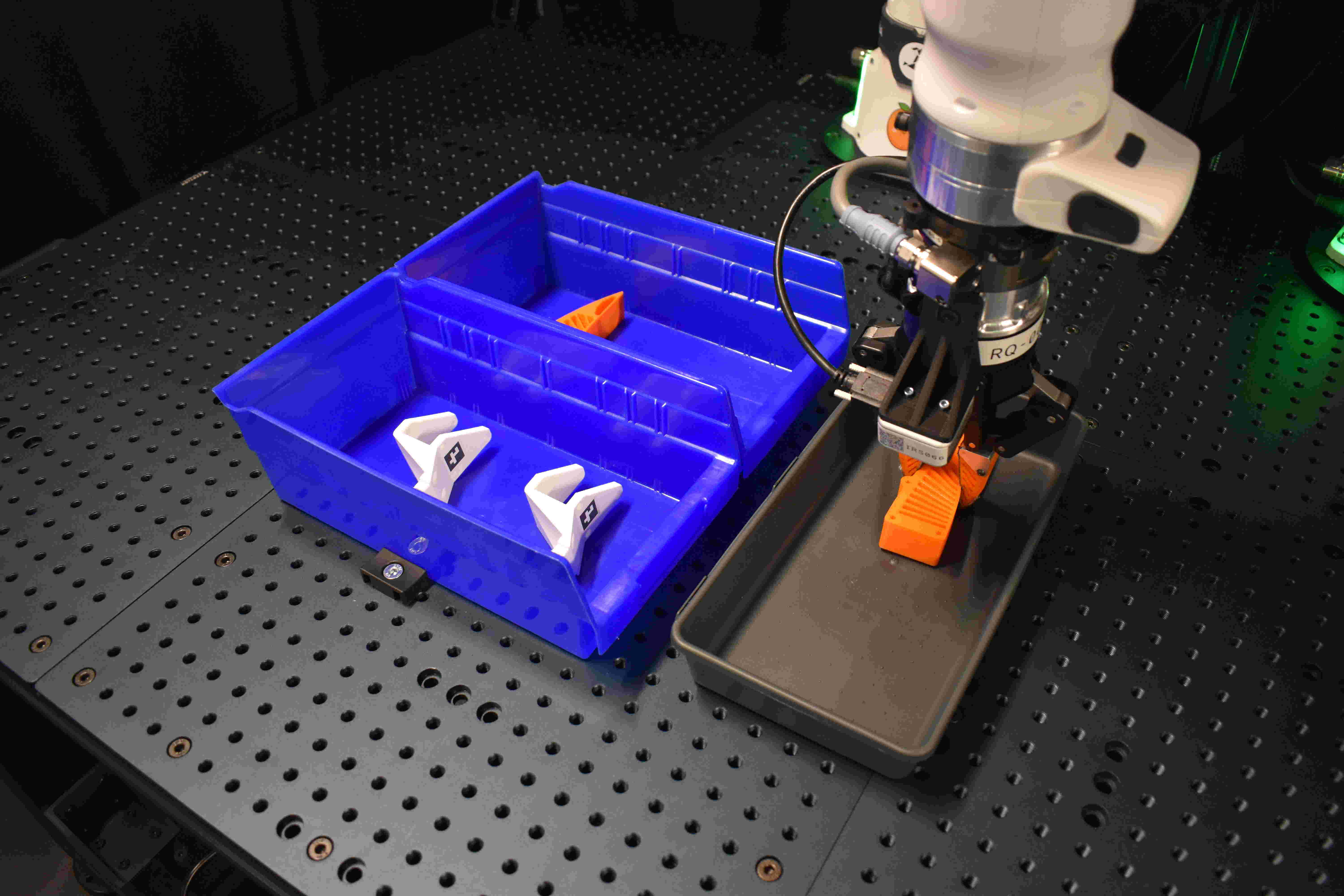}}
    \subfloat[Pick \#2
    ]{\includegraphics[width=0.2\linewidth]{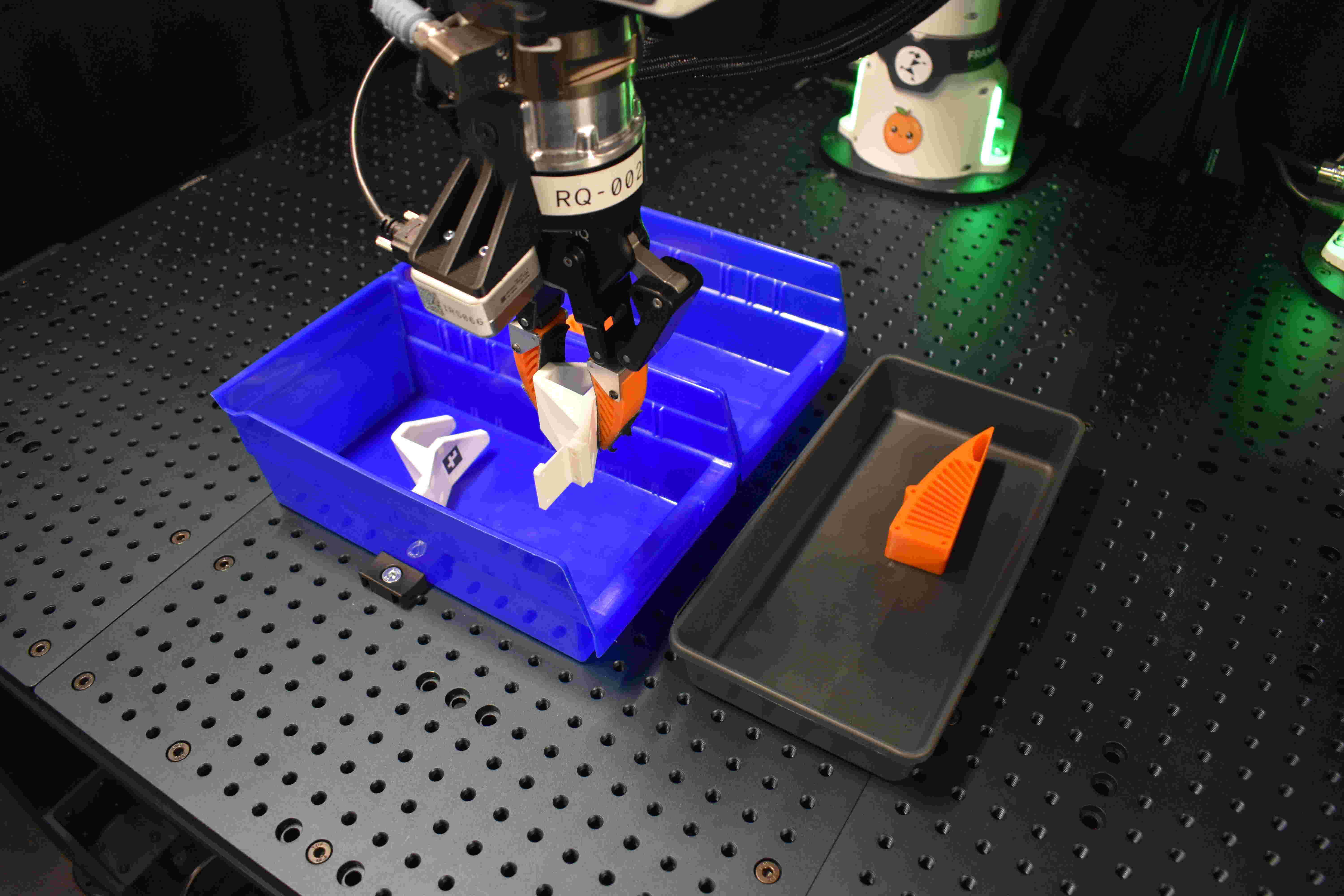}}
    \subfloat[Place \#2 (done)
    ]{\includegraphics[width=0.2\linewidth]{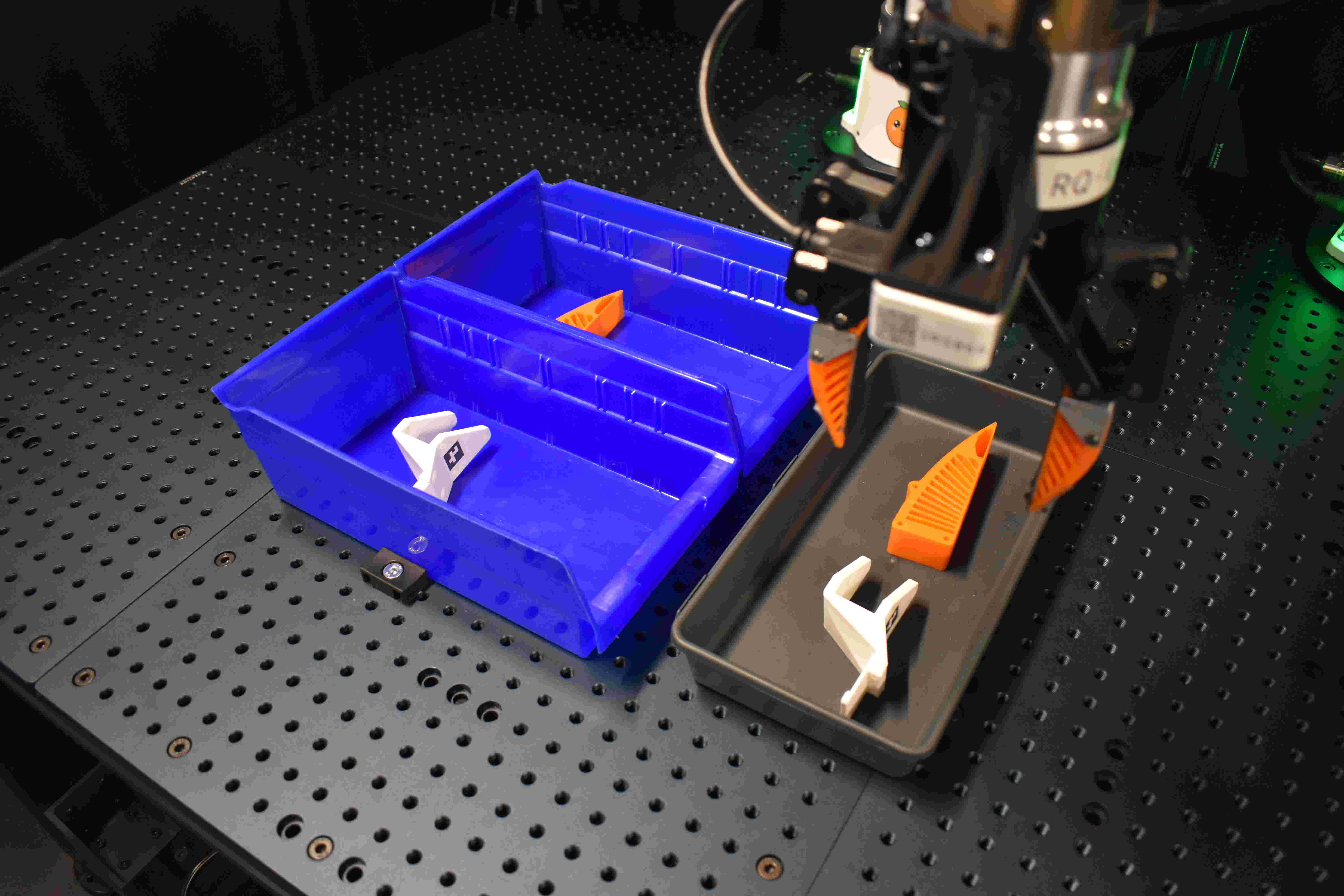}}
    \caption{
    Key-frame images for each task, showing important events in successful trajectories. 
    In Kitting-Original (third row), one part is placed in the center of each bin. 
    In Kitting-Modified (last row), two parts are placed in each bin, and neither are at the center location. 
    Kitting-Modified tests adaptation to task distribution shift, where the goal of picking two pieces into the kitting tray is the same, but the task conditions have changed.
    }
    \label{fig:real_kitting}
\end{figure*}

\textbf{Peg Insertion}.
In this task, the robot inserts a 3D-printed peg into a board. 
The peg tolerance for mating is ~1-2~\unit{\mm} and the insertion length is ~50~\unit{\mm}. 
We evaluate on this task using objects from the Functional Manipulation Assembly~\cite{luo2025fmb} benchmark to provide a comparison similar to recent work that report real-world RL experiments~\cite{luo2024serl, zhou2025efficient}.
However, note that differences in action space such as the size of the safety box, range of allowed end effector position and orientation, and other task parameters mean the results across works are not directly comparable. 
Hence we run baselines on our task configuration to provide a fair comparison. 

\subsubsection*{Starting Configuration}
The board is fixed on the table for all episodes.  
The peg starts grasped by the gripper.
Before the start of every episode, the gripper pose is randomized by sampling a random delta pose action and executing it three times.

\subsubsection*{Success Condition}
The peg is fully inserted into the board.

\textbf{Pipe Assembly}.
In this task, the robot picks up a PVC pipe and inserts into a plumbing fixture. 
The tolerance for mating is ~1-2~\unit{\mm} and the insertion length is ~57~\unit{\mm}.
Besides evaluating insertion with commercially available parts, this task differs from the prior Peg Insertion task in the following ways: (1) the PVC pipe starts on the table and must be grasped before attempting insertion; and (2) the fixture is not aligned with the initial pose of the pipe, so inserting the pipe requires gripper rotation, or else the pipe will jam.

\subsubsection*{Starting Configuration}
The pipe, gripper, and fixture start in the same poses for this task.
Despite having less initial variation in this task compared to Peg Insertion, we observe that this task has lower BC policy success rates as it requires grasping followed by insertion.
The average number of steps required to achieve the task is double that of Peg Insertion. 

\subsubsection*{Success Condition}
The pipe is fully inserted into the fixture.
A black line is drawn on the pipe to help the human operator identify when the pipe is fully inserted.

\begin{table*}[t!]
\centering
\caption{Real-World Evaluation Results}
\label{tab:real_results}
\setlength{\tabcolsep}{2.5pt} 

\begin{tabular}{l cc  cc  cc}
\toprule
& \multicolumn{2}{c}{Peg Insertion}
& \multicolumn{2}{c}{Pipe Assembly}
& \multicolumn{2}{c}{Kitting-Modified} \\
\cmidrule(lr){2-3}
\cmidrule(lr){4-5}
\cmidrule(lr){6-7}

Method
& Success Rate
    & Online Learner Steps
& Success Rate
    & Online Learner Steps
& Success Rate
    & Online Learner Steps\\
\midrule

BC Policy 
    & 0.70 
        & --
    & 0.20 
        & --
    & 0.35 
        & -- \\
IBRL~\cite{hu2023imitation} 
    & \textbf{0.95} 
        & 40k
    & 0.0 
        & 90k
    & 0.0 
        & 165k\\
Q2RL (Ours) 
    & \textbf{1.0} 
        & 60k
    & \textbf{0.75} 
        & 90k
    & \textbf{0.70} 
        & 165k\\

\bottomrule
\end{tabular}
\vspace{0.5em}
\caption*{Success Rate over 20 trials. Results are reported for the best evaluated checkpoint.}
\end{table*}

\textbf{Kitting}.
In this task, the robot takes a robot part from each bin and places it into a kitting tray. 
Performing this task requires two pick and place actions.
This task is the longest of the three tasks, requiring on average nearly double the number of steps as Pipe Assembly.

There are two versions of this task: Kitting-Original and Kitting-Modified. 
In Kitting-Original, only one object part is in each bin at the start of an episode; in Kitting-Modified, each bin starts with two identical parts.  
See Fig.~\ref{fig:real_images} for key-frame images of Kitting-Original and Kitting-Modified. 

The BC policy for this task is trained on demos from the Kitting-Original task. 
When evaluated on Kitting-Original, the BC policy achieves 0.95 success rate, but in the modified setting, the success rate drops to 0.35. 
We evaluate on Kitting-Modified to determine how well each method adapts to task distribution shift.
For \MethodName{}, the Q-estimation phase is provided with 50 episodes from the Kitting-Original setting, and the online replay buffer is seeded with 30 episodes from the Kitting-Modified setting. 
For IBRL, which does not have an initial Q-estimation phase, the online replay buffer is seeded with 30 episodes from the Kitting-Modified setting. 

\subsubsection*{Starting Configuration}
The bins and kitting tray start in the same location for both Kitting-Original and Kitting-Modified.
In Kitting-Original, the sole object part in each bin is placed approximately in the center of the bin with minimal rotation.
In Kitting-Modified, the two parts in each bin are placed so that the space in the bin is divided into thirds, so that neither of the parts in each bin are in the same location as in Kitting-Original. 
The robot gripper starts at the same initial pose for each episode. 

\subsubsection*{Success Condition}
One of each part is placed in the kitting tray. 
In Modified Kitting, where there are two identical parts in each bin, it does not matter which part is picked.

\subsection{Behavior Cloning Policy Details}

The BC policies for real world tasks are all Gaussian Mixture Models with ImageNet pre-trained ResNet-10~\cite{he2016deep} vision encoders.
For training the BC policies, we use hyperparameters provided by robomimic for the Can-Image task.
All BC policies were trained on a single v100 GPU.
Table~\ref{tab:bc_hyperparameters} contains hyperparameters for the BC policy and BC loss weights.

\subsection{Reinforcement Learning Agent Details}

All RL agents for real world tasks use a small convolutional neural network as the vision encoder followed by a Gaussian MLP head. 
The agents were trained and evaluated using stochastic actions, i.e. the actions were sampled from the Gaussian distribution output by the agent. 
We found that stochastic actions performed better for our real-world, contact-rich manipulation tasks.
Table~\ref{tab:rl_hyperparameters} contains hyperparameters for RL agents. 
On-robot RL training was run on a single A6000 GPU.

\subsection{Additional Real World Results}

Tab.~\ref{tab:real_results} contains additional details for the results in Fig.~\ref{fig:real_exps}.
We report the best success rate per task for the evaluated checkpoints.
Online learner steps reported in our tables correspond to a single high-UTD learner update. For example, with UTD = 4, one learner update consists of four critic updates and one actor update.
The learner steps for \MethodName{} contains the Q-estimation steps as well.
Refer to the supplementary video for more qualitative results.

\begin{table}[H]
\centering
\caption{Peg Insertion without Seeded Replay Buffer}
\label{tab:peg_insert_no_data}
\setlength{\tabcolsep}{2.5pt} 

\begin{tabular}{l cc }
\toprule
& \multicolumn{2}{c}{Peg Insertion, No Data}\\
\cmidrule(lr){2-3}

Method
& Success Rate
    & Online Learner Steps \\
\midrule

BC Policy
    & 0.70 
        & -- \\
CalQL~\cite{nakamoto2023cal} Deterministic
    & 0.10 
    & -- \\
CalQL~\cite{nakamoto2023cal} Stochastic
    & 0.20  
    & -- \\
IBRL~\cite{hu2023imitation} 
    & 0.0 
    & 20k \\
Q2RL (Ours)
    & \textbf{1.00} 
    & 20k \\

\bottomrule
\end{tabular}
\vspace{0.3em}
\caption*{All methods either do not do online learning or start with no data in the online replay buffer. Success rate is over 20 trials. Results are reported for the best evaluated checkpoint.}
\end{table}

\subsection{Baseline: Real World CalQL}

We report results on the Peg Insertion task for methods that don't use prior data to seed the online replay buffer in Table~\ref{tab:peg_insert_no_data}. 
In this table, we include results for CalQL~\cite{nakamoto2023cal}, an offline RL baseline that estimates a conservative Q-function from offline data. 
We train CalQL on the dataset of successful human demonstrations that we used to train the BC policy, as that is the data that would be available in our setting.
CalQL was trained with 250k offline training steps. 
We evaluated both stochastic and deterministic variants of CalQL. 
The success rates for both variants was much lower than training a BC policy's 70\%. 
We also observe that stochastic CalQL commanded high-jerk actions that were not in the original demonstration dataset. 
These high-jerk actions resulted in 4 safety violations during evaluation.
These results suggest that offline RL is unable to learn a good Q-function from small datasets of successful demonstrations that are commonly used to train behavior cloning policies.
\MethodName{} was able to improve upon the performance of the original BC policy to achieve 100\% success over 20 trials. 
\MethodName{} leverages the fact that BC policies only require successful demonstrations for training, then uses the trained BC policy to estimate a Q-function $\hat{Q}_{\rm BC}$ (Sec.~\ref{sec:Q_estimation}). 
The estimated $\hat{Q}_{\rm BC}$ is used to train an improved policy using online RL with Q-Gating (Sec.~\ref{sec:Q_gating}).
\putbib
\end{bibunit}


\begin{thebibliography}{43}
\providecommand{\natexlab}[1]{#1}
\providecommand{\url}[1]{\texttt{#1}}
\expandafter\ifx\csname urlstyle\endcsname\relax
  \providecommand{\doi}[1]{doi: #1}\else
  \providecommand{\doi}{doi: \begingroup \urlstyle{rm}\Url}\fi

\bibitem[Agarwal et~al.(2020)Agarwal, Schuurmans, and Norouzi]{agarwal2020optimistic}
Rishabh Agarwal, Dale Schuurmans, and Mohammad Norouzi.
\newblock An optimistic perspective on offline reinforcement learning.
\newblock In \emph{International conference on machine learning}, pages 104--114. PMLR, 2020.

\bibitem[Ball et~al.(2023)Ball, Smith, Kostrikov, and Levine]{ball2023efficient}
Philip~J Ball, Laura Smith, Ilya Kostrikov, and Sergey Levine.
\newblock Efficient online reinforcement learning with offline data.
\newblock In \emph{International Conference on Machine Learning}, pages 1577--1594. PMLR, 2023.

\bibitem[Biza et~al.(2025)Biza, Weng, Sun, Schmeckpeper, Kelestemur, Ma, Platt, van~de Meent, and Wong]{11128466}
Ondrej Biza, Thomas Weng, Lingfeng Sun, Karl Schmeckpeper, Tarik Kelestemur, Yecheng~Jason Ma, Robert Platt, Jan-Willem van~de Meent, and Lawson~L.S. Wong.
\newblock On-robot reinforcement learning with goal-contrastive rewards.
\newblock In \emph{2025 IEEE International Conference on Robotics and Automation (ICRA)}, pages 4797--4805, 2025.
\newblock \doi{10.1109/ICRA55743.2025.11128466}.

\bibitem[Bobu et~al.(2018)Bobu, Bajcsy, Fisac, and Dragan]{bobu2018learning}
Andreea Bobu, Andrea Bajcsy, Jaime~F Fisac, and Anca~D Dragan.
\newblock Learning under misspecified objective spaces.
\newblock In \emph{Conference on Robot Learning}, pages 796--805. PMLR, 2018.

\bibitem[Chen et~al.()Chen, Adebola, and Goldberg]{BerkeleyUR5Website}
Lawrence~Yunliang Chen, Simeon Adebola, and Ken Goldberg.
\newblock Berkeley {UR5} demonstration dataset.
\newblock \url{https://sites.google.com/view/berkeley-ur5/home}.

\bibitem[Chi et~al.(2025)Chi, Xu, Feng, Cousineau, Du, Burchfiel, Tedrake, and Song]{chi2025diffusion}
Cheng Chi, Zhenjia Xu, Siyuan Feng, Eric Cousineau, Yilun Du, Benjamin Burchfiel, Russ Tedrake, and Shuran Song.
\newblock Diffusion policy: Visuomotor policy learning via action diffusion.
\newblock \emph{The International Journal of Robotics Research}, 44\penalty0 (10-11):\penalty0 1684--1704, 2025.

\bibitem[Dodeja et~al.(2025)Dodeja, Schmeckpeper, Vats, Weng, Jia, Konidaris, and Tellex]{dodeja2025accelerating}
Lakshita Dodeja, Karl Schmeckpeper, Shivam Vats, Thomas Weng, Mingxi Jia, George Konidaris, and Stefanie Tellex.
\newblock Accelerating residual reinforcement learning with uncertainty estimation.
\newblock \emph{arXiv preprint arXiv:2506.17564}, 2025.

\bibitem[Fu et~al.(2020)Fu, Kumar, Nachum, Tucker, and Levine]{fu2020d4rl}
Justin Fu, Aviral Kumar, Ofir Nachum, George Tucker, and Sergey Levine.
\newblock D4rl: Datasets for deep data-driven reinforcement learning.
\newblock \emph{arXiv preprint arXiv:2004.07219}, 2020.

\bibitem[Fujimoto et~al.(2019)Fujimoto, Conti, Ghavamzadeh, and Pineau]{fujimoto2019benchmarking}
Scott Fujimoto, Edoardo Conti, Mohammad Ghavamzadeh, and Joelle Pineau.
\newblock Benchmarking batch deep reinforcement learning algorithms.
\newblock \emph{arXiv preprint arXiv:1910.01708}, 2019.

\bibitem[Garg et~al.(2021)Garg, Chakraborty, Cundy, Song, and Ermon]{garg2021iq}
Divyansh Garg, Shuvam Chakraborty, Chris Cundy, Jiaming Song, and Stefano Ermon.
\newblock Iq-learn: Inverse soft-q learning for imitation.
\newblock \emph{Advances in Neural Information Processing Systems}, 34:\penalty0 4028--4039, 2021.

\bibitem[Haarnoja et~al.(2018)Haarnoja, Zhou, Abbeel, and Levine]{haarnoja2018soft}
Tuomas Haarnoja, Aurick Zhou, Pieter Abbeel, and Sergey Levine.
\newblock Soft actor-critic: Off-policy maximum entropy deep reinforcement learning with a stochastic actor.
\newblock In \emph{International conference on machine learning}, pages 1861--1870. Pmlr, 2018.

\bibitem[He et~al.(2016)He, Zhang, Ren, and Sun]{he2016deep}
Kaiming He, Xiangyu Zhang, Shaoqing Ren, and Jian Sun.
\newblock Deep residual learning for image recognition.
\newblock In \emph{Proceedings of the IEEE conference on computer vision and pattern recognition}, pages 770--778, 2016.

\bibitem[Hu et~al.(2024)Hu, Mirchandani, and Sadigh]{hu2023imitation}
Hengyuan Hu, Suvir Mirchandani, and Dorsa Sadigh.
\newblock Imitation bootstrapped reinforcement learning.
\newblock \emph{Robotics: Science and Systems XX, Delft, The Netherlands, July 15-19, 2024}, 2024.
\newblock \doi{10.15607/RSS.2024.XX.056}.

\bibitem[Jiang et~al.(2025)Jiang, Fang, Roy, Lozano-P{\'e}rez, Kaelbling, and Ancha]{jiang2025streaming}
Sunshine Jiang, Xiaolin Fang, Nicholas Roy, Tom{\'a}s Lozano-P{\'e}rez, Leslie~Pack Kaelbling, and Siddharth Ancha.
\newblock Streaming flow policy: Simplifying diffusion $/$ flow-matching policies by treating action trajectories as flow trajectories.
\newblock \emph{arXiv preprint arXiv:2505.21851}, 2025.

\bibitem[Johannink et~al.(2019)Johannink, Bahl, Nair, Luo, Kumar, Loskyll, Ojea, Solowjow, and Levine]{johannink2019residual}
Tobias Johannink, Shikhar Bahl, Ashvin Nair, Jianlan Luo, Avinash Kumar, Matthias Loskyll, Juan~Aparicio Ojea, Eugen Solowjow, and Sergey Levine.
\newblock Residual reinforcement learning for robot control.
\newblock In \emph{2019 international conference on robotics and automation (ICRA)}, pages 6023--6029. IEEE, 2019.

\bibitem[Kelly et~al.(2019)Kelly, Sidrane, Driggs{-}Campbell, and Kochenderfer]{kelly19hgdagger}
Michael Kelly, Chelsea Sidrane, Katherine~Rose Driggs{-}Campbell, and Mykel~J. Kochenderfer.
\newblock Hg-dagger: Interactive imitation learning with human experts.
\newblock In \emph{International Conference on Robotics and Automation, {ICRA} 2019, Montreal, QC, Canada, May 20-24, 2019}, pages 8077--8083. {IEEE}, 2019.

\bibitem[Kolchinsky and Tracey(2017)]{kolchinsky2017estimating}
Artemy Kolchinsky and Brendan~D Tracey.
\newblock Estimating mixture entropy with pairwise distances.
\newblock \emph{Entropy}, 19\penalty0 (7):\penalty0 361, 2017.

\bibitem[Kumar et~al.(2020)Kumar, Zhou, Tucker, and Levine]{kumar2020conservative}
Aviral Kumar, Aurick Zhou, George Tucker, and Sergey Levine.
\newblock Conservative q-learning for offline reinforcement learning.
\newblock \emph{Advances in neural information processing systems}, 33:\penalty0 1179--1191, 2020.

\bibitem[Laidlaw and Dragan(2022)]{laidlaw2022boltzmann}
Cassidy Laidlaw and Anca Dragan.
\newblock The boltzmann policy distribution: Accounting for systematic suboptimality in human models.
\newblock In \emph{ICLR}, 2022.

\bibitem[Lei et~al.(2026)Lei, Li, Yu, Wei, Guo, Jiang, Wang, Liang, and Xu]{lei2026rl100performantroboticmanipulation}
Kun Lei, Huanyu Li, Dongjie Yu, Zhenyu Wei, Lingxiao Guo, Zhennan Jiang, Ziyu Wang, Shiyu Liang, and Huazhe Xu.
\newblock Rl-100: Performant robotic manipulation with real-world reinforcement learning, 2026.
\newblock URL \url{https://arxiv.org/abs/2510.14830}.

\bibitem[Levine et~al.(2020)Levine, Kumar, Tucker, and Fu]{levine2020offline}
Sergey Levine, Aviral Kumar, George Tucker, and Justin Fu.
\newblock Offline reinforcement learning: Tutorial, review, and perspectives on open problems.
\newblock \emph{arXiv preprint arXiv:2005.01643}, 2020.

\bibitem[Lipman et~al.(2022)Lipman, Chen, Ben-Hamu, Nickel, and Le]{lipman2022flow}
Yaron Lipman, Ricky~TQ Chen, Heli Ben-Hamu, Maximilian Nickel, and Matt Le.
\newblock Flow matching for generative modeling.
\newblock \emph{arXiv preprint arXiv:2210.02747}, 2022.

\bibitem[Luce et~al.(1959)]{luce1959individual}
R~Duncan Luce et~al.
\newblock \emph{Individual choice behavior}, volume~4.
\newblock Wiley New York, 1959.

\bibitem[Luo et~al.(2024)Luo, Hu, Xu, Tan, Berg, Sharma, Schaal, Finn, Gupta, and Levine]{luo2024serl}
Jianlan Luo, Zheyuan Hu, Charles Xu, You~Liang Tan, Jacob Berg, Archit Sharma, Stefan Schaal, Chelsea Finn, Abhishek Gupta, and Sergey Levine.
\newblock Serl: A software suite for sample-efficient robotic reinforcement learning.
\newblock In \emph{2024 IEEE International Conference on Robotics and Automation (ICRA)}, pages 16961--16969. IEEE, 2024.

\bibitem[Luo et~al.(2025{\natexlab{a}})Luo, Xu, Liu, Tan, Lin, Wu, Abbeel, and Levine]{luo2025fmb}
Jianlan Luo, Charles Xu, Fangchen Liu, Liam Tan, Zipeng Lin, Jeffrey Wu, Pieter Abbeel, and Sergey Levine.
\newblock Fmb: a functional manipulation benchmark for generalizable robotic learning.
\newblock \emph{The International Journal of Robotics Research}, 44\penalty0 (4):\penalty0 592--606, 2025{\natexlab{a}}.

\bibitem[Luo et~al.(2025{\natexlab{b}})Luo, Xu, Wu, and Levine]{luo2025precise}
Jianlan Luo, Charles Xu, Jeffrey Wu, and Sergey Levine.
\newblock Precise and dexterous robotic manipulation via human-in-the-loop reinforcement learning.
\newblock \emph{Science Robotics}, 10\penalty0 (105):\penalty0 eads5033, 2025{\natexlab{b}}.

\bibitem[Mainprice et~al.(2015)Mainprice, Hayne, and Berenson]{7139282}
Jim Mainprice, Rafi Hayne, and Dmitry Berenson.
\newblock Predicting human reaching motion in collaborative tasks using inverse optimal control and iterative re-planning.
\newblock In \emph{2015 IEEE International Conference on Robotics and Automation (ICRA)}, pages 885--892, 2015.
\newblock \doi{10.1109/ICRA.2015.7139282}.

\bibitem[Mandlekar et~al.(2020)Mandlekar, Xu, Mart{\'\i}n-Mart{\'\i}n, Savarese, and Fei-Fei]{mandlekar2020learning}
Ajay Mandlekar, Danfei Xu, Roberto Mart{\'\i}n-Mart{\'\i}n, Silvio Savarese, and Li~Fei-Fei.
\newblock Learning to generalize across long-horizon tasks from human demonstrations.
\newblock \emph{arXiv preprint arXiv:2003.06085}, 2020.

\bibitem[Mandlekar et~al.(2021)Mandlekar, Xu, Wong, Nasiriany, Wang, Kulkarni, Fei-Fei, Savarese, Zhu, and Mart{\'\i}n-Mart{\'\i}n]{mandlekar2021matters}
Ajay Mandlekar, Danfei Xu, Josiah Wong, Soroush Nasiriany, Chen Wang, Rohun Kulkarni, Li~Fei-Fei, Silvio Savarese, Yuke Zhu, and Roberto Mart{\'\i}n-Mart{\'\i}n.
\newblock What matters in learning from offline human demonstrations for robot manipulation.
\newblock \emph{arXiv preprint arXiv:2108.03298}, 2021.

\bibitem[Nakamoto et~al.(2023)Nakamoto, Zhai, Singh, Sobol~Mark, Ma, Finn, Kumar, and Levine]{nakamoto2023cal}
Mitsuhiko Nakamoto, Simon Zhai, Anikait Singh, Max Sobol~Mark, Yi~Ma, Chelsea Finn, Aviral Kumar, and Sergey Levine.
\newblock Cal-ql: Calibrated offline rl pre-training for efficient online fine-tuning.
\newblock \emph{Advances in Neural Information Processing Systems}, 36:\penalty0 62244--62269, 2023.

\bibitem[Nasiriany et~al.(2022)Nasiriany, Gao, Mandlekar, and Zhu]{nasiriany2022learning}
Soroush Nasiriany, Tian Gao, Ajay Mandlekar, and Yuke Zhu.
\newblock Learning and retrieval from prior data for skill-based imitation learning.
\newblock \emph{arXiv preprint arXiv:2210.11435}, 2022.

\bibitem[Rajeswaran et~al.(2018)Rajeswaran, Kumar, Gupta, Vezzani, Schulman, Todorov, and Levine]{Rajeswaran-RSS-18}
Aravind Rajeswaran, Vikash Kumar, Abhishek Gupta, Giulia Vezzani, John Schulman, Emanuel Todorov, and Sergey Levine.
\newblock {Learning Complex Dexterous Manipulation with Deep Reinforcement Learning and Demonstrations}.
\newblock In \emph{Proceedings of Robotics: Science and Systems (RSS)}, 2018.

\bibitem[Ross and Bagnell(2010)]{ross2010efficient}
St{\'e}phane Ross and Drew Bagnell.
\newblock Efficient reductions for imitation learning.
\newblock In \emph{Proceedings of the thirteenth international conference on artificial intelligence and statistics}, pages 661--668. JMLR Workshop and Conference Proceedings, 2010.

\bibitem[Ross et~al.(2011)Ross, Gordon, and Bagnell]{ross2011reduction}
St{\'e}phane Ross, Geoffrey Gordon, and Drew Bagnell.
\newblock A reduction of imitation learning and structured prediction to no-regret online learning.
\newblock In \emph{Proceedings of the fourteenth international conference on artificial intelligence and statistics}, pages 627--635. JMLR Workshop and Conference Proceedings, 2011.

\bibitem[Sheebaelhamd et~al.(2025)Sheebaelhamd, Tschannen, Muehlebach, and Vernade]{sheebaelhamd2025quantization}
Ziyad Sheebaelhamd, Michael Tschannen, Michael Muehlebach, and Claire Vernade.
\newblock Quantization-free autoregressive action transformer.
\newblock \emph{arXiv preprint arXiv:2503.14259}, 2025.

\bibitem[Silver et~al.(2018)Silver, Allen, Tenenbaum, and Kaelbling]{silver2018residual}
Tom Silver, Kelsey Allen, Josh Tenenbaum, and Leslie Kaelbling.
\newblock Residual policy learning.
\newblock \emph{arXiv preprint arXiv:1812.06298}, 2018.

\bibitem[Tan(2024)]{agentlace2024}
You~Liang Tan.
\newblock Agentlace, framework for distributed agent policy, May 2024.
\newblock URL \url{https://github.com/youliangtan/agentlace}.

\bibitem[Team et~al.(2024)Team, Ghosh, Walke, Pertsch, Black, Mees, Dasari, Hejna, Kreiman, Xu, et~al.]{team2024octo}
Octo~Model Team, Dibya Ghosh, Homer Walke, Karl Pertsch, Kevin Black, Oier Mees, Sudeep Dasari, Joey Hejna, Tobias Kreiman, Charles Xu, et~al.
\newblock Octo: An open-source generalist robot policy.
\newblock \emph{arXiv preprint arXiv:2405.12213}, 2024.

\bibitem[Wu et~al.(2022)Wu, Escontrela, Hafner, Goldberg, and Abbeel]{wu2022daydreamer}
Philipp Wu, Alejandro Escontrela, Danijar Hafner, Ken Goldberg, and Pieter Abbeel.
\newblock Daydreamer: World models for physical robot learning.
\newblock \emph{Conference on Robot Learning}, 2022.

\bibitem[Yarats et~al.(2021)Yarats, Kostrikov, and Fergus]{yarats2021image}
Denis Yarats, Ilya Kostrikov, and Rob Fergus.
\newblock Image augmentation is all you need: Regularizing deep reinforcement learning from pixels.
\newblock In \emph{International conference on learning representations}, 2021.

\bibitem[Yuan et~al.(2024)Yuan, Mu, Tao, Fang, Zhang, and Su]{yuan2024policy}
Xiu Yuan, Tongzhou Mu, Stone Tao, Yunhao Fang, Mengke Zhang, and Hao Su.
\newblock Policy decorator: Model-agnostic online refinement for large policy model.
\newblock \emph{arXiv preprint arXiv:2412.13630}, 2024.

\bibitem[Zhang et~al.(2018)Zhang, McCarthy, Jow, Lee, Chen, Goldberg, and Abbeel]{zhang2018deep}
Tianhao Zhang, Zoe McCarthy, Owen Jow, Dennis Lee, Xi~Chen, Ken Goldberg, and Pieter Abbeel.
\newblock Deep imitation learning for complex manipulation tasks from virtual reality teleoperation.
\newblock In \emph{2018 IEEE international conference on robotics and automation (ICRA)}, pages 5628--5635. Ieee, 2018.

\bibitem[Zhou et~al.(2025)Zhou, Peng, Li, Levine, and Kumar]{zhou2025efficient}
Zhiyuan Zhou, Andy Peng, Qiyang Li, Sergey Levine, and Aviral Kumar.
\newblock Efficient online reinforcement learning fine-tuning need not retain offline data.
\newblock In \emph{The Thirteenth International Conference on Learning Representations}, 2025.
\newblock URL \url{https://openreview.net/forum?id=HN0CYZbAPw}.

\end{thebibliography}
\end{document}